\documentclass{article}

     \PassOptionsToPackage{numbers, compress}{natbib}

     \usepackage[final]{neurips_2023}

\usepackage{mystyle}

\newcommand{\kl}{k_{\mathrm{last}}}
\newcommand{\numin}{\nu_{\mathrm{min}}}
\newcommand{\sigmamin}{\sigma_{\mathrm{min}}}
\newcommand{\hbtheta}{\hat{\btheta}}
\newcommand{\hbpsi}{\hat{\bpsi}}
\newcommand{\ttau}{\tilde{\tau}}
\newcommand{\hbw}{\hat{\bw}}
\newcommand{\cbw}{\check{\bw}}
\newcommand{\tbw}{\tilde{\bw}}
\newcommand{\pesQ}{\check{Q}}
\newcommand{\pesV}{\check{V}}
\newcommand{\normbphiR}[2]{\nbr{\bphi_{#1,#2}}_{\bH_{#1-1,#2}^{-1}}}
\newcommand{\normbphiV}[2]{\nbr{\bphi_{#1,#2}}_{\bSigma_{#1-1,#2}^{-1}}}
\newcommand{\ceil}[1]{\lceil #1 \rceil}

\newcommand{\algobandit}{\textsc{Heavy-OFUL}}
\newcommand{\algomdp}{\textsc{Heavy-LSVI-UCB}}

\title{Tackling Heavy-Tailed Rewards in Reinforcement Learning with Function Approximation: Minimax Optimal and Instance-Dependent Regret Bounds}

\author{%
\textbf{Jiayi Huang}$^1$ \qquad \textbf{Han Zhong}$^1$ \qquad \textbf{Liwei Wang}$^{1,2}$ \qquad \textbf{Lin F. Yang}$^3$ \\
$^1$Center for Data Science, Peking University \\
$^2$National Key Laboratory of General Artificial Intelligence, \\
School of Intelligence Science and Technology, Peking University \\
$^3$University of California, Los Angeles \\
\texttt{\{jyhuang, hanzhong\}@stu.pku.edu.cn}, \\
\texttt{wanglw@cis.pku.edu.cn},\quad \texttt{linyang@ee.ucla.edu}
}

\begin{document}

\maketitle

\begin{abstract}
While numerous works have focused on devising efficient algorithms for reinforcement learning (RL) with uniformly bounded rewards, it remains an open question whether sample or time-efficient algorithms for RL with large state-action space exist when the rewards are \emph{heavy-tailed}, i.e., with only finite $(1+\epsilon)$-th moments for some $\epsilon\in(0,1]$. In this work, we address the challenge of such rewards in RL with linear function approximation. We first design an algorithm, \textsc{Heavy-OFUL}, for heavy-tailed linear bandits, achieving an \emph{instance-dependent} $T$-round regret of $\tilde{O}\big(d T^{\frac{1-\epsilon}{2(1+\epsilon)}} \sqrt{\sum_{t=1}^T \nu_t^2} + d T^{\frac{1-\epsilon}{2(1+\epsilon)}}\big)$, the \emph{first} of this kind. Here, $d$ is the feature dimension, and $\nu_t^{1+\epsilon}$ is the $(1+\epsilon)$-th central moment of the reward at the $t$-th round. We further show the above bound is minimax optimal when applied to the worst-case instances in stochastic and deterministic linear bandits. We then extend this algorithm to the RL settings with linear function approximation. Our algorithm, termed as \textsc{Heavy-LSVI-UCB}, achieves the \emph{first} computationally efficient \emph{instance-dependent} $K$-episode regret of $\tilde{O}(d \sqrt{H \mathcal{U}^*} K^\frac{1}{1+\epsilon} + d \sqrt{H \mathcal{V}^* K})$. Here, $H$ is length of the episode, and $\mathcal{U}^*, \mathcal{V}^*$ are instance-dependent quantities scaling with the central moment of reward and value functions, respectively. We also provide a matching minimax lower bound $\Omega(d H K^{\frac{1}{1+\epsilon}} + d \sqrt{H^3 K})$ to demonstrate the optimality of our algorithm in the worst case. Our result is achieved via a novel robust self-normalized concentration inequality that may be of independent interest in handling heavy-tailed noise in general online regression problems.
\end{abstract}

\section{Introduction}\label{sec:intro}

Designing efficient reinforcement learning (RL) algorithms for large state-action space is a significant challenge within the RL community.
A crucial aspect of RL is understanding the reward functions, which directly impacts the quality of the agent's policy.
In certain real-world situations, reward distributions may exhibit heavy-tailed behavior, characterized by the occurrence of extremely large values at a higher frequency than expected in a normal distribution.
Examples include image noise in signal processing \citep{hamza2001image}, stock price fluctuations in financial markets \citep{cont2001empirical,hull2012risk}, and value functions in online advertising \citep{choi2020online,jebarajakirthy2021mobile}.
However, much of the existing RL literature assumes rewards to be either uniformly bounded or light-tailed (e.g., sub-Gaussian).
In such light-tailed settings, the primary challenge lies in learning the transition probabilities, leading most studies to assume deterministic rewards for ease of analysis \citep{azar2017minimax,jin2020provably,he2023nearly}.
As we will demonstrate, the complexity of learning reward functions may dominate in heavy-tailed settings.
Consequently, the performance of traditional algorithms may decline, emphasizing the need for the development of new, efficient algorithms specifically designed to handle heavy-tailed rewards.

Heavy-tailed distributions have been extensively studied in the field of statistics \citep{catoni2012challenging,lugosi2019mean} and in more specific online learning scenarios, such as bandits \citep{bubeck2013bandits,medina2016no,shao2018almost,xue2020nearly,zhong2021breaking}.
However, there is a dearth of theoretical research in RL concerning heavy-tailed rewards, whose distributions only admit finite $(1+\epsilon)$-th moment for some $\epsilon \in (0, 1]$. One notable exception is \citet{zhuang2021no}, which made a pioneering effort in establishing worst-case regret guarantees in \emph{tabular} Markov Decision Processes (MDPs) with heavy-tailed rewards. However, their algorithm cannot handle RL settings with large state-action space. Moreover, their reliance on truncation-based methods is sub-optimal as these methods heavily depend on raw moments, which do not vanish in deterministic cases.
Therefore, a natural question arises:

\begin{center}
    \emph{Can we derive sample and time-efficient algorithms for RL with large state-action space \\ that achieve instance-dependent regret in the presence of heavy-tailed rewards?}
\end{center}

In this work, we focus on linear MDPs \citep{yang2019sample,jin2020provably} with heavy-tailed rewards and answer the above question affirmatively.
We say a distribution is \emph{heavy-tailed} if it only admits finite $(1+\epsilon)$-th moment for some $\epsilon\in(0,1]$.
Our contributions are summarized as follows.
\begin{itemize}[leftmargin=*]
    \item We first propose a computationally efficient algorithm $\algobandit$ for heavy-tailed linear bandits.
    Such a setting can be regarded as a special case of linear MDPs.
    $\algobandit$ achieves an \emph{instance-dependent} $T$-round regret of $\tilde{O}\big(d T^{\frac{1-\epsilon}{2(1+\epsilon)}} \sqrt{\sum_{t=1}^T \nu_t^2} + d T^{\frac{1-\epsilon}{2(1+\epsilon)}}\big)$, the \emph{first} of this kind.
    Here $d$ is the feature dimension and $\nu_t^{1+\epsilon}$ is the $(1+\epsilon)$-th central moment of the reward at the $t$-th round.
    The instance-dependent regret bound has a main term that only depends on the summation of central moments, and therefore does not have a $\sqrt{T}$ term.
    Our regret bound is shown to be minimax optimal in both stochastic and deterministic linear bandits (See Remark~\ref{rem:bandit-regret} for details).
    \item We then extend this algorithm to time-inhomogeneous linear MDPs with heavy-tailed rewards, resulting in a new computationally efficient algorithm $\algomdp$,
    which achieves a $K$-episode regret scaling as $\tilde{O}(d \sqrt{H\cU^*} K^{\frac{1}{1+\epsilon}} + d \sqrt{H\cV^*K})$ for the \emph{first} time.
    Here, $H$ is the length of the episode and $\cU^*, \cV^*$ are quantities measuring the central moment of the reward functions and transition probabilities, respectively (See Theorem~\ref{thm:mdp-regret} for details).
    Our regret bound is \emph{instance-dependent} since the main term only relies on the instance-dependent quantities, which vanishes when the dynamics and rewards are deterministic. When specialized to special cases, our instance-dependent regret recovers the variance-aware regret in \citet{li2023variance} (See Remark~\ref{rem:mdp-li2023variance} for details) and improves existing first-order regret bounds \citep{wagenmaker2022first,li2023variance} (See Corollary~\ref{cor:mdp-first-order} for details).  
    \item We provide a minimax regret lower bound $\Omega(d H K^{\frac{1}{1+\epsilon}} + d \sqrt{H^3 K})$ for linear MDPs with heavy-tailed rewards, which matches the worst-case regret bound implied by our instance-dependent regret, thereby demonstrating the minimax optimality of $\algomdp$ in the worst case.
\end{itemize}

For better comparisons between our algorithms and state-of-the-art results, we summarize the regrets in Table~\ref{tab:bandit} and \ref{tab:mdp} for linear bandits and linear MDPs, respectively.
\ifarxiv\else
More related works are deferred to Appendix~\ref{sec:related}.
\fi
Remarkably, our results demonstrate that $\epsilon=1$ (i.e. finite variance) is sufficient to obtain variance-aware regret bounds of the same order as the case where rewards are uniformly bounded for both linear bandits and linear MDPs. The main technique contribution behind our results is a novel robust self-normalized concentration inequality inspired by \citet{sun2020adaptive}. To be more specific, it is a non-trivial generalization of adaptive Huber regression from independent and identically distributed (i.i.d.) case to heavy-tailed online regression settings and gives a \emph{self-normalized} bound instead of the $\ell_2$-norm bound in \citet{sun2020adaptive}.
Our result is computationally efficient and only scales with the feature dimension, $d$, $(1+\epsilon)$-th \emph{central} moment of the noise, $\nu$, and does not depend on the absolute magnitude as in other self-normalized concentration inequalities \citep{zhou2021nearly,zhou2022computationally}.

\begin{table*}[t!]
    \caption{Comparisons with previous works on linear bandits.
    $d$, $T$, $\{\sigma_t\}_{t\in[T]}$, $\{\nu_t\}_{t\in[T]}$ are feature dimension, the number of rounds, the variance or central moment of the reward at the $t$-th round.
    }
    \label{tab:bandit}
    \centering
    \resizebox{\linewidth}{!}{
    \begin{tabular}{|c|c|c|c|c|c|}
        \hline
        \textbf{Algorithm}  & \textbf{Regret} & \scriptsize\makecell{\textbf{Instance-}\\\textbf{dependent?}} & \scriptsize\makecell{\textbf{Minimax}\\\textbf{Optimal?}} & \scriptsize\makecell{\textbf{Deterministic-}\\\textbf{Optimal?}} & \scriptsize\makecell{\textbf{Heavy-}\\\textbf{Tailed}\\\textbf{Rewards?}}\\
        \hline
        \makecell{OFUL \citep{abbasi2011improved}} & $\tilde{O}\rbr{d \sqrt{T}}$ & No & \textbf{Yes} & No & No\\
        \hline
        \makecell{IDS-UCB \citep{kirschner2018information} \\ Weighted OFUL+ \citep{zhou2022computationally} \\ AdaOFUL \citep{li2023variance}} & $\tilde{O}\rbr{d \sqrt{\sum_{t=1}^T \sigma_t^2} + d}$ & \textbf{Yes} & \textbf{Yes} & \textbf{Yes} & \makecell{No \\ No \\ $\epsilon=1$}\\
        \hline
        \makecell{MENU and TOFU \citep{shao2018almost}} & $\tilde{O}\rbr{d T^\frac{1}{1+\epsilon}}$ & No & \textbf{Yes} & No & \textbf{Yes}\\
        \hline
        \makecell{$\algobandit$ (\textbf{Ours})} & $\tilde{O}\rbr{d T^{\frac{1-\epsilon}{2(1+\epsilon)}} \sqrt{\sum_{t=1}^T \nu_t^2} + d T^{\frac{1-\epsilon}{2(1+\epsilon)}}}$ & \textbf{Yes} & \textbf{Yes} & \textbf{Yes} & \textbf{Yes}\\
        \hline
    \end{tabular}
    }
\end{table*}

\begin{table*}[t!]
    \caption{Comparisons with previous works on time-inhomogeneous linear MDPs.
    $d$, $H$, $K$, $V^*_1$, $\cG^*$ are feature dimension, the length of the episode, the number of episodes, optimal value function, variance-dependent quantity defined in \citet{li2023variance}. $\cU^*$, $\cV^*$ are defined in Theorem~\ref{thm:mdp-regret}.
    }
    \label{tab:mdp}
    \centering
    \resizebox{\linewidth}{!}{
    \begin{tabular}{|c|c|c|c|c|c|c|}
        \hline
        \textbf{Algorithm}  & \textbf{Regret} & \scriptsize\makecell{\textbf{Central}\\\textbf{Moment-}\\\textbf{Dependent?}} & \scriptsize\makecell{\textbf{First-}\\\textbf{Order?}} & \scriptsize\makecell{\textbf{Minimax}\\\textbf{Optimal?}} & \scriptsize\makecell{\textbf{Computa-}\\\textbf{tionally}\\\textbf{Efficient?}} & \scriptsize\makecell{\textbf{Heavy-}\\\textbf{Tailed}\\\textbf{Rewards?}}\\
        \hline
        \makecell{LSVI-UCB\citep{jin2020provably}} & $\tilde{O}\rbr{\sqrt{d^3H^4K}}$ & No & No & No & \textbf{Yes} & No\\
        \hline
        \makecell{\textsc{Force} \footnotesize 
         \citep{wagenmaker2022first}} & $\tilde{O}\rbr{\sqrt{d^3H^3V^*_1K}}$ & No & \textbf{Yes} & No & No & No\\
        \hline
        \makecell{VO$Q$L \citep{agarwal2023vo} \\ LSVI-UCB++ \citep{he2023nearly}} & $\tilde{O}\rbr{d\sqrt{H^3K}}$ & No & No & \textbf{Yes} & \textbf{Yes} & No\\
        \hline
        \makecell{VARA \citep{li2023variance}} & $\tilde{O}\Big(d\sqrt{H\cG^*K}\Big)$ & \textbf{Yes} & \textbf{Yes} & \textbf{Yes} & \textbf{Yes} & $\epsilon=1$\\
        \hline
        \makecell{$\algomdp$ (\textbf{Ours})} & $\tilde{O}\rbr{d \sqrt{H\cU^*} K^{\frac{1}{1+\epsilon}} + d \sqrt{H\cV^*K}}$ & \textbf{Yes} & \textbf{Yes} & \textbf{Yes} & \textbf{Yes} & \textbf{Yes}\\
        \hline
    \end{tabular}
    }
\end{table*}

\ifarxiv
\paragraph{Road Map} The rest of the paper is organized as follows. Section~\ref{sec:related} gives related work. Section~\ref{sec:preliminary} introduces heavy-tailed linear bandits and linear MDPs. Section~\ref{sec:huber} presents the robust self-normalized concentration inequality for general online regression problems with heavy-tailed noise. Section~\ref{sec:bandit} and \ref{sec:mdp} give the main results for heavy-tailed linear bandits and linear MDPs, respectively. We then conclude in Section~\ref{sec:conclu}. Experiments and All proofs can be found in Appendix.
\else
\paragraph{Road Map} The rest of the paper is organized as follows. Section~\ref{sec:preliminary} introduces heavy-tailed linear bandits and linear MDPs. Section~\ref{sec:huber} presents the robust self-normalized concentration inequality for general online regression problems with heavy-tailed noise. Section~\ref{sec:bandit} and \ref{sec:mdp} give the main results for heavy-tailed linear bandits and linear MDPs, respectively. We then conclude in Section~\ref{sec:conclu}. Related work, experiments and all proofs can be found in Appendix.
\fi

\paragraph{Notations} Let $\|\ba\| := \|\ba\|_2$. Let $[t]:= \{1,2,\dots,t\}$. Let $\cB_d(r) := \{\bx \in \RR^d | \|\bx\| \le r\}$. Let $x_{[a,b]}:=\max\{a,\min\{x,b\}\}$ denote the projection of $x$ onto the close interval $[a,b]$. Let $\sigma(\{X_s\}_{s\in[t]})$ be the $\sigma$-field generated by random vectors $\{X_s\}_{s\in[t]}$.

\ifarxiv
\section{Related Work}\label{sec:related}

\paragraph{RL with linear function approximation}
In this work, we focus on the setting of linear MDPs \citep{yang2019sample,jin2020provably,hu2022nearly,agarwal2023vo,he2023nearly,zhong2023theoretical}, where the reward functions and transition probabilities can be expressed as linear functions of some known, $d$-dimensional state-action features.
\citet{jin2020provably} proposed the first computationally efficient algorithm LSVI-UCB that achieves $\tilde{O}(\sqrt{d^3H^4K})$ worst-case regret in online learning settings.
LSVI-UCB constructs upper confidence bounds for action-value functions based on least squares regression.
Then \citet{agarwal2023vo,he2023nearly}
improve it to $\tilde{O}(d\sqrt{H^3K})$, which matches the minimax lower bound $\Omega(d\sqrt{H^3K})$ by \citet{zhou2021nearly} up to logarithmic factors.
There is another line of works that make the linear mixture MDP assumption \citep{modi2020sample,jia2020model,ayoub2020model,zhou2021nearly,zhou2022computationally,zhao2023variance} where the transition kernel is a linear combination of some basis transition probability functions.

\paragraph{Instance-dependent regret in bandits and RL}

Recently, there are plenty of works that achieve instance-dependent regret in bandits and RL \citep{kirschner2018information,zhou2021nearly,zhang2021improved,kim2022improved,zhou2022computationally,zhao2023variance,zanette2019tighter,zhou2023sharp,wagenmaker2022first,li2023variance}.
Instance-dependent regret uses fine-grained quantities that inherently characterize the problems, and thus provides tighter guarantee than worst-case regret. These results approximately provide two different kinds of regret.
The first is first-order regret, which was originally achieved by \citet{zanette2019tighter} in tabular MDPs.
Then \citet{wagenmaker2022first} proposed a computationally inefficient\footnote{\citet{wagenmaker2022first} also provided a computationally efficient alternative with an extra factor of $\sqrt{d}$.} algorithm \textsc{Force} in linear MDPs that achieves $\tilde{O}(\sqrt{d^3H^3V^*_1K})$ first-order regret, where $V^*_1$ is the optimal value function.
Recently, \citet{li2023variance} improved it to a computationally efficient result of $\tilde{O}(d\sqrt{H^3V^*_1K})$.
Since first-order regret depends on the optimal value function, which is non-diminishing when the MDPs are less stochastic, it is typically sub-optimal in such deterministic cases.
The second is variance-aware regret, which is well-studied in linear bandits with light-tailed rewards \citep{kirschner2018information,zhou2022computationally}.
These works are based on weighted ridge regression and a Bernstein-style concentration inequality.
However, hardly few works consider heavy-tailed rewards. One exception is \citet{li2023variance}, which improved adaptive Huber regression and provided the first variance-aware regret in the presence of finite-variance rewards. Unfortunately, little has been done in the case $\epsilon<1$. 

\paragraph{Heavy-tailed rewards in bandits and RL}
\citet{bubeck2013bandits} made the first attempt to study heavy-tailed rewards in Multi-Armed Bandits (MAB). Since then, robust mean estimators, such as median-of-means and truncated mean, have been broadly utilized in linear bandits \citep{medina2016no,shao2018almost,xue2020nearly,zhong2021breaking} to achieve tight regret bounds. As far as our knowledge, the only work that considers heavy-tailed rewards where the variance could be non-existent in RL is \citet{zhuang2021no}, which established a regret bound in tabular MDPs that is tight with respect to $S,A,K$ by plugging truncated mean into UCRL2 \citep{auer2008near} and Q-Learning \citep{jin2018q}.

\paragraph{Robust mean estimators and robust regression}
\citet{lugosi2019mean} provided an overview of most robust mean estimators for heavy-tailed distributions, including the median-of-means estimator, truncated mean, and Catoni's M-estimator. The median-of-means estimator has limitation with minimum sample size. Truncated mean uses raw $(1+\epsilon)$-th moments, thus is sub-optimal in the deterministic case. The original version of Catoni's M-estimator in \citet{catoni2012challenging} requires the existence of finite variance. Then \citet{chen2021generalized,bhatt2022nearly} generalized it to handle heavy-tailed distributions while preserving the same order as $\epsilon \to 1$.
\citet{wagenmaker2022first} extended Catoni's M-estimator to general heterogeneous online regression settings with some covering arguments. Different from \citet{catoni2012challenging}, their results scale with the second raw moments instead of variance.
\citet{sun2020adaptive} imposed adaptive Huber regression to handle homogeneous offline heavy-tailed noise by utilizing Huber loss \citep{huber1964robust} as a surrogate of squared loss. \citet{li2023variance} modified adaptive Huber regression to handle finite-variance noise in heterogeneous online regression settings.
\fi

\section{Preliminaries}\label{sec:preliminary}

\subsection{Heavy-Tailed Linear Bandits}\label{sec:pre-bandit}

\begin{definition}[Heterogeneous linear bandits with heavy-tailed rewards]
\label{def:bandit}
    Let $\{\cD_t\}_{t\ge1}$ denote a series of fixed decision sets, where all $\bphi_t \in \cD_t$ satisfy $\|\bphi_t\| \le L$ for some known upper bound $L$. At each round $t$, the agent chooses $\bphi_t \in \cD_t$, then receives a reward $R_t$ from the environment. 
    We define the filtration $\{\cF_t\}_{t\ge1}$ as $\cF_t = \sigma(\{\bphi_s,R_s\}_{s\in[t]} \cup \{\bphi_{t+1}\})$ for any $t\ge1$.
    We assume
    \ifarxiv
    \[
        R_t = \dotp{\bphi_t}{\btheta^*} + \varepsilon_t
    \]
    \else
    $R_t = \dotp{\bphi_t}{\btheta^*} + \varepsilon_t$
    \fi
    with the unknown coefficient $\btheta^* \in \cB_d(B)$ for some known upper bound $B$. The random variable $\varepsilon_t \in \RR$ is $\cF_t$-measurable and satisfies $\EE[\varepsilon_t | \cF_{t-1}] = 0, \EE[|\varepsilon_t|^{1+\epsilon} | \cF_{t-1}] = \nu_t^{1+\epsilon}$ for some $\epsilon\in(0,1]$ with $\nu_t$ being $\cF_{t-1}$-measurable.
\end{definition}
The agent aims to minimize the $T$-round \emph{pseudo-regret} defined as
\ifarxiv
\[
    \mathrm{Regret}(T) = \sum_{t=1}^T [\dotp{\bphi_t^*}{\btheta^*} - \dotp{\bphi_t}{\btheta^*}],
\]
\else
$\mathrm{Regret}(T) = \sum_{t=1}^T [\dotp{\bphi_t^*}{\btheta^*} - \dotp{\bphi_t}{\btheta^*}]$,
\fi
where $\bphi_t^* = \argmax_{\bphi \in \cD_t} \dotp{\bphi}{\btheta^*}$.

\subsection{Linear MDPs with Heavy-Tailed Rewards}\label{sec:pre-mdp}

We use a tuple $M = M(\cS, \cA, H, \{R_h\}_{h\in[H]}, \{\PP_h\}_{h\in[H]})$ to describe the \emph{time-inhomogeneous finite-horizon MDP},
where $\cS$ and $\cA$ are state space and action space, respectively, $H$ is the length of the episode, $R_h:\cS\times\cA\to\RR$ is the random reward function with expectation $r_h:\cS\times\cA\to\RR$, and $\PP_h:\cS\times\cA\to\Delta(\cS)$ is the transition probability function. More details can be found in \citet{puterman2014markov}.
A time-dependent \emph{policy} $\pi = \{\pi_h\}_{h\in H}$ satisfies $\pi_h: \cS \to \Delta(\cA)$ for any $h\in[H]$. When the policy is deterministic, we use $\pi_h(s_h)$ to denote the action chosen at the $h$-th step given $s_h$ by policy $\pi$.
For any state-action pair $(s,a) \in \cS\times\cA$, we define the \emph{state-action value function} $Q_h^\pi(s,a)$ and \emph{state value function} $V_h^\pi(s)$ as follows:
\ifarxiv
\[
    Q_h^\pi(s,a) = \EE\sbr{\sum_{h'=h}^H r(s_{h'}, a_{h'}) | s_h=s, a_h=a},\quad V_h^\pi(s) = Q_h^\pi(s,\pi_h(s)),
\]
\else
$Q_h^\pi(s,a) = \EE\sbr{\sum_{h'=h}^H r(s_{h'}, a_{h'}) | s_h=s, a_h=a}$, $V_h^\pi(s) = Q_h^\pi(s,\pi_h(s))$,
\fi
where the expectation is taken with respect to the transition probability of $M$ and the agent's policy $\pi$. If $\pi$ is randomized, then the definition of $V$ should have an expectation. Denote the optimal value functions as $V_h^*(s) = \sup_\pi V_h^\pi(s)$ and $Q_h^*(s,a) = \sup_\pi Q_h^\pi(s,a)$.

We introduce the following shorthands for simplicity. At the $h$-th step, for any value function $V: \cS \to \RR$, let
\ifarxiv
\[
    [\PP_h V](s,a) = \EE_{s' \sim \PP_h(\cdot|s,a)} V(s'),\quad [\VV_h V](s,a) = [\PP_h V^2](s,a) - [\PP_h V]^2 (s,a)
\]
\else
$[\PP_h V](s,a) = \EE_{s' \sim \PP_h(\cdot|s,a)} V(s')$, $[\VV_h V](s,a) = [\PP_h V^2](s,a) - [\PP_h V]^2 (s,a)$
\fi
denote the expectation and the variance of the next-state value function at the $h$-th step given $(s,a)$.

We aim to minimize the $K$-episode \emph{regret} defined as
\ifarxiv
\[
    \mathrm{Regret}(K) = \sum_{k=1}^K [V_1^*(s_1^k) - V_1^{\pi^k}(s_1^k)].
\]
\else
$\mathrm{Regret}(K) = \sum_{k=1}^K [V_1^*(s_1^k) - V_1^{\pi^k}(s_1^k)]$.
\fi

In the rest of this section, we introduce linear MDPs with heavy-tailed rewards. We first give the definition of linear MDPs studied in \citet{yang2019sample,jin2020provably}, with emphasis that the rewards in their settings are deterministic or uniformly bounded. Then we focus on the heavy-tailed random rewards.

\begin{definition}
\label{def:linear-mdp}
    An MDP $M = M(\cS, \cA, H, \{R_h\}_{h\in[H]}, \{\PP_h\}_{h\in[H]})$ is a \emph{time-inhomogeneous finite-horizon linear MDP}, if there exist known feature maps $\bphi(s,a): \cS\times\cA\to\cB_d(1)$, unknown $d$-dimensional signed measures $\{\bmu_h^*\}_{h\in[H]}$ over $\cS$ with $\|\bmu_h^*(\cS)\|:= \int_{s\in\cS} |\bmu(s)| \ud s \le \sqrt{d}$ and unknown coefficients $\{\btheta_h^*\}_{h\in[H]}\subseteq \cB_d(B)$ for some known upper bound $B$ such that
    \[
    r_h(s,a) = \dotp{\bphi(s,a)}{\btheta_h^*},\quad \PP_h(\cdot|s,a) = \dotp{\bphi(s,a)}{\bmu_h^*(\cdot)}
    \]
    for any state-action pair $(s,a)\in\cS\times\cA$ and timestep $h\in[H]$.
\end{definition}

\begin{assumption}[Realizable rewards]\label{ass:reward}
For all $(s,a,h) \in \cS\times\cA\times[H]$, the random reward $R_h(s,a)$ is independent of next state $s' \sim \PP_h(\cdot|s,a)$ and admits the linear structure
\[
R_h(s,a) = \dotp{\bphi(s,a)}{\btheta_h^*} + \varepsilon_h(s,a),
\]
where $\varepsilon_h(s,a)$ is a mean-zero heavy-tailed random variable specified below.
\end{assumption}

We introduce the notation $\bnu_n[X] = \EE[|X - \EE X|^n]$ for the $n$-th central moment of any random variable $X$.
And for any random reward function at the $h$-th step $R_h: \cS\times\cA \to \RR$, let
\begin{align*}
[\EE_h R_h](s,a) &= \EE[R_h(s_h,a_h)|(s_h,a_h)=(s,a)],\\
[\bnu_{1+\epsilon} R_h](s,a) &= \EE [|[R_h - \EE_h R_h](s_h,a_h)|^{1+\epsilon} | (s_h,a_h)=(s,a)]
\end{align*}
denote its expectation and the $(1+\epsilon)$-th central moment given $(s_h,a_h)=(s,a)$ for short.

\begin{assumption}[Heavy-tailedness of rewards]\label{ass:bounded-moment}
Random variable $\varepsilon_h(s,a)$ satisfies $[\EE_h \varepsilon_h](s,a) = 0$. And for some known $\epsilon, \epsilon' \in(0,1]$ and constants $\nu_R,\nu_{R^\epsilon} \ge 0$, the following unknown moments of $\varepsilon_h(s,a)$ satisfy
\[
[\EE_h|\varepsilon_h|^{1+\epsilon}](s,a) \le \nu_R^{1+\epsilon},\quad [\bnu_{1+\epsilon'} |\varepsilon_h|^{1+\epsilon}](s,a) \le \nu_{R^\epsilon}^{1+\epsilon'}
\]
for all $(s,a,h) \in \cS \times \cA \times [H]$.
\end{assumption}

Assumption~\ref{ass:bounded-moment} generalizes Assumption~2.2 of \citet{li2023variance}, which is the weakest moment condition on random rewards in the current literature of RL with function approximation. Setting $\epsilon=1$ and $\epsilon'=1$ immediately recovers their settings.

\begin{assumption}[Realizable central moments]\label{ass:central-moment}
There are some unknown coefficients $\{\bpsi_h^*\}_{h\in[H]} \subseteq \cB_d(W)$ for some known upper bound $W$ such that 
\[
[\EE_h|\varepsilon_h|^{1+\epsilon}](s,a) = \dotp{\bphi(s,a)}{\bpsi_h^*}
\]
for all $(s,a, h) \in \cS \times \cA \times [H]$.
\end{assumption}

\begin{remark}
    When $\epsilon=1$, that is the rewards have finite variance, \citet{li2023variance} use the fact that $[\bnu_2 R_h](s,a) = [\VV_h R_h](s,a) = [\EE_h R_h^2](s,a) - [\EE_h R_h]^2(s,a)$, assume the linear realizability of the second moment $[\EE_h R_h^2](s,a)$, and estimate it instead. However, when $\epsilon < 1$, there is no such relationship between the $(1+\epsilon)$-th central moment $[\bnu_{1+\epsilon} R_h](s,a)$ and the $(1+\epsilon)$-th raw moment $[\EE_h R_h^{1+\epsilon}](s,a)$. Thus, we adopt a new approach to estimate $[\bnu_{1+\epsilon} R_h](s,a)$ directly, and bound the error by a novel perturbation analysis of adaptive Huber regression in Appendix~\ref{sec:huber-perturbation}.
\end{remark}

\begin{assumption}[Bounded cumulative rewards]\label{ass:bounded}
For any policy $\pi$, let $\{s_h,a_h,R_h\}_{h\in[H]}$ be a random trajectory following policy $\pi$. And define $r_\pi = \sum_{h=1}^H [\EE_h R_h](s_h,a_h)= \sum_{h=1}^H r_h(s_h, a_h)$.
We assume
(1) $0 \le r_\pi \le \cH$.
(2) $\sum_{h=1}^H [\bnu_{1+\epsilon}R_h]^\frac{2}{1+\epsilon}(s_{h},a_{h}) \le \cU$.
(3) $\Var(r_\pi) \le \cV$.
\end{assumption}

Here, (1) gives an upper bound of cumulative expected rewards $r_\pi$. (2) assumes the summation of $(1+\epsilon)$-th central moment of rewards $[\bnu_{1+\epsilon}R_h](s_{h},a_{h})$ is bounded since $[\sum_{h=1}^H [\bnu_{1+\epsilon}R_h](s_{h},a_{h})]^\frac{2}{1+\epsilon} \le \sum_{h=1}^H [\bnu_{1+\epsilon}R_h]^\frac{2}{1+\epsilon}(s_{h},a_{h}) \le \cU$ due to Jensen's inequality. And (3) is to bound the variance of $r_\pi$ along the trajectory following policy $\pi$.

\section{Adaptive Huber Regression}\label{sec:huber}
At the core of our algorithms for both heavy-tailed linear bandits and linear MDPs is a new approach -- adaptive Huber regression -- to handle heavy-tailed noise. \citet{sun2020adaptive} imposed adaptive Huber regression to handle i.i.d. heavy-tailed noise by utilizing Huber loss \citep{huber1964robust} as a surrogate of squared loss. \citet{li2023variance} modified adaptive Huber regression for heterogeneous online settings, where the variances in each round are different. However, it is not readily applicable to deal with heavy-tailed noise. Our contribution in this section is to construct a new self-normalized concentration inequality for general online regression problems with heavy-tailed noise.

We first give a brief introduction to Huber loss function and its properties. 

\begin{definition}[Huber loss]\label{def:huber-loss}
\emph{Huber loss} is defined as
\begin{equation}\label{eq:huber-loss}
    \ell_\tau(x) = \begin{cases}
    \frac{x^2}{2} & \text{if } |x| \le \tau,\\
    \tau|x| - \frac{\tau^2}{2} & \text{if } |x| > \tau,
    \end{cases}
\end{equation}
where $\tau>0$ is referred as a robustness parameter.
\end{definition}

Huber loss is first proposed by \citet{huber1964robust} as a robust version of squared loss while preserving the convex property. Specifically, 
Huber loss is a quadratic function of $x$ when $|x|$ is less than the threshold $\tau$, while becomes linearly dependent on $|x|$ when $|x|$ grows larger than $\tau$. It has the property of strongly convex near zero point and is not sensitive to outliers.
See Appendix~\ref{sec:huber-property} for more properties of Huber loss.

Next, we define general online regression problems with heavy-tailed noise, which include heavy-tailed linear bandits as a special case. Then we utilize Huber loss to estimate $\btheta^*$. Below we give the main theorem to bound the deviation of the estimated $\btheta_t$ in Algorithm~\ref{algo:huber} from the ground truth $\btheta^*$.

\begin{definition}
\label{def:regression}
    Let $\{\cF_t\}_{t\ge 1}$ be a filtration. For all $t>0$, let random variables $y_t, \varepsilon_t$ be $\cF_t$-measurable and random vector $\bphi_t \in \cB_d(L)$ be $\cF_{t-1}$-measurable.
    Suppose $y_t = \dotp{\bphi_t}{\btheta^*} + \varepsilon_t$, where $\btheta^* \in \cB_d(B)$ is an unknown coefficient and
    \[
    \EE[\varepsilon_t | \cF_{t-1}] = 0,\quad \EE[|\varepsilon_t|^{1+\epsilon} | \cF_{t-1}] = \nu_t^{1+\epsilon}
    \]
    for some $\epsilon\in(0,1]$.
    The goal is to estimate $\btheta^*$ at any round $t$ given the realizations of $\{\bphi_{s}, y_{s}\}_{s\in[t]}$.
\end{definition}

\begin{algorithm}[t]
	\caption{Adaptive Huber Regression}\label{algo:huber}
	\begin{algorithmic}[1]
	\REQUIRE Number of total rounds $T$, confidence level $\delta$, regularization parameter $\lambda$, $\sigmamin$, parameters for adaptive Huber regression $c_0,c_1,\tau_0$, estimated central moment $\hat{\nu}_t$ and moment parameter $b$ that satisfy $\nu_t / \hat{\nu}_t \le b \text{ for all } t \le T$.\label{eq:b}
    \ENSURE The estimated coefficient $\btheta_t$.
    \STATE $\kappa = d \log(1 + {T L^2}/({d\lambda \sigma_{\min}^2}))$.
    \STATE Set $\bH_{t-1} = \lambda \bI + \sum_{s=1}^{t-1} \sigma_s^{-2} \bphi_s \bphi_s^\top$.
    \STATE Set $\sigma_t = \max\cbr{\hat{\nu}_t, \sigma_{\mathrm{min}}, \frac{\|\bphi_t\|_{\bH_{t-1}^{-1}}}{c_0}, \frac{\sqrt{LB}}{c_1^{\frac{1}{4}} (2\kappa b^2)^{\frac{1}{4}}} \|\bphi_t\|_{\bH_{t-1}^{-1}}^{\frac{1}{2}}}$.\label{eq:sigma_t}
    \STATE Set
        $\tau_t = \tau_0 \frac{\sqrt{1+w_t^2}}{w_t} t^{\frac{1-\epsilon}{2(1+\epsilon)}} \text{ with } w_t = \|\bphi_t / \sigma_t\|_{\bH_{t-1}^{-1}}$.\label{eq:tau_t}
    \STATE Define the loss function $L_t(\btheta) := \frac{\lambda}{2}\|\btheta\|^2 + \sum_{s=1}^{t} \ell_{\tau_s}(\frac{y_s - \dotp{\bphi_s}{\btheta}}{\sigma_s})$.
    \STATE Compute $\btheta_t = \argmin_{\btheta \in \cB_d(B)} L_t(\btheta)$.\label{eq:theta_t}
	\end{algorithmic}
\end{algorithm}

\begin{theorem}
\label{thm:heavy}
For the online regression problems in Definition~\ref{def:regression}, we solve for $\btheta_t$ by adaptive Huber regression in Algorithm~\ref{algo:huber} with $c_0,c_1,\tau_0$ in Appendix~\ref{proof:heavy}.
Then for any $\delta\in(0,1)$, with probability at least $1- 3\delta$, for all $t \le T$, we have $\|\btheta_t - \btheta^*\|_{\bH_t} \le \beta_t$, where $\bH_t$ is defined in Algorithm~\ref{algo:huber} and
\begin{equation}\label{eq:beta_t}
\beta_t = 3 \sqrt{\lambda}B + 24 t^{\frac{1-\epsilon}{2(1+\epsilon)}} \sqrt{2\kappa} b (\log 3T)^{\frac{1-\epsilon}{2(1+\epsilon)}} (\log (2T^2/\delta))^{\frac{\epsilon}{1+\epsilon}}.
\end{equation}
\end{theorem}
\begin{proof}
To derive a tight high-probability bound, we take the most advantage of the properties of Huber loss.
A Chernoff bounding technique is used to bound the main error term, which requires a careful analysis of the moment generating function.
See Appendix~\ref{proof:heavy} for a detailed proof.
\end{proof}

We refer to the regression process in Line~\ref{eq:theta_t} of Algorithm~\ref{algo:huber} as \emph{adaptive Huber regression} in line with \citet{sun2020adaptive} to emphasize that the value of robustness parameter $\tau_t$ is chosen to adapt to data for a better trade-off between bias and robustness. Specifically, since we are in the online setting, $\bphi_t$ are dependent on $\{\bphi_s\}_{s<t}$, which is the key difference from the i.i.d. case in \citet{sun2020adaptive} where they set $\tau_t = \tau_0$, for all $t\le T$. Thus, as shown in Line~\ref{eq:tau_t} of Algorithm~\ref{algo:huber}, inspired by \citet{li2023variance}, we adjust $\tau_t$ according to the importance of observations $w_t = \|\bphi_t / \sigma_t\|_{\bH_{t-1}^{-1}}$, where $\sigma_t$ is specified below.
In the case where $\epsilon < 1$, different from \citet{li2023variance}, we first choose $\tau_t$ to be small for robust purposes, then gradually increase it with $t$ to reduce the bias.

Next, we illustrate the reason for setting $\sigma_t$ via Line~\ref{eq:sigma_t} of Algorithm~\ref{algo:huber}.
We use $\hat{\nu}_t \in \cF_{t-1}$ to estimate the central moment $\nu_t$ and use moment parameter $b$ to measure the closeness between $\hat{\nu}_t$ and $\nu_t$.
When we choose $\hat{\nu}_t$ as an upper bound of $\nu_t$, $b$ becomes a constant that equals to $1$.
And $\sigmamin$ is a small positive constant to avoid singularity.
The last two terms with respect to $c_0$ and $c_1$ are set according to the uncertainty $\|\bphi_t\|_{\bH_{t-1}^{-1}}$.
In addition, setting the parameter $c_0 \le 1$ yields $w_t \le 1$, which is essential to meet the condition of elliptical potential lemma \citep{abbasi2011improved}.

\begin{remark}
    The error bound $\beta_t$ in \eqref{eq:beta_t} is only related to the feature dimension $d$ and moment parameter $b$. While the Bernstein-style self-normalized concentration bounds \citep{zhou2021nearly,zhou2022computationally} depend on the magnitude of $\varepsilon_t$, thus cannot handle heavy-tailed errors.
\end{remark}

\section{Linear Bandits}\label{sec:bandit}

In this section, we show the algorithm $\algobandit$ in Algorithm~\ref{algo:bandit} for heavy-tailed linear bandits in Definition~\ref{def:bandit}. We first give a brief algorithm description, and then provide a theoretical regret analysis.

\begin{algorithm}[t]
	\caption{$\algobandit$}\label{algo:bandit}
	\begin{algorithmic}[1]
	\REQUIRE Number of total rounds $T$, confidence level $\delta$, regularization parameter $\lambda$, $\sigmamin$, parameters for adptive Huber regression $c_0, c_1, \tau_0$, confidence radius $\beta_t$.
    \STATE $\kappa = d \log(1 + \frac{T L^2}{d\lambda \sigma_{\min}^2})$, $\cC_0 = \cB_d(B)$, $\bH_0 = \lambda \bI$.
	\FOR{$t=1,\ldots, T$}
    \STATE Observe $\cD_t$.
    \STATE Set $(\bphi_t,\cdot) = \argmax_{\bphi \in \cD_t, \btheta\in\cC_{t-1}} \dotp{\bphi}{\btheta}$.
    \STATE Play $\bphi_t$ and observe $R_t,\nu_t$.
    \STATE Set $\sigma_t = \max\big\{\nu_t, \sigmamin, \frac{\|\bphi_t\|_{\bH_{t-1}^{-1}}}{c_0}, \frac{\sqrt{LB} }{c_1^{\frac{1}{4}} (2\kappa)^{\frac{1}{4}}} \|\bphi_t\|_{\bH_{t-1}^{-1}}^{\frac{1}{2}}\big\}$.
    \STATE Set $\tau_t = \tau_0 \frac{\sqrt{1+w_t^2}}{w_t} t^{\frac{1-\epsilon}{2(1+\epsilon)}}$ with $w_t = \|\bphi_t / \sigma_t\|_{\bH_{t-1}^{-1}}$.
    \STATE Update $\bH_t = \bH_{t-1} + \sigma_t^{-2}\bphi_t\bphi_t^\top$.
    \STATE Solve for $\btheta_t$ by Algorithm~\ref{algo:huber} and set $\cC_t = \{\btheta \in \RR^d | \|\btheta -\btheta_{t}\|_{\bH_{t}} \le \beta_{t}\}$.\label{eq:bandit-cC_t}
	\ENDFOR
	\end{algorithmic}
\end{algorithm}

\subsection{Algorithm Description}

$\algobandit$ follows the principle of Optimism in the Face of Uncertainty (OFU) \citep{abbasi2011improved}, and uses adaptive Huber regression in Section~\ref{sec:huber} to maintain a set $\cC_t$ that contains the unknown coefficient $\btheta^*$ with high probability.
Specifically, at the $t$-th round, $\algobandit$ estimates the expected reward of any arm $\bphi$ as $\max_{\btheta \in \cC_{t-1}} \dotp{\bphi}{\btheta}$, and selects the arm that maximizes the estimated reward.
The agent then receives the reward $R_t$ and updates the confidence set $\cC_{t}$ based on the information up to round $t$ with its center $\btheta_t$ computed by adaptive Huber regression as in Line~\ref{eq:bandit-cC_t} of Algorithm~\ref{algo:bandit}.

\subsection{Regret Analysis}

We next give the instance-dependent regret upper bound of $\algobandit$ in Theorem~\ref{thm:bandit-regret}.

\begin{theorem}\label{thm:bandit-regret}
    For the heavy-tailed linear bandits in Definition~\ref{def:bandit}, we set $c_0,c_1,\tau_0,\beta_t$ in Algorithm~\ref{algo:bandit} according to Theorem~\ref{thm:heavy} with $b=1$. Besides, let $\lambda = d / B^2$, and $\sigmamin = 1 / \sqrt{T}$. Then with probability at least $1-3\delta$, the regret of $\algobandit$ is bounded by
    \[
    \mathrm{Regret}(T) = \tilde{O}\rbr{d T^{\frac{1-\epsilon}{2(1+\epsilon)}} \sqrt{\sum\nolimits_{t = 1}^T \nu_t^2} + d T^{\frac{1-\epsilon}{2(1+\epsilon)}}}.
    \]
\end{theorem}
\begin{proof}
The proof uses the self-normalized concentration inequality of adaptive Huber regression and a careful analysis to bound the summation of bonuses.
See Appendix~\ref{proof:bandit-regret} for a detailed proof.
\end{proof}

\begin{remark}
\label{rem:bandit-regret}
Theorem~\ref{thm:bandit-regret} shows $\algobandit$ achieves an instance-dependent regret bound. When we assume $\nu_t, \forall t \ge 1$ have uniform upper bound $\nu$ (which can be treated as a constant), then the bound is reduced to $\tilde{O}(dT^{\frac{1}{1+\epsilon}})$. It matches the lower bound $\Omega(dT^{\frac{1}{1+\epsilon}})$ by \citet{shao2018almost} up to logarithmic factors. In the deterministic scenario, where $\epsilon=1$ and $\nu_t = 0$, for all $t \ge 1$, the bound is reduced to $\tilde{O}(d)$. It matches the lower bound $\Omega(d)$\footnote{Consider the decision set consisting of unit bases of $\mathbb{R}^d$. Given that each arm pull can only yield information about a single coordinate, it is inevitable that $d$ pulls are required for exploration.} up to logarithmic factors.
\end{remark}

\section{Linear MDPs}\label{sec:mdp}

In this section, we show the algorithm $\algomdp$ in Algorithm~\ref{algo:mdp} for linear MDP with heavy-tailed rewards defined in Section~\ref{sec:pre-mdp}. Let $\bphi_{k,h} := \bphi(s_{k,h}, a_{k,h})$ for short. We first give the algorithm description intuitively, then provide the computational complexity and regret bound. 

\begin{algorithm}[t!]
\caption{$\algomdp$}\label{algo:mdp}
\begin{algorithmic}[1]
\REQUIRE Number of episodes $K$, confidence level $\delta$, regularization parameter $\lambda_R,\lambda_V$, $\numin,\sigmamin$, confidence radius $\beta_{R^\epsilon}, \beta_0, \beta_{R}, \beta_V$.
\STATE $\kappa = d \log(1 + \frac{K}{d\lambda_R \numin^2})$.
\STATE $\btheta_{0,h}=\hbw_{0,h}=\cbw_{0,h}=\zero$, $\bH_{0,h} = \lambda_R \bI, \bSigma_{0,h} = \lambda_V \bI$, UPDATE = TRUE.
\FOR{$k=1,\ldots,K$}
\STATE $V^k_{H+1}(\cdot) = \pesV^k_{H+1}(\cdot) = 0$.
\FOR{$h=H,\ldots,1$}
\STATE Compute $\btheta_{k-1,h}$, $\hbw_{k-1,h}$ and $\cbw_{k-1,h}$ via \eqref{eq:mdp-btheta} and \eqref{eq:mdp-bw}. \label{line:thetakh}
\IF{UPDATE}\label{eq:rare-switching-start}
\STATE $Q^k_h(\cdot,\cdot) = \dotp{\bphi(\cdot,\cdot)}{\btheta_{k-1,h} + \hbw_{k-1,h}} + \beta_{R,k-1}\|\bphi(\cdot,\cdot)\|_{\bH_{k-1,h}^{-1}} + \beta_V\|\bphi(\cdot,\cdot)\|_{\bSigma_{k-1,h}^{-1}}$.
\STATE $\pesQ^k_h(\cdot,\cdot) = \dotp{\bphi(\cdot,\cdot)}{\btheta_{k-1,h} + \cbw_{k-1,h}} - \beta_{R,k-1}\|\bphi(\cdot,\cdot)\|_{\bH_{k-1,h}^{-1}} - \beta_V\|\bphi(\cdot,\cdot)\|_{\bSigma_{k-1,h}^{-1}}$.
\STATE $Q^k_h(\cdot,\cdot) = \min\{Q^k_h(\cdot,\cdot), Q^{k-1}_h(\cdot,\cdot), \cH\}, \pesQ^k_h(\cdot,\cdot) = \max\{\pesQ^k_h(\cdot,\cdot), \pesQ^{k-1}_h(\cdot,\cdot), 0\}$.\label{line:Qkh}
\STATE Set $\kl = k$.
\ELSE
\STATE $Q^k_h(\cdot,\cdot)=Q^{k-1}_h(\cdot,\cdot)$, $\pesQ^k_h(\cdot,\cdot)=\pesQ^{k-1}_h(\cdot,\cdot)$.
\ENDIF
\STATE $V^k_h(\cdot) = \max_a Q^k_h(\cdot,a)$, $\pesV^k_h(\cdot) = \max_a \pesQ^k_h(\cdot,a)$, $\pi^k_h(\cdot) = \argmax_a Q^k_h(\cdot, a)$.\label{eq:rare-switching-end}
\ENDFOR
\STATE Observe initial state $s_{k,1}$.
\FOR{$h=1,\ldots,H$}
\STATE Take action $a_{k,h}=\pi^k_{h}(s_{k,h})$ and observe $R_{k,h}, s_{k,h+1}$.\label{line:akh}
\STATE Set $\nu_{k,h}$ and $\sigma_{k,h}$ according to \eqref{eq:mdp-nu} and \eqref{eq:mdp-sigma} respectively.
\STATE Set $\tau_{k,h} = \tau_0 \frac{\sqrt{1+w_{k,h}^2}}{w_{k,h}} k^{\frac{1-\epsilon}{2(1+\epsilon)}}$, $\ttau_{k,h} = \ttau_0 \frac{\sqrt{1+w_{k,h}^2}}{w_{k,h}} k^{\frac{1-\epsilon'}{2(1+\epsilon')}}$ with $w_{k,h} = \|\bphi_{k,h} / \nu_{k,h}\|_{\bH_{k,h}^{-1}}$.
\STATE Update $\bH_{k,h} = \bH_{k-1,h} + \frac{1}{\nu_{k,h}^2}\bphi_{k,h}\bphi_{k,h}^\top$ and $\bSigma_{k,h} = \bSigma_{k-1,h} + \frac{1}{\sigma_{k,h}^2}\bphi_{k,h}\bphi_{k,h}^\top$.
\ENDFOR
\IF{$\exists h'\in[H]$ such that $\det(\bH_{k,h'}) \ge 2 \det(\bH_{\kl,h'})$ or $\det(\bSigma_{k,h'}) \ge 2 \det(\bSigma_{\kl,h'})$}
\STATE Set UPDATE = TRUE.
\ELSE
\STATE Set UPDATE = FALSE.
\ENDIF
\ENDFOR
\end{algorithmic}
\end{algorithm}

\subsection{Algorithm Description}

$\algomdp$ features a novel combination of adaptive Huber regression in Section~\ref{sec:huber} and existing algorithmic frameworks for linear MDPs with bounded rewards \citep{jin2020provably,he2023nearly}.
At a high level, $\algomdp$ employs separate estimation techniques to handle heavy-tailed rewards and transition kernels.
Specifically, we utilize adaptive Huber regression proposed in Section~\ref{sec:huber} to estimate heavy-tailed rewards and weighted ridge regression \citep{zhou2021nearly,he2023nearly} to estimate the expected next-state value functions.
Then, it follows the value iteration scheme to update the optimistic and pessimistic estimation of the optimal value function $Q^k_{h}$, $V^k_h$ and $\pesQ^k_h$, $\pesV^k_h$, respectively, via a rare-switching policy as in Line~\ref{eq:rare-switching-start} to \ref{eq:rare-switching-end} of Algorithm~\ref{algo:mdp}.
We highlight the key steps of $\algomdp$ as follows. 

\paragraph{Estimation for expected heavy-tailed rewards}

Since the expected rewards have linear structure in linear MDPs, i.e., $r_h(s,a) = \dotp{\bphi(s,a)}{\btheta_h^*}$, we use adaptive Huber regression to estimate $\btheta_h^*$:
\begin{equation}\label{eq:mdp-btheta}
    \btheta_{k,h} = \argmin_{\btheta \in \cB_d(B)} \frac{\lambda_R}{2}\|\btheta\|^2 + \sum_{i=1}^{k} \ell_{\tau_{i,h}}\rbr{\frac{R_{i,h} - \dotp{\bphi_{i,h}}{\btheta}}{\nu_{i,h}}},
\end{equation}
where $\nu_{i,h}$ will be specified later.

\paragraph{Estimation for central moment of rewards}

By Assumption~\ref{ass:central-moment}, the $(1+\epsilon)$-th central moment of rewards is linear in $\bphi$, i.e., $[\bnu_{1+\epsilon} R_h](s,a) = \dotp{\bphi(s,a)}{\bpsi_h^*}$. Motivated by this, we estimate $\bpsi_h^*$ by adaptive Huber regression as
\begin{equation}\label{eq:mdp-bpsi}
\bpsi_{k,h} = \argmin_{\bpsi \in \cB_d(W)} \frac{\lambda_R}{2}\|\bpsi\|^2 + \sum_{i=1}^{k} \ell_{\ttau_{i,h}}\rbr{\frac{|\varepsilon_{i,h}|^{1+\epsilon} - \dotp{\bphi_{i,h}}{\bpsi}}{\nu_{i,h}}},
\end{equation}
where $W$ is the upper bound of $\|\bpsi_h^*\|$ defined in Assumption~\ref{ass:central-moment}. Since $\varepsilon_{i,h}$ is intractable, we estimate it by $\hat{\varepsilon}_{i,h} = R_{i,h} - \dotp{\bphi_{i,h}}{\btheta_{i,h}}$, which gives $\hbpsi_{k,h}$ as
\begin{equation}\label{eq:mdp-hbpsi}
\hbpsi_{k,h} = \argmin_{\bpsi \in \cB_d(W)} \frac{\lambda_R}{2}\|\bpsi\|^2 + \sum_{i=1}^{k} \ell_{\ttau_{i,h}}\rbr{\frac{|\hat{\varepsilon}_{i,h}|^{1+\epsilon} - \dotp{\bphi_{i,h}}{\bpsi}}{\nu_{i,h}}}.
\end{equation}
The inevitable error between $\bpsi_{k,h}$ and $\hbpsi_{k,h}$ can be quantified by a novel perturbation analysis of adaptive Huber regression in Appendix~\ref{sec:huber-perturbation}.

We then set the weight $\nu_{k,h}$ for adaptive Huber regression as
\begin{equation}\label{eq:mdp-nu}
\nu_{k,h} = \max\cbr{\hat{\nu}_{k,h}, \numin, \frac{\normbphiR{k}{h}}{c_0}, \frac{\sqrt{\max\{B,W\}}}{c_1^{\frac{1}{4}} (2\kappa)^{\frac{1}{4}}} \normbphiR{k}{h}^{\frac{1}{2}}},
\end{equation}
where $\numin$ is a small positive constant to avoid the singularity, $\hat{\nu}_{k,h}^{1+\epsilon} = [\hat{\bnu}_{1+\epsilon} R_h](s_{k,h},a_{k,h}) + W_{k,h}$
is a high-probability upper bound of rewards' central moment $[\bnu_{1+\epsilon} R_h](s_{k,h},a_{k,h})$ with 
\begin{equation}\label{eq:mdp-hatbnu}
[\hat{\bnu}_{1+\epsilon} R_h](s_{k,h},a_{k,h}) = \dotp{\bphi_{k,h}}{\hbpsi_{k-1,h}},
\end{equation}
\begin{equation}\label{eq:mdp-W}
W_{k,h} = (\beta_{R^\epsilon,k-1} + 6 \cH^\epsilon \beta_{R,k-1} \kappa) \normbphiR{k}{h},
\end{equation}
where $\kappa$ is defined in Algorithm~\ref{algo:mdp}, $\beta_{R^\epsilon,k} = \tilde{O}(\sqrt{d}\nu_{R^\epsilon}k^\frac{1-\epsilon'}{2(1+\epsilon')}/\numin)$ and $\beta_{R,k} = \tilde{O}(\sqrt{d}k^\frac{1-\epsilon}{2(1+\epsilon)})$.

\paragraph{Estimation for expected next-state value functions}
For any value function $f: \cS \to \RR$, we define the following notations for simplicity:
\begin{equation}\label{eq:bwf}
\bw_{h}[f] = \int_{s\in\cS} \bmu^*_h(s)f(s)\ud s,\quad
\hbw_{k,h}[f] = \bSigma_{k,h}^{-1} \sum_{i=1}^k \sigma_{i,h}^{-2} \bphi_{i,h} f(s_{i,h+1}),
\end{equation}
where $\sigma_{i,h}$ will be specified later. Note for any state-action pair $(s,a)\in\cS\times\cA$, by linear structure of transition probabilities, we have
$$
[\PP_h f](s,a) = \int_{s'\in\cS} \dotp{\bphi(s,a)}{\bmu^*_h(s')} f(s') \ud s' = \dotp{\bphi(s,a)}{\bw_{h}[f]}.
$$
In addition, for any $f,g: \cS \to \RR$, it holds that $\bw_{h}[f+g] = \bw_{h}[f] + \bw_{h}[g]$ and $\hbw_{k,h}[f+g] = \hbw_{k,h}[f] + \hbw_{k,h}[g]$ due to the linear property of integration and ridge regression.

We remark $\hbw_{k,h}[f]$ is the estimation of $\bw_{h}[f]$ by weighted ridge regression on $\{\bphi_{i,h}, f(s_{i,h+1})\}_{i\in[k]}$.
And we estimate the coefficients $\hbw_{k,h},\cbw_{k,h},\tbw_{k,h}$
\begin{equation}\label{eq:mdp-bw}
\hbw_{k,h} = \hbw_{k,h}[V^k_{h+1}],\quad
\cbw_{k,h} = \hbw_{k,h}[\pesV^k_{h+1}],\quad
\tbw_{k,h} = \hbw_{k,h}[(V^k_{h+1})^2],
\end{equation}
where $V^k_h$ and $\pesV^k_h$ are optimistic and pessimistic estimation of the optimal value functions.

\paragraph{Estimation for variance of next-state value functions}

Inspired by \citet{he2023nearly}, we set the weight $\sigma_{k,h}$ for weighted ridge regression in \eqref{eq:bwf} as
\begin{equation}\label{eq:mdp-sigma}
\sigma_{k,h} = \max\cbr{\hat{\sigma}_{k,h}, \sqrt{d^3H D_{k,h}}, \sigmamin, \normbphiV{k}{h}, \sqrt{d^\frac{5}{2}H\cH} \normbphiV{k}{h}^{\frac{1}{2}}},
\end{equation}
where $\sigmamin$ is a small constant to avoid singularity, $\hat{\sigma}_{k,h}^2 = [\hat{\VV}_h V^k_{h+1}](s_{k,h}, a_{k,h}) + E_{k,h}$ with
\begin{equation}\label{eq:mdp-hatVV}
\sbr{\hat{\VV}_h V^k_{h+1}} (s_{k,h}, a_{k,h}) = \dotp{\bphi_{k,h}}{\tbw_{k-1,h}}_{[0,\cH^2]} - \dotp{\bphi_{k,h}}{\hbw_{k-1,h}}_{[0,\cH]}^2,
\end{equation}
\begin{align}
E_{k,h} &= \min\cbr{4\cH \dotp{\bphi_{k,h}}{\hbw_{k-1,h} - \cbw_{k-1,h}} + 11\cH\beta_0\normbphiV{k}{h}, \cH^2}.\label{eq:mdp-E}\\
D_{k,h} &= \min\cbr{2\cH \dotp{\bphi_{k,h}}{\hbw_{k-1,h} - \cbw_{k-1,h}} + 4\cH\beta_0\normbphiV{k}{h}, \cH^2},\label{eq:mdp-D}
\end{align}
where $\beta_0 = \tilde{O}(\sqrt{d^3H\cH^2}/\sigmamin)$. Here $\hat{\sigma}_{k,h}^2$ and $D_{k,h}$ are
upper bounds of $[\VV_h V^*_{h+1}](s_{k,h},a_{k,h})$ and $\max\{[\VV_h\rbr{V^k_{h+1} - V^*_{h+1}}](s_{k,h},a_{k,h}), [\VV_h\rbr{V^*_{h+1} - \pesV^k_{h+1}}](s_{k,h},a_{k,h})\}$, respectively.

\subsection{Computational Complexity}
\label{sec:computational-complexity}

\begin{theorem}
\label{thm:computational-complexity}
For the linear MDPs with heavy-tailed rewards defined in Section~\ref{sec:pre-mdp}, the computational complexity of $\algomdp$ is $\tilde{O}(d^4|\mathcal{A}|H^3K + HK\mathcal{R})$. Here $\mathcal{R}$ is the cost of the optimization algorithm for solving adaptive Huber regression in \eqref{eq:mdp-btheta}. Furthermore, we can specialize $\mathcal{R}$ by adopting the Nesterov accelerated method, which gives $\mathcal{R}=\tilde{O}(d+d^{-\frac{1-\epsilon}{2(1+\epsilon)}} H^{\frac{1-\epsilon}{2(1+\epsilon)}} K^{\frac{1+2\epsilon}{2(1+\epsilon)}})$.
\end{theorem}
\begin{proof}
See Appendix~\ref{proof:computational-complexity} for a detailed proof.
\end{proof}

Such a complexity allows us to focus on the complexity introduced by the RL algorithm rather than the optimization subroutine for solving adaptive Huber regression. Compared to that of LSVI-UCB++ \citep{he2023nearly}, $\tilde{O}(d^4|\mathcal{A}|H^3K)$, the extra term $\tilde{O}(HK\mathcal{R})$ causes a slightly worse computational time in terms of $K$. This is due to the absence of a closed-form solution of adaptive Huber regression in \eqref{eq:mdp-btheta}. Thus extra optimization steps are unavoidable. Nevertheless, Nesterov accelerated method gives $\mathcal{R}=\tilde{O}\rbr{K^{\frac{1+2\epsilon}{2(1+\epsilon)}}}$ with respect to $K$, which implies the computational complexity of $\algomdp$ is better than that of LSVI-UCB \citep{jin2020provably}, $\tilde{O}(d^2|\mathcal{A}|HK^2)$ in terms of $K$, thanks to the rare-switching updating policy. We conduct numerical experiments in Appendix~\ref{sec:exp} to further corroborate the computational efficiency of adaptive Huber regression.

\subsection{Regret Bound}\label{sec:mdp-result}

\ifarxiv
\begin{theorem}\label{thm:mdp-regret}
    For the time-inhomogeneous linear MDPs with heavy-tailed rewards defined in Section~\ref{sec:pre-mdp}, we set parameters in Algorithm~\ref{algo:mdp} as follows: $\lambda_R = d/\max\cbr{B^2,W^2}$, $\lambda_V = 1/\cH^2$,
    $\numin$ in Lemma~\ref{lem:mdp-sumR}, $\sigmamin$ in Lemma~\ref{lem:mdp-sumV},
    $\beta_{R^\epsilon}$, $\beta_0$, $\beta_R$, $\beta_V$ in \eqref{eq:mdp-beta_Reps}, \eqref{eq:mdp-beta_0}, \eqref{eq:mdp-beta_R}, \eqref{eq:mdp-beta_V}, respectively.
    Then for any $\delta\in(0,1)$, with probability at least $1-16\delta$, the regret of $\algomdp$ is bounded by
    \begin{align*}
    \mathrm{Regret}(K) = \tilde{O}\bigg(&d\sqrt{H\cU^*}K^\frac{1}{1+\epsilon} + d\sqrt{H\cV^*K} + \nu_{R^\epsilon}^\frac{1}{1+\epsilon} d^\frac{2+\epsilon}{1+\epsilon} H^\frac{3+\epsilon}{2(1+\epsilon)} K^\frac{2+(1-\epsilon)(1+\epsilon')}{2(1+\epsilon)(1+\epsilon')}\\
    & + \sqrt{d^7 H^{3.5} \cH^2} K^\frac{1}{4} + \rbr{d^{2+\epsilon} H^\frac{1+\epsilon}{2} \cH^\epsilon K^\frac{1-\epsilon}{2}}^\frac{1}{\epsilon}\\
    & + (\nu_R+\max\cbr{B,W})dHK^\frac{1-\epsilon}{2(1+\epsilon)} + d^5H^3\cH\bigg),
    \end{align*}
    where
    \begin{align}
    \cU^* &= \min\cbr{\cU^*_0,\cU}\text{ with }\cU^*_0 = \frac{1}{K}\sum_{k=1}^K \EE_{(s_h,a_h)\sim d_h^{\pi^k}} \sbr{\sum_{h=1}^H \sbr{\bnu_{1+\epsilon} R_h}^\frac{2}{1+\epsilon}(s_h,a_h)},\label{eq:U}\\
    \cV^* &= \min\cbr{\cV^*_0,\cV}\text{ with }\cV^*_0 = \frac{1}{K}\sum_{k=1}^K \EE_{(s_h,a_h)\sim d_h^{\pi^k}} \sbr{\sum_{h=1}^H \sbr{\VV_h V^*_{h+1}}(s_h,a_h)},\label{eq:V}\\
    &\qquad\qquad\qquad\,\, d_h^{\pi^k}(s, a) = \PP^{\pi^k}((s_h,a_h)=(s,a)|s_0=s_{k,1})\label{eq:d}
    \end{align}
    and $\cH, \cU, \cV$ are defined in Assumption~\ref{ass:bounded}.
\end{theorem}
\begin{proof}
See Appendix~\ref{sec:mdp-regret} for a detailed proof.
\end{proof}

\begin{remark}
If $\epsilon > \frac{1}{1+\epsilon'}$, then $K^\frac{2+(1-\epsilon)(1+\epsilon')}{2(1+\epsilon)(1+\epsilon')} < \sqrt{K}$. When the number of episodes $K$ is sufficiently large, the regret can be simplified to 
\[
\tilde{O}\rbr{d\sqrt{H\cU^*}K^\frac{1}{1+\epsilon} + d\sqrt{H\cV^*K}}.
\]
\end{remark}
\else
\begin{theorem}[Informal]\label{thm:mdp-regret}
    For the linear MDPs with heavy-tailed rewards defined in Section~\ref{sec:pre-mdp}, we set parameters in Algorithm~\ref{algo:mdp} as follows: $\lambda_R = d/\max\cbr{B^2,W^2}$, $\lambda_V = 1/\cH^2$, $\numin$, $\sigmamin$, $c_0$, $c_1$, $\tau_0$, $\ttau_0$, $\beta_{R^\epsilon}$, $\beta_0$, $\beta_R$, $\beta_V$ in Appendix~\ref{sec:mdp-events}. Then for any $\delta\in(0,1)$, with probability at least $1-16\delta$, the regret of $\algomdp$ is bounded by
    \[
    \mathrm{Regret}(K) = \tilde{O}\rbr{d\sqrt{H\cU^*}K^\frac{1}{1+\epsilon} + d\sqrt{H\cV^*K}},
    \]
    where $\epsilon\in(0,1]$, $\cU^* = \min\{\cU^*_0,\cU\}$, $\cV^* = \min\{\cV^*_0,\cV\}$ with $\cU^*_0$, $\cV^*_0$ defined in Appendix~\ref{sec:mdp-regret} and $\cH, \cU, \cV$ defined in Assumption~\ref{ass:bounded}.
\end{theorem}
\begin{proof}
See Appendix~\ref{sec:mdp-regret} for a formal version of Theorem~\ref{thm:mdp-regret} and its detailed proof.
\end{proof}
\fi

\paragraph{Quantities $\cU^*$, $\cV^*$}
We make a few explanations for the quantities $\cU^*$, $\cV^*$.
On one hand, $\cU^*$ is upper bounded by $\cU$, which is the upper bound of the sum of the $(1+\epsilon)$-th central moments of reward functions along a single trajectory. On the other hand, $\cU^*$ is no more than $\cU^*_0$, which is the sum of the $(1+\epsilon)$-th central moments with respect to the averaged occupancy measure of the first $K$ episodes.
$\cV^*$ is defined similar to $\cU^*$, but measures the randomness of transition probabilities.

\begin{remark}\label{rem:mdp-li2023variance}
When $\epsilon=1$, we can show that this regret is bounded by $\tilde{O}(d\sqrt{H\cG^*K})$, where $\cG^*$ is an variance-dependent quantity defined by \citet{li2023variance}. Thus, our result recovers their variance-aware regret bound. See Remark~\ref{rem:mdp-li2023variance-appendix} in Appendix~\ref{sec:mdp-regret} for a detailed proof.
\end{remark}
To demonstrate the optimality of our results and establish connections with existing literature, we can specialize Theorem~\ref{thm:mdp-regret} to obtain the worst-case regret \citep{jin2020provably,agarwal2023vo,he2023nearly} and first-order regret \citep{wagenmaker2022first}.

\begin{corollary}[Worst-case regret]
For the linear MDPs with heavy-tailed rewards defined in Section~\ref{sec:pre-mdp} and for any $\delta\in(0,1)$, with probability at least $1-16\delta$, the regret of $\algomdp$ is bounded by
\[
\tilde{O}(dHK^\frac{1}{1+\epsilon} + d \sqrt{H^3 K}).
\]
\end{corollary}
\begin{proof}
    Notice $\cU^*$ and $\cV^*$ are upper bounded by $H\nu_R^2$ and $\cH^2$ (total variance lemma in \citet{jin2018q}) respectively. When $\cH = H$, and 
    we treat $\nu_R$ as a constant, the result follows.
\end{proof}
Next, we give the regret lower bound of linear MDPs with heavy-tailed rewards in Theorem~\ref{thm:lower-bound-heavy}, which shows our proposed $\algomdp$ is minimax optimal in the worst case.

\begin{theorem}\label{thm:lower-bound-heavy}
    For any algorithm, there exists an $H$-episodic, $d$-dimensional linear MDP with heavy-tailed rewards such that for any $K$, the algorithm's regret is
    \[
        \Omega(d H K^{\frac{1}{1+\epsilon}} + d \sqrt{H^3 K}).
    \]
\end{theorem}
\begin{proof}
Intuitively, the proof of Theorem~\ref{thm:lower-bound-heavy} follows from a combination of the lower bound constructions for heavy-tailed linear bandits in \citet{shao2018almost} and linear MDPs in \citet{zhou2021nearly}. See Appendix~\ref{proof:lower-bound-heavy} for a detailed proof.
\end{proof}

Theorem~\ref{thm:lower-bound-heavy} shows that for sufficiently large $K$, the reward term dominates in the regret bound. Thus, in heavy-tailed settings, the main difficulty is learning the reward functions.

\begin{corollary}[First-order regret]\label{cor:mdp-first-order}
For the linear MDPs with heavy-tailed rewards defined in Section~\ref{sec:pre-mdp} and for any $\delta\in(0,1)$, with probability at least $1-16\delta$, the regret of $\algomdp$ is bounded by
\[
\tilde{O}(d\sqrt{H\cU^*}K^\frac{1}{1+\epsilon} + d\sqrt{H \cH V_1^* K}).
\]
And when the rewards are uniformly bounded in $[0,1]$, the result is reduced to the first-order regret bound of $\tilde{O}(d\sqrt{H^2 V_1^* K})$.
\end{corollary}
\begin{proof}
    See Section~\ref{proof:mdp-first-order} for a detailed proof.
\end{proof}

Our first-order regret $\tilde{O}(d\sqrt{H^2 V_1^* K})$ is minimax optimal in the worst case since $V^*_1 \le H$. And it improves the state-of-the-art result $\tilde{O}(d\sqrt{H^3 V_1^* K})$ \citep{li2023variance} by a factor of $\sqrt{H}$.

\section{Conclusion}\label{sec:conclu}

In this work, we propose two computationally efficient algorithms for heavy-tailed linear bandits and linear MDPs, respectively.
Our proposed algorithms, termed as $\algobandit$ and $\algomdp$, are based on a novel self-normalized concentration inequality for adaptive Huber regression, which may be of independent interest.
$\algobandit$ and $\algomdp$ achieve minimax optimal and instance-dependent regret bounds scaling with the central moments.
We also provide a lower bound for linear MDPs with heavy-tailed rewards to demonstrate the optimality of $\algomdp$.
To the best of our knowledge, we are the first to study heavy-tailed rewards in RL with function approximation and provide a new algorithm for this setting which is both statistically and computationally efficient.

\section*{Acknowledgments}
Liwei Wang is supported by National Key R\&D Program of China (2022ZD0114900) and National
Science Foundation of China (NSFC62276005).
Lin F. Yang is supported in part by NSF grant 2221871, and an Amazon Research Grant.

\bibliographystyle{plainnat}
\bibliography{rl_ref}

\begin{thebibliography}{44}
\providecommand{\natexlab}[1]{#1}
\providecommand{\url}[1]{\texttt{#1}}
\expandafter\ifx\csname urlstyle\endcsname\relax
  \providecommand{\doi}[1]{doi: #1}\else
  \providecommand{\doi}{doi: \begingroup \urlstyle{rm}\Url}\fi

\bibitem[Abbasi-Yadkori et~al.(2011)Abbasi-Yadkori, P{\'a}l, and Szepesv{\'a}ri]{abbasi2011improved}
Yasin Abbasi-Yadkori, D{\'a}vid P{\'a}l, and Csaba Szepesv{\'a}ri.
\newblock Improved algorithms for linear stochastic bandits.
\newblock In \emph{Advances in Neural Information Processing Systems}, volume~24, 2011.

\bibitem[Agarwal et~al.(2023)Agarwal, Jin, and Zhang]{agarwal2023vo}
Alekh Agarwal, Yujia Jin, and Tong Zhang.
\newblock {VO$Q$L}: Towards optimal regret in model-free rl with nonlinear function approximation.
\newblock In \emph{The Thirty Sixth Annual Conference on Learning Theory}, pages 987--1063. PMLR, 2023.

\bibitem[Auer et~al.(2008)Auer, Jaksch, and Ortner]{auer2008near}
Peter Auer, Thomas Jaksch, and Ronald Ortner.
\newblock Near-optimal regret bounds for reinforcement learning.
\newblock \emph{Advances in neural information processing systems}, 21, 2008.

\bibitem[Ayoub et~al.(2020)Ayoub, Jia, Szepesvari, Wang, and Yang]{ayoub2020model}
Alex Ayoub, Zeyu Jia, Csaba Szepesvari, Mengdi Wang, and Lin Yang.
\newblock Model-based reinforcement learning with value-targeted regression.
\newblock In \emph{International Conference on Machine Learning}, pages 463--474. PMLR, 2020.

\bibitem[Azar et~al.(2017)Azar, Osband, and Munos]{azar2017minimax}
Mohammad~Gheshlaghi Azar, Ian Osband, and R{\'e}mi Munos.
\newblock Minimax regret bounds for reinforcement learning.
\newblock In \emph{International Conference on Machine Learning}, pages 263--272. PMLR, 2017.

\bibitem[Bhatt et~al.(2022)Bhatt, Fang, Li, and Samorodnitsky]{bhatt2022nearly}
Sujay Bhatt, Guanhua Fang, Ping Li, and Gennady Samorodnitsky.
\newblock Nearly optimal catoni’s m-estimator for infinite variance.
\newblock In \emph{International Conference on Machine Learning}, pages 1925--1944. PMLR, 2022.

\bibitem[Bubeck et~al.(2013)Bubeck, Cesa-Bianchi, and Lugosi]{bubeck2013bandits}
S{\'e}bastien Bubeck, Nicolo Cesa-Bianchi, and G{\'a}bor Lugosi.
\newblock Bandits with heavy tail.
\newblock \emph{IEEE Transactions on Information Theory}, 59\penalty0 (11):\penalty0 7711--7717, 2013.

\bibitem[Bubeck et~al.(2015)]{bubeck2015convex}
S{\'e}bastien Bubeck et~al.
\newblock Convex optimization: Algorithms and complexity.
\newblock \emph{Foundations and Trends{\textregistered} in Machine Learning}, 8\penalty0 (3-4):\penalty0 231--357, 2015.

\bibitem[Catoni(2012)]{catoni2012challenging}
Olivier Catoni.
\newblock Challenging the empirical mean and empirical variance: a deviation study.
\newblock In \emph{Annales de l'IHP Probabilit{\'e}s et statistiques}, volume~48, pages 1148--1185, 2012.

\bibitem[Chen et~al.(2021)Chen, Jin, Li, and Xu]{chen2021generalized}
Peng Chen, Xinghu Jin, Xiang Li, and Lihu Xu.
\newblock A generalized catoni’s m-estimator under finite $\alpha$-th moment assumption with $\alpha \in$(1, 2).
\newblock \emph{Electronic Journal of Statistics}, 15\penalty0 (2):\penalty0 5523--5544, 2021.

\bibitem[Choi et~al.(2020)Choi, Mela, Balseiro, and Leary]{choi2020online}
Hana Choi, Carl~F Mela, Santiago~R Balseiro, and Adam Leary.
\newblock Online display advertising markets: A literature review and future directions.
\newblock \emph{Information Systems Research}, 31\penalty0 (2):\penalty0 556--575, 2020.

\bibitem[Cont(2001)]{cont2001empirical}
Rama Cont.
\newblock Empirical properties of asset returns: stylized facts and statistical issues.
\newblock \emph{Quantitative finance}, 1\penalty0 (2):\penalty0 223, 2001.

\bibitem[Freedman(1975)]{freedman1975tail}
David~A Freedman.
\newblock On tail probabilities for martingales.
\newblock \emph{the Annals of Probability}, pages 100--118, 1975.

\bibitem[Hamza and Krim(2001)]{hamza2001image}
A~Ben Hamza and Hamid Krim.
\newblock Image denoising: A nonlinear robust statistical approach.
\newblock \emph{IEEE transactions on signal processing}, 49\penalty0 (12):\penalty0 3045--3054, 2001.

\bibitem[He et~al.(2023)He, Zhao, Zhou, and Gu]{he2023nearly}
Jiafan He, Heyang Zhao, Dongruo Zhou, and Quanquan Gu.
\newblock Nearly minimax optimal reinforcement learning for linear markov decision processes.
\newblock In \emph{International Conference on Machine Learning}, pages 12790--12822. PMLR, 2023.

\bibitem[Hu et~al.(2022)Hu, Chen, and Huang]{hu2022nearly}
Pihe Hu, Yu~Chen, and Longbo Huang.
\newblock Nearly minimax optimal reinforcement learning with linear function approximation.
\newblock In \emph{International Conference on Machine Learning}, pages 8971--9019, 2022.

\bibitem[Huber(1964)]{huber1964robust}
Peter~J Huber.
\newblock Robust estimation of a location parameter.
\newblock \emph{The Annals of Mathematical Statistics}, pages 73--101, 1964.

\bibitem[Hull(2012)]{hull2012risk}
John Hull.
\newblock \emph{Risk management and financial institutions,+ Web Site}, volume 733.
\newblock John Wiley \& Sons, 2012.

\bibitem[Jebarajakirthy et~al.(2021)Jebarajakirthy, Maseeh, Morshed, Shankar, Arli, and Pentecost]{jebarajakirthy2021mobile}
Charles Jebarajakirthy, Haroon~Iqbal Maseeh, Zakir Morshed, Amit Shankar, Denni Arli, and Robin Pentecost.
\newblock Mobile advertising: A systematic literature review and future research agenda.
\newblock \emph{International Journal of Consumer Studies}, 45\penalty0 (6):\penalty0 1258--1291, 2021.

\bibitem[Jia et~al.(2020)Jia, Yang, Szepesvari, and Wang]{jia2020model}
Zeyu Jia, Lin Yang, Csaba Szepesvari, and Mengdi Wang.
\newblock Model-based reinforcement learning with value-targeted regression.
\newblock In \emph{Learning for Dynamics and Control}, pages 666--686. PMLR, 2020.

\bibitem[Jin et~al.(2018)Jin, Allen-Zhu, Bubeck, and Jordan]{jin2018q}
Chi Jin, Zeyuan Allen-Zhu, Sebastien Bubeck, and Michael~I Jordan.
\newblock Is q-learning provably efficient?
\newblock \emph{Advances in neural information processing systems}, 31, 2018.

\bibitem[Jin et~al.(2020)Jin, Yang, Wang, and Jordan]{jin2020provably}
Chi Jin, Zhuoran Yang, Zhaoran Wang, and Michael~I Jordan.
\newblock Provably efficient reinforcement learning with linear function approximation.
\newblock In \emph{Conference on Learning Theory}, pages 2137--2143. PMLR, 2020.

\bibitem[Kim et~al.(2022)Kim, Yang, and Jun]{kim2022improved}
Yeoneung Kim, Insoon Yang, and Kwang-Sung Jun.
\newblock Improved regret analysis for variance-adaptive linear bandits and horizon-free linear mixture mdps.
\newblock \emph{Advances in Neural Information Processing Systems}, 35:\penalty0 1060--1072, 2022.

\bibitem[Kirschner and Krause(2018)]{kirschner2018information}
Johannes Kirschner and Andreas Krause.
\newblock Information directed sampling and bandits with heteroscedastic noise.
\newblock In \emph{Conference On Learning Theory}, pages 358--384. PMLR, 2018.

\bibitem[Li et~al.(2023)Li, Cai, Chen, Wei, and Chi]{li2023q}
Gen Li, Changxiao Cai, Yuxin Chen, Yuting Wei, and Yuejie Chi.
\newblock Is q-learning minimax optimal? a tight sample complexity analysis.
\newblock \emph{Operations Research}, 2023.

\bibitem[Li and Sun(2023)]{li2023variance}
Xiang Li and Qiang Sun.
\newblock Variance-aware robust reinforcement learning with linear function approximation with heavy-tailed rewards.
\newblock \emph{arXiv preprint arXiv:2303.05606}, 2023.

\bibitem[Lugosi and Mendelson(2019)]{lugosi2019mean}
G{\'a}bor Lugosi and Shahar Mendelson.
\newblock Mean estimation and regression under heavy-tailed distributions: A survey.
\newblock \emph{Foundations of Computational Mathematics}, 19\penalty0 (5):\penalty0 1145--1190, 2019.

\bibitem[Medina and Yang(2016)]{medina2016no}
Andres~Munoz Medina and Scott Yang.
\newblock No-regret algorithms for heavy-tailed linear bandits.
\newblock In \emph{International Conference on Machine Learning}, pages 1642--1650. PMLR, 2016.

\bibitem[Modi et~al.(2020)Modi, Jiang, Tewari, and Singh]{modi2020sample}
Aditya Modi, Nan Jiang, Ambuj Tewari, and Satinder Singh.
\newblock Sample complexity of reinforcement learning using linearly combined model ensembles.
\newblock In \emph{International Conference on Artificial Intelligence and Statistics}, pages 2010--2020. PMLR, 2020.

\bibitem[Puterman(2014)]{puterman2014markov}
Martin~L Puterman.
\newblock \emph{Markov decision processes: discrete stochastic dynamic programming}.
\newblock John Wiley \& Sons, 2014.

\bibitem[Shao et~al.(2018)Shao, Yu, King, and Lyu]{shao2018almost}
Han Shao, Xiaotian Yu, Irwin King, and Michael~R Lyu.
\newblock Almost optimal algorithms for linear stochastic bandits with heavy-tailed payoffs.
\newblock \emph{Advances in Neural Information Processing Systems}, 31, 2018.

\bibitem[Sun et~al.(2020)Sun, Zhou, and Fan]{sun2020adaptive}
Qiang Sun, Wen-Xin Zhou, and Jianqing Fan.
\newblock Adaptive huber regression.
\newblock \emph{Journal of the American Statistical Association}, 115\penalty0 (529):\penalty0 254--265, 2020.

\bibitem[Wagenmaker et~al.(2022)Wagenmaker, Chen, Simchowitz, Du, and Jamieson]{wagenmaker2022first}
Andrew~J Wagenmaker, Yifang Chen, Max Simchowitz, Simon Du, and Kevin Jamieson.
\newblock First-order regret in reinforcement learning with linear function approximation: A robust estimation approach.
\newblock In \emph{International Conference on Machine Learning}, pages 22384--22429. PMLR, 2022.

\bibitem[Xue et~al.(2020)Xue, Wang, Wang, and Zhang]{xue2020nearly}
Bo~Xue, Guanghui Wang, Yimu Wang, and Lijun Zhang.
\newblock Nearly optimal regret for stochastic linear bandits with heavy-tailed payoffs.
\newblock In \emph{Proceedings of the Twenty-Ninth International Joint Conference on Artificial Intelligence, {IJCAI-20}}, pages 2936--2942. International Joint Conferences on Artificial Intelligence Organization, 2020.

\bibitem[Yang and Wang(2019)]{yang2019sample}
Lin Yang and Mengdi Wang.
\newblock Sample-optimal parametric q-learning using linearly additive features.
\newblock In \emph{International Conference on Machine Learning}, pages 6995--7004. PMLR, 2019.

\bibitem[Zanette and Brunskill(2019)]{zanette2019tighter}
Andrea Zanette and Emma Brunskill.
\newblock Tighter problem-dependent regret bounds in reinforcement learning without domain knowledge using value function bounds.
\newblock In \emph{International Conference on Machine Learning}, pages 7304--7312. PMLR, 2019.

\bibitem[Zhang et~al.(2021)Zhang, Yang, Ji, and Du]{zhang2021improved}
Zihan Zhang, Jiaqi Yang, Xiangyang Ji, and Simon~S Du.
\newblock Improved variance-aware confidence sets for linear bandits and linear mixture mdp.
\newblock \emph{Advances in Neural Information Processing Systems}, 34:\penalty0 4342--4355, 2021.

\bibitem[Zhao et~al.(2023)Zhao, He, Zhou, Zhang, and Gu]{zhao2023variance}
Heyang Zhao, Jiafan He, Dongruo Zhou, Tong Zhang, and Quanquan Gu.
\newblock Variance-dependent regret bounds for linear bandits and reinforcement learning: Adaptivity and computational efficiency.
\newblock In \emph{Proceedings of Thirty Sixth Conference on Learning Theory}, volume 195 of \emph{Proceedings of Machine Learning Research}, pages 4977--5020. PMLR, 2023.

\bibitem[Zhong and Zhang(2023)]{zhong2023theoretical}
Han Zhong and Tong Zhang.
\newblock A theoretical analysis of optimistic proximal policy optimization in linear markov decision processes.
\newblock \emph{Advances in Neural Information Processing Systems}, 36, 2023.

\bibitem[Zhong et~al.(2021)Zhong, Huang, Yang, and Wang]{zhong2021breaking}
Han Zhong, Jiayi Huang, Lin Yang, and Liwei Wang.
\newblock Breaking the moments condition barrier: No-regret algorithm for bandits with super heavy-tailed payoffs.
\newblock \emph{Advances in Neural Information Processing Systems}, 34:\penalty0 15710--15720, 2021.

\bibitem[Zhou and Gu(2022)]{zhou2022computationally}
Dongruo Zhou and Quanquan Gu.
\newblock Computationally efficient horizon-free reinforcement learning for linear mixture mdps.
\newblock \emph{Advances in neural information processing systems}, 35:\penalty0 36337--36349, 2022.

\bibitem[Zhou et~al.(2021)Zhou, Gu, and Szepesvari]{zhou2021nearly}
Dongruo Zhou, Quanquan Gu, and Csaba Szepesvari.
\newblock Nearly minimax optimal reinforcement learning for linear mixture markov decision processes.
\newblock In \emph{Conference on Learning Theory}, pages 4532--4576. PMLR, 2021.

\bibitem[Zhou et~al.(2023)Zhou, Zihan, and Du]{zhou2023sharp}
Runlong Zhou, Zhang Zihan, and Simon~Shaolei Du.
\newblock Sharp variance-dependent bounds in reinforcement learning: Best of both worlds in stochastic and deterministic environments.
\newblock In \emph{International Conference on Machine Learning}, pages 42878--42914. PMLR, 2023.

\bibitem[Zhuang and Sui(2021)]{zhuang2021no}
Vincent Zhuang and Yanan Sui.
\newblock No-regret reinforcement learning with heavy-tailed rewards.
\newblock In \emph{International Conference on Artificial Intelligence and Statistics}, pages 3385--3393. PMLR, 2021.

\end{thebibliography}

\clearpage
\appendix
\ifarxiv\else

\fi

\section{Experiments}
\label{sec:exp}

In this section, we conduct empirical evaluations of the proposed algorithm, $\algobandit$, for heavy-tailed linear bandit problems in Definition~\ref{def:bandit} which can be regarded as a special case of linear MDPs. Comparisons are made between MENU and TOFU \citep{shao2018almost}, which give the worst-case optimal regret bound in such settings (See Table~\ref{tab:bandit}). To the best of our knowledge, we are the first to address the challenge of heavy-tailed rewards in RL with function approximation, even when $\epsilon$ is less than $1$. Consequently, no other algorithms in the RL literature can be readily compared to our approach (See Table~\ref{tab:mdp}).

\begin{figure}[ht]
    \centering
    \includegraphics[width=0.6\textwidth]{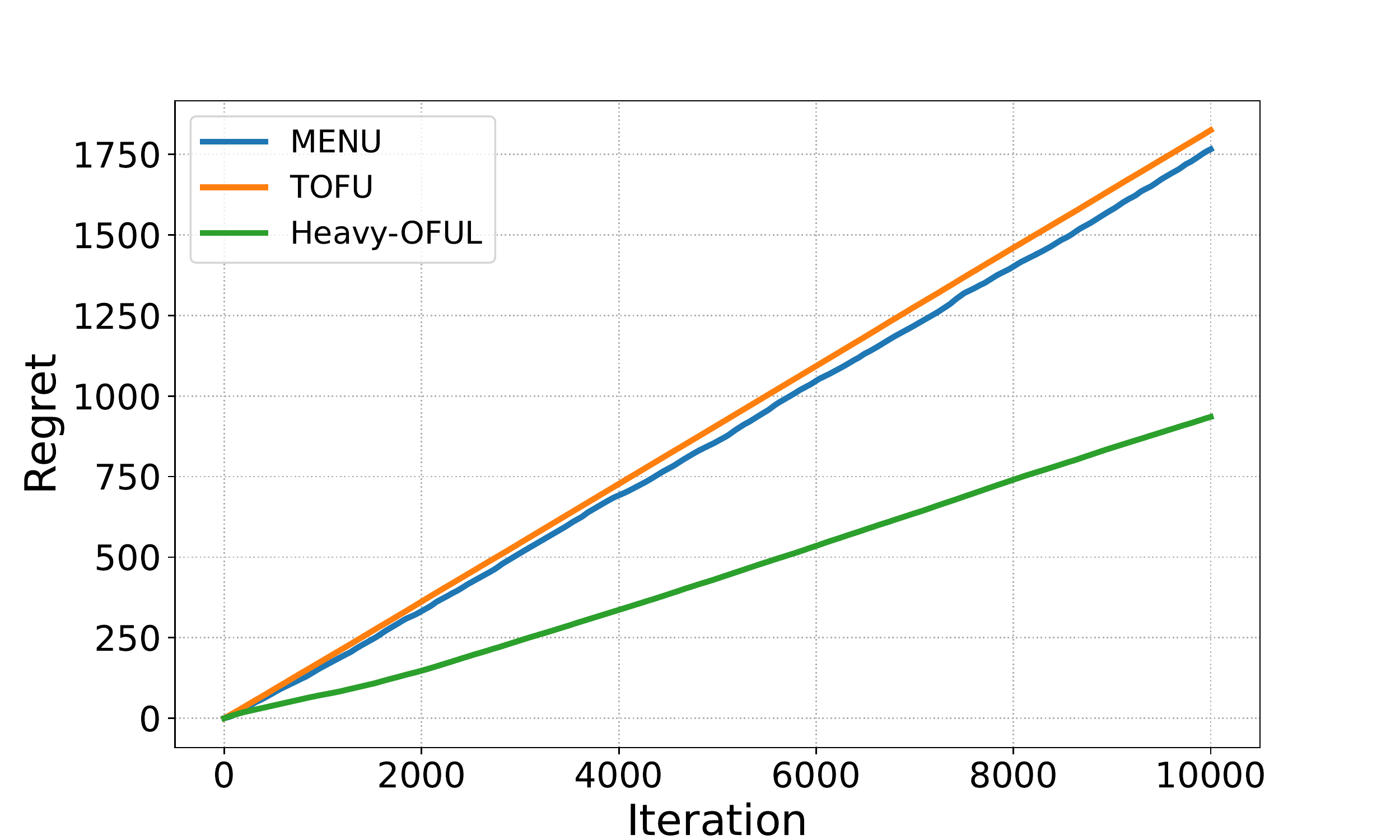}
    \caption{Comparisons of our algorithm ($\algobandit$) versus MENU and TOFU in heavy-tailed linear bandits problems  for $1\times10^4$ rounds.}
    \label{fig:heavy-tailed-comparison}
\end{figure}

We generate $5$ independent paths for each algorithm and show the average cumulative regret. The experimental setup is as follows: Let the feature dimension $d = 10$. For the chosen arm $\bm{\phi}_t \in \mathcal{D}_t$, reward is $R_t = \langle \bm{\phi}_t, \btheta^{*} \rangle + \varepsilon_t$, where $\btheta^* = \mathbf{1}_d / \sqrt{d} \in \mathbb{R}^d$ so that $\|\btheta^*\|_2 = 1$. $\varepsilon_t$ is first sampled from a Student's $t$-distribution with degree of freedom $\text{df}=2$, then is multiplied by a scaling factor $\alpha$ such that the central moments of $\varepsilon_t$ in each rounds are different, where $\log_{10}(\alpha) \sim \mathrm{Unif}(0,2)$. Note the variance of $\varepsilon_t$ does not exist and we choose $\epsilon=0.99$. Normalization is made to ensure $L=B=1$. As shown in Figure~\ref{fig:heavy-tailed-comparison}, results demonstrate the effectiveness of the proposed algorithm, which further corroborates our theoretical findings.

\section{Proofs for Section~\ref{sec:huber}}\label{proof:huber}

\subsection{Properties of Huber Loss}\label{sec:huber-property}

The following properties of Huber loss are important in the proofs.

\begin{property}\label{prop:huber}
Denote $\Psi_{\tau}(x)$ as the derivative of Huber loss, i.e., $\Psi_{\tau}(x) := \ell'_{\tau}(x)$, where $\Psi_{\tau}(x)$ is Huber loss defined in Definition~\ref{def:huber-loss}. Then the followings are true:
\begin{enumerate}[label=(\arabic*)]
    \item $|\Psi_\tau(x)| = \min \{|x|, \tau\}$,\label{prop:huber-1}
    \item $\Psi_{\tau}(x) = \tau \Psi_{1}(\frac{x}{\tau})$,\label{prop:huber-2}
    \item $-\log(1 - x + |x|^{1+\epsilon}) \le \Psi_1(x) \le \log(1 + x + |x|^{1+\epsilon})$\text{ for any }$\epsilon\in(0,1],$\label{prop:huber-3}
    \item $\Psi'_\tau(x) = \ell_\tau''(x) = \ind\{|x|\le \tau\}$.\label{prop:huber-4}
\end{enumerate}
\end{property}

Here, \ref{prop:huber-1} gives an upper bound of the derivative of Huber loss. \ref{prop:huber-2} demonstrates the homogeneous property. \ref{prop:huber-3} shows the similarity between the derivative of Huber loss and influence function in \citet{catoni2012challenging}, thereby motivating us to study the moment generating function for characterizing the noise distributions (See details in the proof of Lemma~\ref{lem:I1} in Appendix~\ref{proof:I1}). And \ref{prop:huber-4} is to characterize the second derivative of Huber loss.

\subsection{Proof of Theorem~\ref{thm:heavy}}\label{proof:heavy}

\paragraph{Parameters for Adaptive Huber Regression.}

First, we set the parameters in Algorithm~\ref{algo:huber} as follows:

\[
c_0 = \frac{1}{\sqrt{23\log \frac{2T^2}{\delta}}},
\quad c_1 = \frac{(\log 3T)^{\frac{1-\epsilon}{1+\epsilon}}}{48\rbr{\log \frac{2T^2}{\delta}}^{\frac{2}{1+\epsilon}}},
\quad \tau_0 = \frac{\sqrt{2\kappa} b (\log 3T)^{\frac{1-\epsilon}{2(1+\epsilon)}}}{\rbr{\log \frac{2T^2}{\delta}}^{\frac{1}{1+\epsilon}}}.
\]

Then, we give the proof of Theorem~\ref{thm:heavy}.

\begin{proof}[Proof of Theorem~\ref{thm:heavy}]

Denote $z_t(\btheta) := \frac{y_t - \dotp{\bphi_t}{\btheta}}{\sigma_t}$. By the definition of $b$ and $\sigma_t$ in Line~\ref{eq:b} and \ref{eq:sigma_t} of Algorithm~\ref{algo:huber}, respectively,
$$
\EE[{|z_t(\btheta^*)|^{1+\epsilon}|\cF_{t-1}}] = \EE\sbr{\abr{\frac{\varepsilon_t}{\sigma_t}}^{1+\epsilon} \bigg| \cF_{t-1}} \le b^{1+\epsilon}.
$$

Notice the gradient of $L_t(\btheta)$ is give by
\begin{equation}\label{eq:gradient}
    \nabla L_t(\btheta) = \lambda \btheta - \sum_{s=1}^t \frac{\bphi_s}{\sigma_s} \Psi_{\tau_s}(z_s(\btheta)).
\end{equation}
Recall $\btheta_t$ is the solution of Line~\ref{eq:theta_t} in Algorithm~\ref{algo:huber}, which is a constrained convex optimization problem. Thus it holds that $\dotp{\nabla L_{t}(\btheta_t)}{\btheta_t- \btheta} \le 0,\,\, \forall \btheta \in \cB_d(B)$. Since $\|\btheta^*\| \le B$, we have
\begin{equation}\label{eq:first-order}
    \dotp{\nabla L_{t}(\btheta_t)}{\btheta_t- \btheta^*} \le 0.
\end{equation}

By the mean value theorem for vector-valued functions, we have
\begin{equation}\label{eq:mean-value-thm}
\nabla L_{t}(\btheta_{t}) - \nabla L_{t}(\btheta^*)
= \int_0^1 \nabla^2  L_{t}((1-\eta) \btheta^* + \eta \btheta_{t}) \mathrm{d} \eta \cdot (\btheta_t - \btheta^*).
\end{equation}

In Lemma~\ref{lem:hessian}, we prove that $\nabla^2 L_{t}(\btheta)$ approximates $\bH_t$ well as long as $\btheta \in \cB_d(B)$.

\begin{lemma}\label{lem:hessian}
Assume $\EE[z_t(\btheta^*)|\cF_{t-1}] = 0$ and $\EE[|z_t(\btheta^*)|^{1+\epsilon} | \cF_{t-1}] \le b^{1+\epsilon}$, set
$$
\tau_0 \ge \frac{\sqrt{2\kappa} b (\log 3T)^{\frac{1-\epsilon}{2(1+\epsilon)}}}{\rbr{\log \frac{2T^2}{\delta}}^{\frac{1}{1+\epsilon}}},
\quad c_0 \le \frac{1}{\sqrt{23\log \frac{2T^2}{\delta}}},
\quad c_1 \le \frac{(\log 3T)^{\frac{1-\epsilon}{1+\epsilon}}}{48\rbr{\log \frac{2T^2}{\delta}}^{\frac{2}{1+\epsilon}}}.
$$
Then with probability at least $1-\delta$, for all $t \le T$ and $\btheta\in \cB_d(B)$, we have
$$
\frac{1}{3}\bH_t \preceq \nabla^2 L_t(\btheta) \preceq \bH_t.
$$
\end{lemma}
\begin{proof}
See Appendix~\ref{proof:hessian} for a detailed proof.
\end{proof}

We set $\tau_0, c_0, c_1$ to follow Lemma~\ref{lem:hessian}. Lemma~\ref{lem:hessian} implies that for all $t \le T$ and $\eta \in [0,1]$, we have
$$
\nabla^2 L_t((1-\eta) \btheta^* + \eta \btheta_{t}) \succeq \frac{1}{3}\bH_t.
$$
Thus, multiplying $\btheta_t - \btheta^*$ in both sides of \eqref{eq:mean-value-thm} yields
\begin{equation}
\begin{aligned}\label{eq:hessian-approx}
    &\dotp{\btheta_t - \btheta^*}{\nabla L_{t}(\btheta_{t}) - \nabla L_{t}(\btheta^*)}\\
    ={}& \dotp{\btheta_t - \btheta^*}{\int_0^1 \nabla^2  L_{t}((1-\eta) \btheta^* + \eta \btheta_{t}) \mathrm{d} \eta \cdot (\btheta_t - \btheta^*)}\\
    \ge{}& \dotp{\btheta_t - \btheta^*}{\frac{1}{3}\bH_t \cdot (\btheta_t - \btheta^*)}\\
    ={}& \frac{1}{3} \|\btheta_t - \btheta^*\|_{\bH_t}^2.
\end{aligned}
\end{equation}
On the other hand, by \eqref{eq:first-order}, we have
\[
\dotp{\btheta_t - \btheta^*}{\nabla L_{t}(\btheta_{t}) - \nabla L_{t}(\btheta^*)}
\le \dotp{\btheta_t - \btheta^*}{- \nabla L_{t}(\btheta^*)} 
\le \|\btheta_t - \btheta^*\|_{\bH_t} \cdot \|\nabla L_{t}(\btheta^*)\|_{\bH_t^{-1}}.
\]
That implies
\begin{equation}\label{eq:surrogate}
    \|\btheta_t - \btheta^*\|_{\bH_t} \le 3 \|\nabla L_t(\btheta^*)\|_{\bH_t^{-1}}.
\end{equation}

Next we use Lemma~\ref{lem:grad} to give a high probability upper bound of $\|\nabla L_t(\btheta^*)\|_{\bH_t^{-1}}$.

\begin{lemma}\label{lem:grad}
Assume $\EE[z_t(\btheta^*)|\cF_{t-1}] = 0$, $\EE[|z_t(\btheta^*)|^{1+\epsilon} | \cF_{t-1}] \le b^{1+\epsilon}$ and 
\[
\tau_0 \ge \frac{\sqrt{2\kappa} b (\log 3T)^{\frac{1-\epsilon}{2(1+\epsilon)}}}{\rbr{\log \frac{2T^2}{\delta}}^{\frac{1}{1+\epsilon}}},
\]
then with probability at least $1-2\delta, \forall t \ge 1$, we have
$$
\|\nabla L_t(\btheta^*)\|_{\bH_t^{-1}} \le \sqrt{\lambda}B + \alpha_t,
$$
where
$$
\alpha_t = 4 t^{\frac{1-\epsilon}{2(1+\epsilon)}} \sbr{\frac{(\sqrt{2\kappa} b)^{1+\epsilon} (\log 3t)^{\frac{1-\epsilon}{2}}}{\tau_0^\epsilon} + \tau_0 \log \frac{2t^2}{\delta}}.
$$
\end{lemma}
\begin{proof}
See Appendix~\ref{proof:grad} for a detailed proof.
\end{proof}

Notice the conditions of Lemma~\ref{lem:grad} are satisfied. On the event where Lemma~\ref{lem:hessian} and \ref{lem:grad} hold, whose probability is at least $1-3\delta$, combining \eqref{eq:surrogate} and Lemma~\ref{lem:grad} 
gives
\[
\|\btheta_t - \btheta^*\|_{\bH_t} \le 3\sqrt{\lambda}B + 12 t^{\frac{1-\epsilon}{2(1+\epsilon)}} \sbr{\frac{(\sqrt{2\kappa} b)^{1+\epsilon} (\log 3t)^{\frac{1-\epsilon}{2}}}{\tau_0^\epsilon} + \tau_0 \log \frac{2t^2}{\delta}}.
\]
Choosing $\tau_0 = \frac{\sqrt{2\kappa} b (\log 3T)^{\frac{1-\epsilon}{2(1+\epsilon)}}}{\rbr{\log \frac{2T^2}{\delta}}^{\frac{1}{1+\epsilon}}}$ gives the final result.
\end{proof}

\subsection{Perturbation Analysis of Adaptive Huber Regression}\label{sec:huber-perturbation}

\begin{lemma}
\label{lem:perturbation}
    Consider the setting in Definition~\ref{def:regression}. Let $L_t(\cdot), \btheta_t$ be defined in Theorem~\ref{thm:heavy}.
    Define $\hat{L}_t(\btheta) := \frac{\lambda}{2}\|\btheta\|^2 + \sum_{s=1}^{t} \ell_{\tau_s}\rbr{\frac{\hat{y}_s - \dotp{\bphi_s}{\btheta}}{\sigma_s}}$, where the parameters $\tau_s,\sigma_s$ are chosen according to Theorem~\ref{thm:heavy}. $\hat{L}_t(\cdot)$ is the same as $L_t(\cdot)$ except for substituting $y_s$ with $\hat{y}_s \in \cF_s$.
    And let $\hat{\btheta}_t := \argmin_{\btheta \in \cB_d(B)} \hat{L}_t(\btheta)$.
    Suppose $\abr{\hat{y}_s - y_s} \le \hat{\beta}_s \nbr{\bphi_s}_{\bH_{s}^{-1}}$ for $s\le t$. Then, under the event where Theorem~\ref{thm:heavy} holds, we have
    \[
    \|\hat{\btheta}_t - \btheta_t\|_{\bH_t} \le 6 \kappa \max_{s\le t}\hat{\beta}_s.
    \]
\end{lemma}
\begin{proof}
Notice $\btheta_t$ and $\hbtheta_t$ are solutions of convex optimization problems, we have
\begin{equation}\label{eq:perturbation-convex}
    \dotp{\nabla L_t(\btheta_t)}{\btheta_t - \hbtheta_t} \le 0,\quad
    \dotp{\nabla \hat{L}_t(\hbtheta_t)}{\hbtheta_t - \btheta_t} \le 0.
\end{equation}
With a similar argument as for \eqref{eq:hessian-approx} in the proof of Theorem~\ref{thm:heavy} in Appendix~\ref{proof:heavy}, on the event where Theorem~\ref{thm:heavy} holds, we have
$$
\dotp{\btheta_t - \hbtheta_t}{\nabla L_t(\btheta) - \nabla L_t(\hbtheta_t)} \ge \frac{1}{3} \|\btheta_t - \hbtheta_t\|_{\bH_t}^2.
$$
Combined with \eqref{eq:perturbation-convex}, it holds that
$$
\|\btheta_t - \hbtheta_t\|_{\bH_t}^2 \le 3 \dotp{\btheta_t - \hbtheta_t}{\nabla \hat{L}_t(\hbtheta_t) - \nabla L_t(\hbtheta_t)} \le 3 \|\btheta_t - \hbtheta_t\|_{\bH_t} \|\nabla \hat{L}_t(\hbtheta_t) - \nabla L_t(\hbtheta_t)\|_{\bH_t^{-1}}.
$$
Thus
\begin{align*}
    \|\btheta_t - \hbtheta_t\|_{\bH_t} &\le 3 \|\nabla \hat{L}_t(\hbtheta_t) - \nabla L_t(\hbtheta_t)\|_{\bH_t^{-1}}\\
    &= 3 \nbr{\sum_{s=1}^t \frac{\bphi_s}{\sigma_s} \sbr{\Psi_{\tau_s}\rbr{\frac{\hat{y}_s - \dotp{\phi_s}{\hbtheta_t}}{\sigma_s}} - \Psi_{\tau_s}\rbr{\frac{y_s - \dotp{\phi_s}{\hbtheta_t}}{\sigma_s}}}}_{\bH_t^{-1}}\\
    &\overset{(a)}{\le} 3 \nbr{\sum_{s=1}^t \frac{\bphi_s}{\sigma_s} \abr{\frac{\hat{y}_s - y_s}{\sigma_s}}}_{\bH_t^{-1}}\\
    &\overset{(b)}{\le} 3 \cdot (\max_{s\le t}\hat{\beta}_s) \cdot \nbr{\sum_{s=1}^t \frac{\bphi_s}{\sigma_s} \cdot \nbr{\frac{\bphi_s}{\sigma_s}}_{\bH_{s}^{-1}}}_{\bH_t^{-1}}\\
    &\le 3 \cdot (\max_{s\le t}\hat{\beta}_s) \cdot \sum_{s=1}^t \nbr{\frac{\bphi_s}{\sigma_s}}_{\bH_t^{-1}} \cdot \nbr{\frac{\bphi_s}{\sigma_s}}_{\bH_{s}^{-1}}
\end{align*}
where $(a)$ is due to \ref{prop:huber-1} of Property~\ref{prop:huber} and $(b)$ uses the condition that $|\hat{y}_s - y_s| \le \hat{\beta}_s \nbr{\bphi_s}_{\bH_{s}^{-1}}$.
Notice $\bH_t^{-1} \preceq \bH_s^{-1}$ for $s\le t$, we have $\nbr{\frac{\bphi_s}{\sigma_s}}_{\bH_t^{-1}} \le \nbr{\frac{\bphi_s}{\sigma_s}}_{\bH_s^{-1}}$, which further implies
\begin{equation}\label{eq:sherman-morrison}
    \nbr{\frac{\bphi_s}{\sigma_s}}_{\bH_s^{-1}}^2
    = \frac{\bphi_s^\top}{\sigma_s} \rbr{\bH_{s-1}^{-1}  - \frac{\bH_{s-1}^{-1} \frac{\bphi_s}{\sigma_s} \frac{\bphi_s^\top}{\sigma_s}\bH_{s-1}^{-1}}{1+\frac{\bphi_s^\top}{\sigma_s} \bH_{s-1}^{-1}\frac{\bphi_s^\top}{\sigma_s}}}\frac{\bphi_s}{\sigma_s}
    = w_s^2 - \frac{w_s^4}{1+w_s^2}
    = \frac{w_s^2}{1+w_s^2}
\end{equation}
due to Sherman-Morrison formula with $w_t = \|\bphi_t / \sigma_t\|_{\bH_{t-1}^{-1}}$.
Then It follows that
$$
\sum_{s=1}^t \nbr{\frac{\bphi_s}{\sigma_s}}_{\bH_t^{-1}} \cdot \nbr{\frac{\bphi_s}{\sigma_s}}_{\bH_{s}^{-1}}
\le \sum_{s=1}^t \frac{w_s^2}{1+w_s^2}
\le \sum_{s=1}^t \min\cbr{1, w_s^2} \le 2 \kappa,
$$
where the last inequality holds due to Lemma~\ref{lem:w-sum}.
Finally, we have 
$$
\|\btheta_t - \hbtheta_t\|_{\bH_t} \le 6 \kappa \max_{s\le t}\hat{\beta}_s.
$$
\end{proof}

\subsection{Proof of Lemma~\ref{lem:hessian}}\label{proof:hessian}

In the following proof, we denote $z_t^* := z_t(\btheta^*) = \frac{y_t - \dotp{\bphi_t}{\btheta^*}}{\sigma_t} = \frac{\varepsilon_t}{\sigma_t}$ for short.

\begin{proof}[Proof of Lemma~\ref{lem:hessian}]
Using \ref{prop:huber-4} of Property~\ref{prop:huber}, the Hessian matrix of $L_t(\btheta)$ is given by
\begin{equation}\label{eq:hessian}
\nabla^2 L_t(\btheta)
= \lambda \bI + \sum_{s=1}^t \frac{1}{\sigma_s^2} \bphi_s \bphi_s^\top \ind\{|z_s(\btheta)| \le \tau_s\}
= \bH_t - \sum_{s=1}^t \frac{1}{\sigma_s^2} \bphi_s \bphi_s^\top \ind\{|z_s(\btheta)| > \tau_s\}.
\end{equation}
Thus
$$
\nabla^2 L_t(\btheta) \preceq \bH_t
$$
holds trivially. And we only need to prove
$$
\nabla^2 L_t(\btheta) \succeq \frac{1}{3} \bH_t.
$$

We can decompose $z_s(\btheta)$ as
$$
z_s(\btheta) = \frac{y_s - \dotp{\bphi_s}{\btheta}}{\sigma_s}
= \frac{\varepsilon_s + \dotp{\bphi_s}{\btheta^* - \btheta}}{\sigma_s}
= z_s^* + \frac{\dotp{\bphi_s}{\btheta^* - \btheta}}{\sigma_s}.
$$
Then notice
$$
\ind\{|z_s(\btheta)| > \tau_s\} \le \ind \cbr{|z_s^*| > \frac{\tau_s}{2}} + \ind \cbr{\abr{\frac{\dotp{\bphi_s}{\btheta - \btheta^*}}{\sigma_s}} > \frac{\tau_s}{2}}.
$$
For all $\bu \in \cB_d(1)$, we have
\begin{align*}
    & \bu^\top \nabla^2 L_t(\btheta) \bu = \bu^\top \bH_t \bu - \sum_{s=1}^t \frac{\dotp{\bphi_s}{\bu}^2}{\sigma_s^2} \ind\{|z_s(\btheta)| > \tau_s\}\\
    \ge{}& \bu^\top \bH_t \bu - \underbrace{\sum_{s=1}^t \frac{\dotp{\bphi_s}{\bu}^2}{\sigma_s^2} \ind \cbr{|z_s^*| > \frac{\tau_s}{2}}}_{\text{(i)}} - \underbrace{\sum_{s=1}^t \frac{\dotp{\bphi_s}{\bu}^2}{\sigma_s^2} \ind \cbr{\abr{\frac{\dotp{\bphi_s}{\btheta - \btheta^*}}{\sigma_s}} > \frac{\tau_s}{2}}}_{\text{(ii)}}.
\end{align*}

Next we establish upper bounds for term (i) and (ii) respectively.

\paragraph{Term (i).} 

Notice
\begin{equation}\label{eq:lem1-bound-i-transform}
\begin{aligned}
    \sum_{s=1}^t \frac{\dotp{\bphi_s}{\bu}^2}{\sigma_s^2} \ind \cbr{|z_s^*| > \frac{\tau_s}{2}}
    &\overset{(a)}{\le} \sum_{s=1}^t \nbr{\frac{\bphi_s}{\sigma_s}}_{\bH_t^{-1}}^2 \|\bu\|_{\bH_t}^2 \ind \cbr{|z_s^*| > \frac{\tau_s}{2}}\\
    &\overset{(b)}{=} \|\bu\|_{\bH_t}^2 \sum_{s=1}^t \frac{w_s^2}{1+w_s^2} \ind \cbr{|z_s^*| > \frac{\tau_s}{2}}\\
    &\overset{(c)}{\le} c_0^2 \|\bu\|_{\bH_t}^2 \sum_{s=1}^t \ind \cbr{|z_s^*| > \frac{\tau_s}{2}},
\end{aligned}
\end{equation}
where $(a)$ holds due to the Cauchy-Schwartz inequality, $(b)$ holds with the same argument as \eqref{eq:sherman-morrison}, while $(c)$ holds due to $\frac{w_s^2}{1+w_s^2} \le w_s^2$ and our choice of $\sigma_s$ in Line~\ref{eq:sigma_t} of Algorithm~\ref{algo:huber} such that $\sigma_s \ge \frac{\nbr{\bphi_s}_{\bH_{s-1}^{-1}}}{c_0}$.

Next we give an upper bound of $\sum_{s=1}^t \ind \{|z_s^*| > \frac{\tau_s}{2}\}$. Define $X_s := \ind \{|z_s^*| > \frac{\tau_s}{2}\}$ and $Y_s := X_s - \EE[X_s|\cF_{s-1}]$, which are $\cF_s$-measurable. Lemma~\ref{lem:freedman} implies that with probability at least $1 - 2t^2 / \delta, \text{for any fixed } t\ge 1$, we have
\begin{equation}\label{eq:lem1-bound-i-freedman}
    \sum_{s=1}^{t} Y_s \le \sqrt{2 V \log\frac{2t^2}{\delta}} + \frac{2M}{3}\log \frac{2t^2}{\delta} ,
\end{equation}
where $\sum_{s=1}^t \EE[Y_s^2|\cF_{s-1}] \le V \text{ and } |Y_t| \le M$.
It follows that $|Y_s| \le 1$ and
\begin{align*}
    \sum_{s=1}^t \EE[X_s|\cF_{s-1}] &= \sum_{s=1}^t \PP\rbr{|z_s^*| > \frac{\tau_s}{2}}\\
    &\overset{(a)}{\le} 2^{1+\epsilon} \sum_{s=1}^t \frac{b^{1+\epsilon}}{\tau_s^{1+\epsilon}}\\
    &\overset{(b)}{=} \rbr{\frac{2b}{\tau_0}}^{1+\epsilon} \sum_{s=1}^t \rbr{\frac{w_s}{\sqrt{1+w_s^2}}}^{1+\epsilon} s^{-\frac{1-\epsilon}{2}},
\end{align*}
where $(a)$ holds due to Markov's inequality and $(b)$ holds due to the definition of $\tau_s$ in Line~\ref{eq:tau_t} of Algorithm~\ref{algo:huber}.
Notice
\begin{equation}\label{eq:holder-sum}
    \sum_{s=1}^t \rbr{\frac{w_s}{\sqrt{1+w_s^2}}}^{1+\epsilon} s^{-\frac{1-\epsilon}{2}}
    \le \rbr{\sum_{s=1}^t \frac{w_s^2}{1+w_s^2}}^{\frac{1+\epsilon}{2}} \rbr{\sum_{s=1}^t s^{-1}}^{\frac{1-\epsilon}{2}}
    \le (2 \kappa)^{\frac{1+\epsilon}{2}} (\log 3t)^{\frac{1-\epsilon}{2}},
\end{equation}
where the first inequality holds due to Hölder's inequality and the second holds due to $\frac{w_s^2}{1+w_s^2} \le \min\{1, w_s^2\}$, Lemma~\ref{lem:w-sum} and $\sum_{s=1}^t s^{-1} \le 1 + \log t \le \log 3t$.
Thus
\[
    \sum_{s=1}^t \EE[X_s|\cF_{s-1}] \le \rbr{\frac{2 b}{\tau_0}}^{1+\epsilon} (2 \kappa)^{\frac{1+\epsilon}{2}} (\log 3t)^{\frac{1-\epsilon}{2}}= \rbr{\frac{2 \sqrt{2\kappa} b (\log 3t)^{\frac{1-\epsilon}{2(1+\epsilon)}}}{\tau_0}}^{1+\epsilon}.
\]
Then it follows that
$$
\sum_{s=1}^t \EE[Y_s^2|\cF_{s-1}] \le \sum_{s=1}^t \EE[X_s^2 | \cF_{s-1}] = \sum_{s=1}^t \EE[X_s|\cF_{s-1}].
$$
Setting $V = \rbr{\frac{2 \sqrt{2\kappa} b (\log 3t)^{\frac{1-\epsilon}{2(1+\epsilon)}}}{\tau_0}}^{1+\epsilon}$ and $M=1$ in \eqref{eq:lem1-bound-i-freedman}, using a union bound over $t \ge 1$ and the fact that $\sum_{t=1}^{\infty} \frac{1}{2t^2} \le 1$, with probability at least $1-\delta, \forall t\ge 1$, we have
\begin{equation}\label{eq:lem1-bound-i-sum}
\begin{aligned}
    &\sum_{s=1}^{t}\ind \cbr{|z_s^*| > \frac{\tau_s}{2}} = \sum_{s=1}^{t} X_s\\
    \le{}& \sum_{s=1}^t \EE[X_s|\cF_{s-1}] + \sqrt{2 V \log\frac{2t^2}{\delta}} + \frac{2M}{3}\log \frac{2t^2}{\delta}\\
    ={}& \rbr{\frac{2 \sqrt{2\kappa} b (\log 3t)^{\frac{1-\epsilon}{2(1+\epsilon)}}}{\tau_0}}^{1+\epsilon} + \sqrt{2 \rbr{\frac{2 \sqrt{2\kappa} b (\log 3t)^{\frac{1-\epsilon}{2(1+\epsilon)}}}{\tau_0}}^{1+\epsilon} \log\frac{2t^2}{\delta}} + \frac{2}{3}\log \frac{2t^2}{\delta}.
\end{aligned}
\end{equation}
Choosing $\tau_0 \ge \frac{\sqrt{2\kappa} b (\log 3T)^{\frac{1-\epsilon}{2(1+\epsilon)}}}{\rbr{\log \frac{2T^2}{\delta}}^{\frac{1}{1+\epsilon}}}$ yields $\rbr{\frac{2 \sqrt{2\kappa} b (\log 3t)^{\frac{1-\epsilon}{2(1+\epsilon)}}}{\tau_0}}^{1+\epsilon} \le 2^{1+\epsilon} \log \frac{2T^2}{\delta}$.

Then it follows from \eqref{eq:lem1-bound-i-transform} and \eqref{eq:lem1-bound-i-sum} that
\begin{align*}
    \sum_{s=1}^t \frac{\dotp{\bphi_s}{\bu}^2}{\sigma_s^2} \ind \cbr{|z_s^*| > \frac{\tau_s}{2}} &\le c_0^2 \|\bu\|_{\bH_t}^2 \sum_{s=1}^t \ind \cbr{|z_s^*| > \frac{\tau_s}{2}}\\
    &\le c_0^2 \|\bu\|_{\bH_t}^2 \rbr{2^{1+\epsilon} + \sqrt{2 \cdot 2^{1+\epsilon}} + \frac{2}{3}} \log \frac{2T^2}{\delta} \\
    &\le \frac{23}{3} c_0^2 \log \frac{2T^2}{\delta} \|\bu\|_{\bH_t}^2\\
    &\le \frac{1}{3} \|\bu\|_{\bH_t}^2,
\end{align*}
where the third inequality holds due to $\epsilon \le 1$ and the last holds due to the definition of $c_0$.

\paragraph{Term (ii).}
\begin{equation}
\begin{aligned}\label{eq:lem1-bound-ii-sum}
    \sum_{s=1}^t \frac{\dotp{\bphi_s}{\bu}^2}{\sigma_s^2} \ind\cbr{\abr{\frac{\dotp{\bphi_s}{\btheta - \btheta^*}}{\sigma_s}} > \frac{\tau_s}{2}} &\le \sum_{s=1}^t \frac{\dotp{\bphi_s}{\bu}^2}{\sigma_s^2} \frac{\dotp{\bphi_s}{\btheta - \btheta^*}^2}{\sigma_s^2} \frac{4}{\tau_s^2}\\
    &\overset{(a)}{\le} 16 \sum_{s=1}^t \frac{\dotp{\bphi_s}{\bu}^2}{\sigma_s^2}  \frac{L^2 B^2}{\sigma_s^2 \tau_0^2} \frac{w_s^2}{1+w_s^2} s^{-\frac{1-\epsilon}{1+\epsilon}}\\
    &\overset{(b)}{\le} 16 \sum_{s=1}^t \frac{\dotp{\bphi_s}{\bu}^2}{\sigma_s^2} \frac{L^2 B^2 w_s^2}{\sigma_s^2 \cdot 2\kappa b^2} \frac{2\kappa b^2}{\tau_0^2}\\
    &\overset{(c)}{\le} 16 \sum_{s=1}^t \frac{\dotp{\bphi_s}{\bu}^2}{\sigma_s^2} c_1 \frac{2\kappa b^2}{\tau_0^2}\\
    &\overset{(d)}{\le} \frac{1}{3} \sum_{s=1}^t \frac{\dotp{\bphi_s}{\bu}^2}{\sigma_s^2},
\end{aligned}
\end{equation}
where $(a)$ holds due to Cauchy-Schwartz inequality with $\|\bphi_s\| \le L$, $\|\btheta - \btheta^*\| \le 2B$ and the definition of $\tau_s$ in Line~\ref{eq:tau_t} of Algorithm~\ref{algo:huber}. $(b)$ holds due to $\frac{1}{1+w_s^2} s^{-\frac{1-\epsilon}{1+\epsilon}} \le 1$. $(c)$ uses the definition of $\sigma_s$ in Line~\ref{eq:sigma_t} in Algorithm~\ref{algo:huber}, that is $\frac{L^2 B^2 w_s^2}{\sigma_s^2 \cdot 2\kappa b^2} \le c_1$ implied by $\sigma_s \ge \frac{\sqrt{LB} \|\bphi_s\|_{\bH_{s-1}^{-1}}^{\frac{1}{2}}}{c_1^{\frac{1}{4}} (2\kappa b^2)^{\frac{1}{4}}}$. $(d)$ holds due to our choice of $\tau_0$ and $c_1$.

\paragraph{Putting pieces together.}

Combining \eqref{eq:lem1-bound-i-sum} and \eqref{eq:lem1-bound-ii-sum}, with probability at least $1-\delta$, for all $t \le T, \btheta\in \cB_d(B)$ and $\bu \in \cB_d(1)$, we have
\begin{align*}
    &\bu^\top \nabla^2 L_t(\btheta) \bu\\
    \ge{}& \bu^\top \bH_t \bu - \sum_{s=1}^t \frac{\dotp{\bphi_s}{\bu}^2}{\sigma_s^2} \ind \cbr{|z_s^*| > \frac{\tau_s}{2}} - \sum_{s=1}^t \frac{\dotp{\bphi_s}{\bu}^2}{\sigma_s^2} \ind\cbr{\abr{\frac{\dotp{\bphi_s}{\btheta - \btheta^*}}{\sigma_s}} > \frac{\tau_s}{2}}.\\
    \ge{}& \bu^\top \bH_t \bu - \frac{1}{3} \bu^\top \bH_t \bu - \frac{1}{3} \sum_{s=1}^t \frac{\dotp{\bphi_s}{\bu}^2}{\sigma_s^2} \\
    \ge{}& \frac{2}{3} \bu^\top \bH_t \bu - \frac{1}{3} \rbr{\lambda \|\bu\|^2 + \sum_{s=1}^t \frac{\dotp{\bphi_s}{\bu}^2}{\sigma_s^2}}\\
    ={}& \frac{1}{3} \bu^\top \bH_t \bu,
\end{align*}
which implies $\nabla^2 L_t(\btheta) \succeq \frac{1}{3} \bH_t$.
\end{proof}

\subsection{Proof of Lemma~\ref{lem:grad}}\label{proof:grad}

\begin{proof}[Proof of Lemma~\ref{lem:grad}]

The expression of $\nabla L_t(\btheta)$ in \eqref{eq:gradient} gives
\begin{align*}
    \|\nabla L_t(\btheta^*)\|_{\bH_t^{-1}} &\le \lambda\|\btheta\|_{\bH_t^{-1}} + \Bigg \| \underbrace{\sum_{s=1}^{t} \Psi_{\tau_s}(z_s^*) \frac{\bphi_s}{\sigma_s}}_{\bd_t} \Bigg \|_{\bH_t^{-1}}.
\end{align*}

We next construct upper bounds of $\lambda\|\btheta\|_{\bH_t^{-1}}$ and $\|\bd_t\|_{\bH_t^{-1}}$, respectively.

\paragraph{Bound $\lambda\|\btheta\|_{\bH_t^{-1}}$.} Since $\bH_t \succeq \lambda\bI$, we have $\bH_t^{-1} \preceq \frac{1}{\lambda}\bI$. Thus, $\lambda\|\btheta\|_{\bH_t^{-1}} \le \lambda \frac{1}{\sqrt{\lambda}}\|\btheta\| \le \sqrt{\lambda} B$.

\paragraph{Bound $\|\bd_t\|_{\bH_t^{-1}}$.}

We aim to decompose $\|\bd_t\|_{\bH_t^{-1}}^2$ into two terms and bound them separately.
The fact that $\bH_t = \bH_{t-1} + \frac{\bphi_t \bphi_t^\top}{\sigma_t^2}$ together with the Sherman-Morrison formula implies that
\begin{equation}
	\label{eq:sherman-morrison-formula}
	\bH_t^{-1} =\bH_{t-1}^{-1} - \frac{\bH_{t-1}^{-1}\bphi_t \bphi_t^\top\bH_{t-1}^{-1} }{\sigma_t^2 (1+w_t^2)}, 
	\quad \text{where }
	w_t^2 :=  \frac{ \bphi_t^\top \bH_{t-1}^{-1} \bphi_t}{\sigma_t^2}
	= \nbr{\frac{\bphi_t}{\sigma_t}}_{\bH_{t-1}}^2.
\end{equation}
Clearly, $w_t$ is $\cF_{t-1}$-measurable and thus is predictable.
By definition of $\bd_t$ and \eqref{eq:sherman-morrison-formula}, 
\begin{align}
	\label{eq:help7}
	&\|\bd_t\|_{\bH_t^{-1}}^2 \\
	={}& \rbr{\bd_{t-1} + \Psi_{\tau_t}(z_t^*) \frac{\bphi_t}{\sigma_t}}^{\top}\bH_{t}^{-1}\rbr{\bd_{t-1} + \Psi_{\tau_t}(z_t^*) \frac{\bphi_t}{\sigma_t}} \nonumber \\
	={}&\|  \bd_{t-1}\|_{\bH_{t-1}^{-1}}^2 -  \frac{1}{1+w_t^2} \rbr{\frac{\bd_{t-1}^\top \bH_{t-1}^{-1} \bphi_t}{\sigma_t}}^2 + 2 \Psi_{\tau_t}(z_t^*) \frac{\bd_{t-1}^\top \bH_{t}^{-1}\bphi_t}{\sigma_t} 
	+ \Psi_{\tau_t}^2(z_t^*) \frac{ \bphi_t^\top \bH_{t}^{-1}\bphi_t}{\sigma_t^2} \nonumber \\
	\le{}& \|  \bd_{t-1}\|_{\bH_{t-1}^{-1}}^2 +  
	\underbrace{2 \Psi_{\tau_t}(z_t^*) \frac{\bd_{t-1}^\top \bH_t^{-1}\bphi_t}{\sigma_t}  }_{I_1}
	+ \underbrace{\Psi_{\tau_t}^2(z_t^*) \frac{ \bphi_t^\top \bH_t^{-1}\bphi_t}{\sigma_t^2}}_{I_2}.
\end{align}
For $I_1$, by~\eqref{eq:sherman-morrison-formula}, we have
\begin{align*}
	I_1 &=  2 \Psi_{\tau_t}(z_t^*) \frac{1}{\sigma_t}\bd_{t-1}^\top \rbr{\bH_{t-1}^{-1} -  \frac{\bH_{t-1}^{-1}\bphi_t \bphi_t^\top\bH_{t-1}^{-1} }{\sigma_t^2 (1+w_t^2)}} \bphi_t\\
	&= 2 \Psi_{\tau_t}(z_t^*)
	\frac{1}{1+w_t^2} 
	\frac{\bd_{t-1}^\top  \bH_{t-1}^{-1} \bphi_t}{\sigma_t}.
\end{align*}
For $I_2$, we have
\begin{align*}
	I_2  &= \Psi_{\tau_t}^2(z_t^*) 
	\frac{ \bphi_t^\top \bH_t^{-1}\bphi_t}{\sigma_t^2}\\
	&= 
	\Psi_{\tau_t}^2(z_t^*)   \frac{1}{\sigma_t^2}
	\bphi_t^\top \rbr{\bH_{t-1}^{-1} -  \frac{\bH_{t-1}^{-1}\bphi_t \bphi_t^\top\bH_{t-1}^{-1} }{\sigma_t^2 (1+w_t^2)}}  \bphi_t\\
	&= \Psi_{\tau_t}^2(z_t^*) \rbr{w_t^2 - \frac{w_t^4}{1+w_t^2}} \\
	&= \Psi_{\tau_t}^2(z_t^*) \frac{w_t^2}{1+w_t^2}.
\end{align*}
Using the equations for $I_1, I_2$ and iterating~\eqref{eq:help7}, we have
\[
\|  \bd_t\|_{\bH_t^{-1}}^2 \le \underbrace{\sum_{s=1}^{t} \Psi_{\tau_s} (z_s^*) \frac{2}{1+w_s^2}\frac{\bd_{s-1}^\top\bH_{s-1}^{-1}\bphi_s}{\sigma_s}}_{\text{(i)}} + \underbrace{\sum_{s=1}^{t} \Psi_{\tau_s}^2 (z_s^*) \frac{w_s^2}{1+w_s^2}}_{\text{(ii)}}.
\]

Next we use Lemma~\ref{lem:I1} and Lemma~\ref{lem:I2} to bound (i) and (ii), respectively. We remark that (i) is the leading term. In the proof of Lemma~\ref{lem:I1}, inspired by \citep{catoni2012challenging,sun2020adaptive}, we make use of \ref{prop:huber-3} of Property~\ref{prop:huber} to carefully quantify the moment generating function, and thus achieves a tight bound.

\begin{lemma}\label{lem:I1}
Assume $\EE[z_t(\btheta^*)|\cF_{t-1}] = 0$ and $\EE[|z_t(\btheta^*)|^{1+\epsilon} | \cF_{t-1}] \le b^{1+\epsilon}$, let $A_t$ define the event where $\|\bd_s\|_{\bH_s^{-1}} \le \alpha_s$ for $s \le t$. Then with probability at least $1-\delta, \forall t \ge 1$, we have
$$
\sum_{s=1}^{t} \Psi_{\tau_s}(z_s^*) \frac{2\bd_{s-1}^\top \bH_{s-1}^{-1} \bphi_s \ind_{A_{s-1}}}{(1+w_s^2)\sigma_s} \le (\max_{s\le t} \alpha_s) 2 t^{\frac{1-\epsilon}{2(1+\epsilon)}} \sbr{\frac{(\sqrt{2\kappa} b)^{1+\epsilon} (\log 3t)^{\frac{1-\epsilon}{2}}}{\tau_0^\epsilon} + \tau_0 \log \frac{2t^2}{\delta}}.
$$
\end{lemma}
\begin{proof}
See Appendix~\ref{proof:I1} for a detailed proof.
\end{proof}

\begin{lemma}\label{lem:I2}
Assume $\EE[z_t(\btheta^*)|\cF_{t-1}] = 0$ and $\EE[|z_t(\btheta^*)|^{1+\epsilon} | \cF_{t-1}] \le b^{1+\epsilon}$. Then with probability at least $1-\delta, \forall t \ge 1$, we have
$$
\sum_{s=1}^{t} \Psi_{\tau_s}^2 (z_s^*) \frac{w_s^2}{1+w_s^2} \le \sbr{t^{\frac{1-\epsilon}{2(1+\epsilon)}} \rbr{\sqrt{\tau_0^{1-\epsilon} (\sqrt{2\kappa} b)^{1+\epsilon} (\log 3t)^{\frac{1-\epsilon}{2}}} + \tau_0 \sqrt{2 \log \frac{2t^2}{\delta}}}}^2.
$$
\end{lemma}
\begin{proof}
See Appendix~\ref{proof:I2} for a detailed proof.
\end{proof}

Recall $\tau_0 \ge \frac{\sqrt{2\kappa} b (\log 3T)^{\frac{1-\epsilon}{2(1+\epsilon)}}}{\rbr{\log \frac{2T^2}{\delta}}^{\frac{1}{1+\epsilon}}}$. Thus $\sqrt{\tau_0^{1-\epsilon} (\sqrt{2\kappa} b)^{1+\epsilon} (\log 3t)^{\frac{1-\epsilon}{2}}} \le \tau_0 \sqrt{\log \frac{2t^2}{\delta}}$.
We choose
$$
\alpha_t = 4 t^{\frac{1-\epsilon}{2(1+\epsilon)}} \sbr{\frac{(\sqrt{2\kappa} b)^{1+\epsilon} (\log 3t)^{\frac{1-\epsilon}{2}}}{\tau_0^\epsilon} + \tau_0 \log \frac{2t^2}{\delta}}.
$$
Then on the event where Lemma~\ref{lem:I1} and \ref{lem:I2} hold, whose probability is at least $1-2\delta, \forall t \ge 1$, we have
\[
\sum_{s=1}^{t} \Psi_{\tau_s} (z_s^*) \frac{2}{1+w_s^2}\frac{\bd_{s-1}^\top \bH_{s-1}^{-1} \bphi_s \ind_{A_{s-1}}}{\sigma_s} \le \frac{\alpha_t^2}{2},\quad
\sum_{s=1}^{t} \Psi_{\tau_s}^2 (z_s^*) \frac{w_s^2}{1+w_s^2} \le \frac{\alpha_t^2}{2}.
\]

Finally, we can conclude that all $\{A_t\}_{t \ge 0}$ is true and thus $	\|  \bd_t\|_{\bH_t^{-1}}  \le \alpha_t$.
\end{proof}

\subsection{Proof of Lemma~\ref{lem:I1}}\label{proof:I1}

\begin{proof}[Proof of Lemma~\ref{lem:I1}]

It follows that
\begin{align*}
    & \Psi_{\tau_s} (z_s^*) \frac{2}{1+w_s^2} \frac{\bd_{s-1}^\top \bH_{s-1}^{-1} \bphi_s \ind_{A_{s-1}}}{\sigma_s}\\
    = {}& \tau_0 \Psi_1\rbr{\frac{z_s^*}{\tau_s}} \frac{2}{w_s \sqrt{1+w_s^2}} s^{\frac{1-\epsilon}{2(1+\epsilon)}} \frac{\bd_{s-1}^\top \bH_{s-1}^{-1} \bphi_s \ind_{A_{s-1}}}{\sigma_s}\\
    = {}& (\max_{s\le t} \alpha_s) 2 \tau_0 t^{\frac{1-\epsilon}{2(1+\epsilon)}} \Psi_1\rbr{\frac{z_s^*}{\tau_s}} \underbrace{\frac{1}{\sqrt{1+w_s^2}} \rbr{\frac{s}{t}}^{\frac{1-\epsilon}{2(1+\epsilon)}} \frac{\bd_{s-1}^\top \bH_{s-1}^{-1} \bphi_s \ind_{A_{s-1}}}{(\max_{s\le t} \alpha_s) w_s \sigma_s}}_{\text{denoted as }M_s},
\end{align*}
where the first equality is by \ref{prop:huber-2} of Property~\ref{prop:huber} and the definition of $\tau_s$.
Note that $M_s$ is $\cF_{s-1}$ measurable and
\begin{align*}
    |M_s| &\le \frac{|\bd_{s-1}^\top \bH_{s-1}^{-1} \bphi_s| \ind_{A_{s-1}}}{(\max_{s\le t} \alpha_s) w_s \sigma_s}\\
    &\le \frac{\|\bd_{s-1}\|_{\bH_{s-1}^{-1}} \|\bphi_s\|_{\bH_{s-1}^{-1}} \ind_{A_{s-1}}}{(\max_{s\le t} \alpha_s) w_s \sigma_s}\\
    &\le 1,
\end{align*}
where the first inequality holds due to $\frac{1}{\sqrt{1+w_s^2}} \rbr{\frac{s}{t}}^{\frac{1-\epsilon}{2(1+\epsilon)}} \le 1$, the second holds due to Cauchy-Schwartz inequality and the last holds due to the definition of $\ind_{A_{s-1}}$.

Next we make use of \ref{prop:huber-3} of Property~\ref{prop:huber} to carefully quantify the moment generating function of $\sum_{s=1}^t M_s \Psi_1\rbr{\frac{z_s^*}{\tau_s}}$ and leverage a Chernoff bounding technique to complete the proof.

It follows from \ref{prop:huber-3} of Property~\ref{prop:huber} that
\begin{align*}
    M_s \Psi_1\rbr{\frac{z_s^*}{\tau_s}} \le {}& M_s \ind\{M_s \ge 0\} \log\rbr{1 + \frac{z_s^*}{\tau_s} + \abr{\frac{z_s^*}{\tau_s}}^{1+\epsilon}}\\
    &- M_s \ind\{M_s < 0\} \log\rbr{1 - \frac{z_s^*}{\tau_s} + \abr{\frac{z_s^*}{\tau_s}}^{1+\epsilon}}.
\end{align*}
Thus
\begin{align*}
    \exp\cbr{M_s \Psi_1\rbr{\frac{z_s^*}{\tau_s}}} &\le \rbr{1 + \frac{z_s^*}{\tau_s} + \abr{\frac{z_s^*}{\tau_s}}^{1+\epsilon}}^{M_s \ind\{M_s \ge 0\}} \rbr{1 - \frac{z_s^*}{\tau_s} + \abr{\frac{z_s^*}{\tau_s}}^{1+\epsilon}}^{- M_s \ind\{M_s < 0\}}\\
    &\overset{(a)}{\le} 1 + M_s \frac{z_s^*}{\tau_s} + |M_s| \abr{\frac{z_s^*}{\tau_s}}^{1+\epsilon}\\
    &\overset{(b)}{\le} 1 + M_s \frac{z_s^*}{\tau_s} + \abr{\frac{z_s^*}{\tau_s}}^{1+\epsilon},
\end{align*}
where $(a)$ holds due to $(1+u)^v \le 1+uv, \forall u \ge -1, v \in (0,1]$ and $(b)$ holds due to $|M_s| \le 1$.
Then it follows that
\begin{align*}
    \EE\sbr{\exp\cbr{M_s \Psi_1\rbr{\frac{z_s^*}{\tau_s}}} \bigg|\cF_{s-1}} &\le \EE\sbr{1 + M_s \frac{z_s^*}{\tau_s} + \abr{\frac{z_s^*}{\tau_s}}^{1+\epsilon} \bigg| \cF_{s-1}}\\
    &= 1 + \frac{M_s}{\tau_s}\EE[z_s^*|\cF_{s-1}] + \EE\sbr{\abr{\frac{z_s^*}{\tau_s}}^{1+\epsilon} \bigg| \cF_{s-1}}\\
    &\le 1 + \frac{b^{1+\epsilon}}{\tau_s^{1+\epsilon}},
\end{align*}
where the first equality holds due to $\frac{M_s}{\tau_s} \in \cF_{s-1}$.
Hence, we have
\begin{align*}
    &\EE\sbr{\exp\cbr{\sum_{s=1}^t M_s \Psi_1\rbr{\frac{z_s^*}{\tau_s}}}}\\
    \overset{(a)}{=}{}& \EE\sbr{\exp\cbr{\sum_{s=1}^{t-1} M_s \Psi_1\rbr{\frac{z_s^*}{\tau_s}}}\EE\sbr{\exp\cbr{M_t \Psi_1 \rbr{\frac{z_t^*}{\tau_t}}} \bigg| \cF_{t-1}}}\\
    \le{}& \rbr{1+\frac{b^{1+\epsilon}}{\tau_t^{1+\epsilon}}} \EE\sbr{\exp\cbr{\sum_{s=1}^{t-1} M_s \Psi_1\rbr{\frac{z_s^*}{\tau_s}}}} \le \dots \le \prod_{s=1}^t \rbr{1 + \frac{b^{1+\epsilon}}{\tau_s^{1+\epsilon}}}\\
    \overset{(b)}{\le}{}& \exp\cbr{b^{1+\epsilon} \sum_{s=1}^t \frac{1}{\tau_s^{1+\epsilon}}} = \exp\cbr{\rbr{\frac{b}{\tau_0}}^{1+\epsilon} \sum_{s=1}^t \rbr{\frac{w_s}{\sqrt{1+w_s^2}}}^{1+\epsilon} s^{-\frac{1-\epsilon}{2}}}\\
    \overset{(c)}{\le}{}& \exp\cbr{\rbr{\frac{b}{\tau_0}}^{1+\epsilon} (2\kappa)^{\frac{1+\epsilon}{2}} (\log 3t)^{\frac{1-\epsilon}{2}}} = \exp\cbr{\frac{(\sqrt{2\kappa} b)^{1+\epsilon} (\log 3t)^{\frac{1-\epsilon}{2}}}{\tau_0^{1+\epsilon}}},
\end{align*}
where $(a)$ holds due to tower property of conditional expectation and $\sum_{s=1}^{t-1} M_s \Psi_1\rbr{\frac{z_s^*}{\tau_s}} \in \cF_{t-1}$, $(b)$ holds due to $1+x \le \exp\{x\}$ and $(c)$ holds with the same argument as \eqref{eq:holder-sum}.

Then, for any $t \ge 1$, a high probability upper bound of $\sum_{s=1}^t M_s \Psi_1\rbr{\frac{z_s^*}{\tau_s}}$ can be constructed by
\begin{align*}
    &\PP\rbr{\sum_{s=1}^t M_s \Psi_1\rbr{\frac{z_s^*}{\tau_s}} \ge \frac{(\sqrt{2\kappa} b)^{1+\epsilon} (\log 3t)^{\frac{1-\epsilon}{2}}}{\tau_0^{1+\epsilon}} + \log \frac{2t^2}{\delta}}\\
    \overset{(a)}{=} {}& \PP\rbr{\exp\cbr{\sum_{s=1}^t M_s \Psi_1\rbr{\frac{z_s^*}{\tau_s}}} \ge \exp\cbr{\frac{(\sqrt{2\kappa} b)^{1+\epsilon} (\log 3t)^{\frac{1-\epsilon}{2}}}{\tau_0^{1+\epsilon}}} \cdot \frac{2t^2}{\delta}}\\
    \overset{(b)}{\le} {}& \frac{\EE\sbr{\exp\cbr{\sum_{s=1}^t M_s \Psi_1\rbr{\frac{z_s^*}{\tau_s}}}}}{\exp\cbr{\frac{(\sqrt{2\kappa} b)^{1+\epsilon} (\log 3t)^{\frac{1-\epsilon}{2}}}{\tau_0^{1+\epsilon}}}} \frac{\delta}{2t^2} \le \frac{\delta}{2t^2},
\end{align*}
where $(a)$ holds due to the non-decreasing property of $\exp\{x\}$ and $(b)$ holds due to Markov's inequality.

Finally, with a union bound over $t\ge 1$ and the fact that $\sum_{t=1}^{\infty} \frac{1}{2t^2} \le 1$, with probability at least $1-\delta, \forall t\ge 1$, we have
\begin{align*}
    &\sum_{s=1}^t \Psi_{\tau_s} (z_s^*) \frac{2}{1+w_s^2} \frac{\bd_{s-1}^\top \bH_{s-1}^{-1} \bphi_s \ind_{A_{s-1}}}{\sigma_s}\\
    = {}& (\max_{s\le t} \alpha_s) 2 \tau_0 t^{\frac{1-\epsilon}{2(1+\epsilon)}} \sum_{s=1}^t M_s \Psi_1\rbr{\frac{z_s^*}{\tau_s}}\\
    \le {}& (\max_{s\le t} \alpha_s) 2 \tau_0 t^{\frac{1-\epsilon}{2(1+\epsilon)}} \sbr{\frac{(\sqrt{2\kappa} b)^{1+\epsilon} (\log 3t)^{\frac{1-\epsilon}{2}}}{\tau_0^{1+\epsilon}} + \log \frac{2t^2}{\delta}}\\
    = {}& (\max_{s\le t} \alpha_s) 2 t^{\frac{1-\epsilon}{2(1+\epsilon)}} \sbr{\frac{(\sqrt{2\kappa} b)^{1+\epsilon} (\log 3t)^{\frac{1-\epsilon}{2}}}{\tau_0^{\epsilon}} + \tau_0 \log \frac{2t^2}{\delta}}.
\end{align*}

\end{proof}

\subsection{Proof of Lemma~\ref{lem:I2}}\label{proof:I2}

\begin{proof}[Proof of Lemma~\ref{lem:I2}]

It follows that 
\[
\sum_{s=1}^{t} \Psi_{\tau_s}^2 (z_s^*) \frac{w_s^2}{1+w_s^2} = \sum_{s=1}^{t} \tau_s^2 \Psi_{1}^2 \rbr{\frac{z_s^*}{\tau_s}} \frac{w_s^2}{1+w_s^2} \le \tau_0^2 t^{\frac{1-\epsilon}{1+\epsilon}} \sum_{s=1}^t \Psi_{1}^2 \rbr{\frac{z_s^*}{\tau_s}}.
\]

Define $X_s := \Psi_{1}^2 \rbr{\frac{z_s^*}{\tau_s}}$ and $Y_s = X_s - \EE[X_s|\cF_{s-1}]$, which are $\cF_s$-measurable. Next we make use of Lemma~\ref{lem:freedman}, with probability at least $1-2t^2/\delta$, for any $t\ge 1$, we have
\begin{equation}\label{eq:lem4-freedman}
    \sum_{s=1}^{t} Y_s \le \sqrt{2 V \log\frac{2t^2}{\delta}} + \frac{2M}{3}\log \frac{2t^2}{\delta}
\end{equation}
where $\sum_{s=1}^t \EE[Y_s^2|\cF_{s-1}] \le V \text{ and } |Y_t| \le M$.
It follows that $|Y_s| \le |X_s| + |\EE[X_s|\cF_{s-1}]| \le 2$ by \ref{prop:huber-1} of Property~\ref{prop:huber} and
\begin{equation}\label{eq:I2-mean-sum}
\begin{aligned}
    \sum_{s=1}^t \EE[X_s|\cF_{s-1}] &\overset{(a)}{\le} \sum_{s=1}^t \EE\sbr{\abr{\frac{z_s^*}{\tau_s}}^{1+\epsilon} \bigg| \cF_{s-1}}\\
    &\le \sum_{s=1}^t \rbr{\frac{b}{\tau_s}}^{1+\epsilon} = \rbr{\frac{b}{\tau_0}}^{1+\epsilon} \sum_{s=1}^t \rbr{\frac{w_s}{\sqrt{1+w_s^2}}}^{1+\epsilon} s^{-\frac{1-\epsilon}{2}}\\
    &\overset{(b)}{\le} \rbr{\frac{\sqrt{2\kappa}b}{\tau_0}}^{1+\epsilon} (\log 3t)^{\frac{1-\epsilon}{2}},
\end{aligned}
\end{equation}

where (a) holds due to \ref{prop:huber-1} of Property~\ref{prop:huber} and (b) holds with the same argument as \eqref{eq:holder-sum}.

Through conducting a similar argument as \eqref{eq:I2-mean-sum}, we have
$$
\sum_{s=1}^t \EE[Y_s^2|\cF_{s-1}] \le \sum_{s=1}^t \EE[X_s^2|\cF_{s-1}] \le \rbr{\frac{\sqrt{2\kappa}b}{\tau_0}}^{1+\epsilon} (\log 3t)^{\frac{1-\epsilon}{2}}.
$$

Thus, setting $V = \rbr{\frac{\sqrt{2\kappa}b}{\tau_0}}^{1+\epsilon} (\log 3t)^{\frac{1-\epsilon}{2}}$ and $M=2$ in \eqref{eq:lem4-freedman}, we have
\begin{align*}
    \sum_{s=1}^t X_s &\le \sum_{s=1}^t \EE[X_s|\cF_{s-1}] + \sqrt{2 V \log\frac{2t^2}{\delta}} + \frac{2M}{3}\log \frac{2t^2}{\delta}\\
    &\le \rbr{\frac{\sqrt{2\kappa}b}{\tau_0}}^{1+\epsilon} (\log 3t)^{\frac{1-\epsilon}{2}} + \sqrt{2 \rbr{\frac{\sqrt{2\kappa}b}{\tau_0}}^{1+\epsilon} (\log 3t)^{\frac{1-\epsilon}{2}} \log\frac{2t^2}{\delta}} + \frac{4}{3}\log \frac{2t^2}{\delta}.
\end{align*}

Finally, with a union bound over $t \ge 1$ and the fact that $\sum_{t=1}^{\infty} \frac{1}{2t^2} \le 1$, with probability at least $1-\delta, \forall t\ge 1$, we have
\begin{align*}
    &\sum_{s=1}^{t} \Psi_{\tau_s}^2 (z_s^*) \frac{w_s^2}{1+w_s^2}\\
    \le {}& \tau_0^2 t^{\frac{1-\epsilon}{1+\epsilon}}  \sbr{\rbr{\frac{\sqrt{2\kappa}b}{\tau_0}}^{1+\epsilon} (\log 3t)^{\frac{1-\epsilon}{2}} + \sqrt{2 \rbr{\frac{\sqrt{2\kappa}b}{\tau_0}}^{1+\epsilon} (\log 3t)^{\frac{1-\epsilon}{2}} \log\frac{2t^2}{\delta}} + \frac{4}{3}\log \frac{2t^2}{\delta}}\\
    \le {}& \tau_0^2 t^{\frac{1-\epsilon}{1+\epsilon}} \sbr{\sqrt{ \rbr{\frac{\sqrt{2\kappa}b}{\tau_0}}^{1+\epsilon} (\log 3t)^{\frac{1-\epsilon}{2}}} + \sqrt{2 \log \frac{2t^2}{\delta}}}^2\\
    = {}& \sbr{t^{\frac{1-\epsilon}{2(1+\epsilon)}} \rbr{\sqrt{\tau_0^{1-\epsilon} (\sqrt{2\kappa} b)^{1+\epsilon} (\log 3t)^{\frac{1-\epsilon}{2}}} + \tau_0\sqrt{2 \log \frac{2t^2}{\delta}}}}^2.
\end{align*}

\end{proof}

\section{Proofs for Section~\ref{sec:bandit}}\label{proof:bandit}

\subsection{Proof of Theorem~\ref{thm:bandit-regret}}\label{proof:bandit-regret}

\begin{proof}[Proof of Theorem~\ref{thm:bandit-regret}]
Note that with probability at least $1-3\delta$, it holds that
\begin{equation}\label{eq:proof-bandit-help1}
    \|\btheta_t - \btheta^*\|_{\bH_t} \le \beta_t
\end{equation}
by Theorem~\ref{thm:heavy} with $b=1$. Then we have
\begin{align*}
    \mathrm{Regret}(T) &= \sum_{t=1}^T \sbr{\dotp{\bphi_t^*}{\btheta^*} - \dotp{\bphi_t}{\btheta^*}}\\
    &\le \sum_{t=1}^T \sbr{\max_{\btheta \in \cC_{t-1}}\dotp{\bphi_t^*}{\btheta} - \dotp{\bphi_t}{\btheta^*}}\\
    &\overset{(a)}{\le} \sum_{t=1}^T \sbr{\max_{\btheta \in \cC_{t-1}}\dotp{\bphi_t}{\btheta} - \dotp{\bphi_t}{\btheta^*}}\\
    &\overset{(b)}{\le} \sum_{t=1}^T \nbr{\bphi_t}_{\bH_{t-1}^{-1}} \max_{\btheta \in \cC_{t-1}}\nbr{\btheta - \btheta^*}_{\bH_{t-1}}\\
    &\overset{(c)}{\le} \sum_{t=1}^T 2 \beta_t \nbr{\bphi_t}_{\bH_{t-1}^{-1}}\\
    &\overset{(d)}{\le} 2 \beta_T \sum_{t=1}^T \sigma_t w_t,
\end{align*}
where $(a)$ holds due to the optimism of action $\bphi_t$, $(b)$ holds due to Cauchy-Schwartz inequality, $(c)$ holds due to \eqref{eq:proof-bandit-help1} and $(d)$ holds due to the fact that $\beta_t$ is increasing with $t$.

Notice
\begin{equation}\label{eq:proof-bandit-w-sum}
\sum_{t=1}^T w_t^2 = \min\cbr{1,w_t^2} \le 2\kappa
\end{equation}
by $w_s \le c_0 \le 1$ and Lemma~\ref{lem:w-sum}. We will use \eqref{eq:proof-bandit-w-sum} several times in the following proof.

Next we bound the sum of bonus $\sum_{t=1}^T \sigma_t w_t$ separately by the value of $\sigma_t$. Recall the definition of $\sigma_t$ in Algorithm~\ref{algo:bandit}, we decompose $[T]$ as the union of three disjoint sets $\cJ_1, \cJ_2, \cJ_3$ where
\begin{align*}
\cJ_1 &= \cbr{t\in[T] | \sigma_t\in\cbr{\nu_t,\sigmamin}},\\
\cJ_2 &= \cbr{t\in[T] \bigg| \sigma_t = \frac{\|\bphi_t\|_{\bH_{t-1}^{-1}}}{c_0}},\\
\cJ_3 &= \cbr{t\in[T] \Bigg| \sigma_t = \frac{\sqrt{LB}}{c_1^{\frac{1}{4}} (2\kappa)^{\frac{1}{4}}} \|\bphi_t\|_{\bH_{t-1}^{-1}}^{\frac{1}{2}}}.
\end{align*}

For the summation over $\cJ_1$, we have
\[
\sum_{t\in\cJ_1} \sigma_t w_t = \sum_{t\in\cJ_1} \max\cbr{\nu_t, \sigmamin} w_t
\overset{(a)}{\le} \sqrt{\sum_{t=1}^T (\nu_t^2+\sigmamin^2)} \sqrt{\sum_{t=1}^T w_t^2}
\overset{(b)}{\le} \sqrt{2\kappa} \sqrt{\sum_{t=1}^T (\nu_t^2 + \sigmamin^2)},
\]
where $(a)$ holds due to Cauchy-Schwartz inequality and $(b)$ holds due to \eqref{eq:proof-bandit-w-sum}.

And for the summation over $\cJ_2$, we have $w_t = \nbr{\frac{\bphi_t}{\sigma_t}}_{\bH_{t-1}^{-1}} = c_0$. Then
\[
\sum_{t\in\cJ_2} \sigma_t w_t = \frac{1}{c_0^2} \sum_{t\in\cJ_2} \nbr{\bphi_t}_{\bH_{t-1}^{-1}} w_t^2
\overset{(a)}{\le} \frac{L}{c_0^2\sqrt{\lambda}} \sum_{t=1}^T w_t^2
\overset{(b)}{\le} \frac{2L\kappa}{c_0^2\sqrt{\lambda}},
\]
where $(a)$ holds due to $\bH_{t-1} \succeq \lambda \bI$, thus $\nbr{\bphi_t}_{\bH_{t-1}^{-1}} \le \frac{\nbr{\bphi_t}}{\sqrt{\lambda}} \le \frac{L}{\sqrt{\lambda}}$. And $(b)$ holds due to \eqref{eq:proof-bandit-w-sum}.

Then, for the summation over $\cJ_3$, we have $\sigma_t = \frac{LB}{\sqrt{2 c_1 \kappa}}w_t$. Therefore
$$
\sum_{t\in\cJ_3} \sigma_t w_t = \frac{LB}{\sqrt{2c_1\kappa}} \sum_{t\in\cJ_3} w_t^2 \le \frac{LB \sqrt{2\kappa}}{\sqrt{c_1}},
$$
where the inequality holds due to \eqref{eq:proof-bandit-w-sum}.

Finally, putting pieces together finishes the proof.
\end{proof}

\section{Proof of Theorem~\ref{thm:computational-complexity}}
\label{proof:computational-complexity}

\begin{proof}[Proof of Theorem~\ref{thm:computational-complexity}]
Recall our proposed algorithm $\algomdp$ is detailed in Algorithm~\ref{algo:mdp}.
First, to compute $\btheta_{k-1,h}$ in line~\ref{line:thetakh}, we notice the loss function in \eqref{eq:mdp-btheta} is $\lambda_R$-strongly convex and $(\lambda_R+K/\nu_\mathrm{min}^2)$-smooth, so there are plenty of convex optimization algorithms available. For example, Nesterov accelerated method can be used. According to \citet{bubeck2015convex}, the number of iteration of Nesterov's method is ${O}(\sqrt{\beta/\alpha}\log(R^2/\epsilon))$ with one derivation ($O(d)$ operations) per iteration. Here the loss function is supposed to be $\alpha$-strongly convex and $\beta$-smooth. $R$ is the maximum distance of two points and $\epsilon$ is the precision. Thus the total computational cost is $\tilde{O}(HK\mathcal{R})$ with $\mathcal{R} = \tilde{O}\rbr{d\sqrt{1+K/(\lambda_R\nu_\mathrm{min}^2})}=\tilde{O}\rbr{d+d^{-\frac{1-\epsilon}{2(1+\epsilon)}} H^{\frac{1-\epsilon}{2(1+\epsilon)}} K^{\frac{1+2\epsilon}{2(1+\epsilon)}}}$.

Second, to evaluate the updated action-value function $Q^k_h(s,a)$ in line~\ref{line:Qkh} for a given pair $(s,a)$, we take the minimum over at most $\tilde{O}(dH)$ action-value functions (See Lemma~\ref{lem:rare-update}) with $O(d^2)$ operations (Using Sherman-Morrison formula to compute $H_{k-1,h}^{-1}$ and $\Sigma_{k-1,h}^{-1}$) for each function. Thus it takes $\tilde{O}(d^3H)$ to evaluate the updated action-value function. As a result, to compute $\hat{w}_{k-1,h}$ in line~\ref{line:thetakh}, notice $\hat{w}_{k,h}=\Sigma_{k,h}^{-1} \sum_{i=1}^k \sigma_{i,h}^{-2} \phi_{i,h} V^k_h(s_{i,h+1})$, if $V^k_h$ remains unchanged, we only need to compute the new term $\sigma_{k,h}^{-2} \phi_{k,h} V^k_h(s_{k,h+1})$, which takes $\tilde{O}(d^3|\mathcal{A}|H)$ computational time. Else if $V^k_h$ is updated, we need to recalculate $\lbrace V^k_h(s_{i,h+1}) \rbrace_{i\in[k]}$, which takes $\tilde{O}(d^3|\mathcal{A}|HK)$ computational time. Note the number of updating episode is at most $\tilde{O}(dH)$ and the length of each episode is $H$, so the total computational cost is $\tilde{O}(d^4|\mathcal{A}|H^3K)$.

Last, to take action $a_{k,h}$ in line~\ref{line:akh}, we need to compute $\lbrace Q^k_h(s_{k,h},a) \rbrace_{a\in\mathcal{A}}$ and take the maximum, which takes $\tilde{O}(d^3|\mathcal{A}|H)$ time, incurring a total cost of $\tilde{O}(d^3|\mathcal{A}|H^2K)$. Finally, combining the total costs above gives the computational complexity of $\algomdp$.
\end{proof}

\section{Proof of Theorem~\ref{thm:mdp-regret}}\label{proof:mdp-regret}

In this section, we give the proof sketch of Theorem~\ref{thm:mdp-regret}. In Appendix~\ref{sec:mdp-events}, leveraging the technique in \citet{he2023nearly}, we define several events and prove them hold with high probability.
In Appendix~\ref{sec:mdp-regret}, we then decompose the regret into a lower order term and summations of bonus terms with respect to reward functions and transition probabilities. Finally, we adopt a novel approach that deal with the two bonus terms separately.

\subsection{High-Probability Events}\label{sec:mdp-events}

\paragraph{Parameters for Adaptive Huber Regression in Algorithm~\ref{algo:mdp}.}

First, we set the parameters for adaptive Huber regression in Algorithm~\ref{algo:mdp} as follows:
\begin{gather*}
c_0 = \frac{1}{\sqrt{23\log \frac{2HK^2}{\delta}}},
\quad c_1 = \min\cbr{\frac{(\log 3K)^{\frac{1-\epsilon}{1+\epsilon}}}{48\rbr{\log \frac{2HK^2}{\delta}}^{\frac{2}{1+\epsilon}}}, \frac{(\log 3K)^{\frac{1-\epsilon'}{1+\epsilon'}}}{48\rbr{\log \frac{2HK^2}{\delta}}^{\frac{2}{1+\epsilon'}}}},\\
\tau_0 = \frac{\sqrt{2\kappa} (\log 3K)^{\frac{1-\epsilon}{2(1+\epsilon)}}}{\rbr{\log \frac{2HK^2}{\delta}}^{\frac{1}{1+\epsilon}}},
\quad \ttau_0 = \frac{\sqrt{2\kappa} \nu_{R^\epsilon} (\log 3K)^{\frac{1-\epsilon'}{2(1+\epsilon')}}}{\numin \rbr{\log \frac{2HK^2}{\delta}}^{\frac{1}{1+\epsilon'}}}.
\end{gather*}

\paragraph{Measurability.}
We define filtration $\cbr{\cG_{k,h}}_{(k,h)\in[K]\times[H]}$ and $\cbr{\cF_{k,h}}_{(k,h)\in[K]\times[H]}$ as follows. Let $I_{k,h} := \{ (i, j): i \in [k-1], j \in [H] \ \text{or} \ i =k, j\in[h] \}$ denote the set of index pairs up to and including the $k$-th episode and the $h$-th step. We further define $\cG_{k,h} = \sigma \rbr{ \bigcup_{(i, j) \in I_{k,h-1} }\cbr{ s_{i, j}, a_{i, j}, r_{i, j} } \cup \cbr{s_{k,h},a_{k,h}}}$ and $\cF_{k,h} = \sigma \rbr{ \bigcup_{(i, j) \in I_{k,h} }\cbr{ s_{i, j}, a_{i, j}, r_{i, j} } }$.
We make a convention that $\cG_{k,0} = \cG_{k-1,H+1}$ and $\cF_{k,0} = \cF_{k-1,H}$.
Note $\cG_{k,h} \subset \cF_{k,h} \subset \cG_{k,h+1}$.

We first introduce the following high-probability events:
\begin{enumerate}[leftmargin=*]
    \item{We define $\cE_{R^\epsilon}$ as the event that the following inequalities hold for all $k,h \in [K]\times[H]$,
    $$
    \nbr{\bpsi_{k,h} - \bpsi_h^*}_{\bH_{k,h}^{-1}} \le \beta_{R^\epsilon,k},
    $$
    where $\bpsi_{k,h}$ is defined in \eqref{eq:mdp-bpsi}
    \begin{equation}\label{eq:mdp-beta_Reps}
    \beta_{R^\epsilon,k} = 3 \sqrt{\lambda_R}W + 24 k^{\frac{1-\epsilon'}{2(1+\epsilon')}} \frac{\sqrt{2\kappa}\nu_{R^\epsilon}}{\numin} (\log 3K)^{\frac{1-\epsilon'}{2(1+\epsilon')}} \rbr{\log \frac{2HK^2}{\delta}}^{\frac{\epsilon'}{1+\epsilon'}}.
    \end{equation}
    For simplicity, we further define $\beta_{R^\epsilon} := \beta_{R^\epsilon,K}$.
    }
    \item{We define $\cE_0$ as the event that the following inequalities hold for all $k,h \in [K]\times[H]$,
    \begin{align*}
    \nbr{\hbw_{k-1,h} - \bw_{h}\sbr{V^k_{h+1}}}_{\bSigma_{k-1,h}} &\le \beta_0,\\
    \nbr{\cbw_{k-1,h} - \bw_{h}\sbr{\pesV^k_{h+1}}}_{\bSigma_{k-1,h}} &\le \beta_0,\\
    \nbr{\tbw_{k-1,h} - \bw_{h}\sbr{\rbr{V^k_{h+1}}^2}}_{\bSigma_{k-1,h}} &\le \cH\beta_0,
    \end{align*}
    where $\hbw_{k-1,h},\cbw_{k-1,h},\tbw_{k-1,h}$ are defined in \eqref{eq:mdp-bw} and
    \begin{equation}\label{eq:mdp-beta_0}
    \beta_0 = 2\sqrt{d\lambda_V}\cH + \frac{3\cH}{\sigmamin} \sqrt{d^3H\iota_0^2 + \log\frac{H}{\delta}},
    \end{equation}
    and
    \begin{align*}
    \iota_0  = \max\Bigg\{&\log_2\rbr{1 + \frac{K}{\lambda_R\nu_{\min}^2}}, \log_2\rbr{1 + \frac{K}{\lambda_V\sigma_{\min}^2}}, \log\rbr{1+ \frac{8(B+L)K}{\lambda_V\cH \sqrt{d} \sigma_{\min}^2}},\\
    &\log\rbr{1+\frac{32\beta_{R}^2K^2}{\sqrt{d}\lambda_R\lambda_V^2 \cH^2 \sigma_{\min}^4}}, \log\rbr{1+\frac{32\beta_V^2K^2}{\sqrt{d}\lambda_V^3 \cH^2 \sigma_{\min}^4}}\Bigg\}, L = \cH\sqrt{\frac{dK}{\lambda_V}}.
    \end{align*}
    }
    \item{We define $\cE_{R,k}$ as the event that the following inequalities hold for all $i,h \in [k]\times[H]$,
    $$
    \nbr{\btheta_{i,h} - \btheta_h^*}_{\bH_{i,h}^{-1}} \le \beta_{R,i},
    $$
    where $\btheta_{i,h}$ is defined in \eqref{eq:mdp-btheta} and
    \begin{equation}\label{eq:mdp-beta_R}
    \beta_{R,i} = O\rbr{\sqrt{\lambda_R}B + \sqrt{d} i^{\frac{1-\epsilon}{2(1+\epsilon)}} \iota},
    \end{equation}
    and
    \[
    \iota  = \max\cbr{\log\rbr{1 + \frac{K}{d\lambda_R \numin^2}}, \log 3K, \log \frac{2HK^2}{\delta}}.
    \]
    For simplicity, we further define $\beta_R := \beta_{R,K} = O\rbr{\sqrt{\lambda_R}B + \sqrt{d} K^{\frac{1-\epsilon}{2(1+\epsilon)}} \iota}$ and $\cE_R:=\cE_{R,K}$.
    }
    \item{We define $\cE_h$ as the event that the following inequalities hold for all $k\in [K], h \le h' \le H$,
    \begin{align*}
    \nbr{\hbw_{k-1,h'} - \bw_{h'}\sbr{V^k_{h'+1}}}_{\bSigma_{k-1,h'}} &\le \beta_V,\\
    \nbr{\cbw_{k-1,h'} - \bw_{h'}\sbr{\pesV^k_{h'+1}}}_{\bSigma_{k-1,h'}} &\le \beta_V.
    \end{align*}
    where
    \begin{equation}\label{eq:mdp-beta_V}
    \beta_V = O\rbr{\sqrt{d\lambda_V}\cH + \sqrt{d}\iota_1^2},
    \end{equation}
    and
    \begin{align*}
    \iota_1  = \max\Bigg\{&\iota_0, \log \rbr{ 1 + \frac{K}{\sigma_{\min}^2d\lambda_V}}, \log\frac{4HK^2}{\delta}, \log\rbr{1 + \frac{4(B+L)\sqrt{d^3H}}{\sigma_{\min}}},\\
    &\log \rbr{1 + \frac{8\sqrt{d^7}H\beta_{R}^2}{\lambda_R \sigma_{\min}^2}}, \log \rbr{1 + \frac{8\sqrt{d^7}H\beta_{V}^2}{\lambda_V \sigma_{\min}^2}}\Bigg\}, L = \cH\sqrt{\frac{dK}{\lambda_V}}.
    \end{align*}
    For simplicity, we further define $\cE_V:=\cE_1$.
    }
\end{enumerate}

Our ultimate goal is to show $\cE_R \cap \cE_V$ holds with high probability, which is a `refined' event where the radius $\beta_R, \beta_V$ are smaller than $\beta_{R^\epsilon},\beta_0$ in the `coarse' event $\cE_{R^\epsilon}\cap\cE_0$. Leveraging the technique in \citet{he2023nearly}, we first prove event $\cE_{R^\epsilon}\cap\cE_0$ holds with high probability, then come to $\cE_R \cap \cE_V$.

\begin{lemma}\label{lem:mdp-EReps}
    Event $\cE_{R^\epsilon}$ holds with probability at least $1-3\delta$.
\end{lemma}
\begin{proof}
See Appendix~\ref{proof:mdp-EReps} for a detailed proof.
\end{proof}

\begin{lemma}\label{lem:mdp-W}
    On event $\cE_{R^\epsilon}\cap\cE_{R,k}$, for all $h\in[H]$, we have
    $$
    \abr{\sbr{\hat{\bnu}_{1+\epsilon} R_h - \bnu_{1+\epsilon} R_h}(s_{k,h},a_{k,h})} \le W_{k,h},
    $$
    where $\hat{\bnu}_{1+\epsilon}R_h$ and $W_{k,h}$ are defined in \eqref{eq:mdp-hatbnu} and \eqref{eq:mdp-W} respectively.
\end{lemma}
\begin{proof}
See Appendix~\ref{proof:mdp-W} for a detailed proof.
\end{proof}

\begin{lemma}\label{lem:mdp-ER}
    On event $\cE_{R^\epsilon}$, event $\cE_R$ holds with probability at least $1-3\delta$.
\end{lemma}
\begin{proof}
See Appendix~\ref{proof:mdp-ER} for a detailed proof.
\end{proof}

\begin{lemma}\label{lem:mdp-E0}
    Event $\cE_0$ holds with probability at least $1-3\delta$.
\end{lemma}
\begin{proof}
See Appendix~\ref{proof:mdp-E0} for a detailed proof.
\end{proof}

\begin{lemma}\label{lem:mdp-opt}
    On event $\cE_{R} \cap \cE_{h}$, for all $k\in[K]$, we have $\pesQ^k_{h}(\cdot,\cdot) \le Q^*_h(\cdot,\cdot) \le Q^k_{h}(\cdot,\cdot)$. In addition, we have $\pesV^k_h(\cdot) \le V^*_h(\cdot) \le V^k_h(\cdot)$.
\end{lemma}
\begin{proof}
See Appendix~\ref{proof:mdp-opt} for a detailed proof.
\end{proof}

\begin{lemma}\label{lem:mdp-E}
    On event $\cE_0 \cap \cE_{R} \cap \cE_{h+1}$, for all $k\in[K]$, we have
    $$
    \abr{\sbr{\hat{\VV}_h V^k_{h+1} - \VV_h V^*_{h+1}}(s_{k,h},a_{k,h})} \le E_{k,h},
    $$
    where $\hat{\VV}_h V^k_{h+1}$ and $E_{k,h}$ are defined in \eqref{eq:mdp-hatVV} and \eqref{eq:mdp-E} respectively.
\end{lemma}
\begin{proof}
See Appendix~\ref{proof:mdp-E} for a detailed proof.
\end{proof}

\begin{lemma}\label{lem:mdp-D}
    On event $\cE_0 \cap \cE_{R} \cap \cE_{h+1}$, for all $i \le k \le K$, we have
    $$
    \max\cbr{\sbr{\VV_h\rbr{V^k_{h+1} - V^*_{h+1}}}(s_{i,h},a_{i,h}), \sbr{\VV_h\rbr{V^*_{h+1} - \pesV^k_{h+1}}}(s_{i,h},a_{i,h})} \le D_{i,h},
    $$
    where $D_{i,h}$ is defined in \eqref{eq:mdp-D}.
\end{lemma}
\begin{proof}
See Appendix~\ref{proof:mdp-D} for a detailed proof.
\end{proof}

\begin{lemma}\label{lem:mdp-EV}
    With probability at least $1-2\delta$, on event $\cE_0\cap\cE_R$, event $\cE_V$ holds.
\end{lemma}
\begin{proof}
See Appendix~\ref{proof:mdp-EV} for a detailed proof.
\end{proof}

\subsection{Regret Analysis}\label{sec:mdp-regret}

In this section, we will prove the regret bound based on the events defined in Appendix~\ref{sec:mdp-events}, which hold with probability at least $1-11\delta$. By the optimism of Lemma~\ref{lem:mdp-opt}, we have
$$
\mathrm{Regret}(K) = \sum_{k=1}^K \sbr{V^*_1(s_{k,1}) - V^{\pi^k}_{1}(s_{k,1})} \le \sum_{k=1}^K \sbr{V^k_1(s_{k,1}) - V^{\pi^k}_{1}(s_{k,1})}.
$$
Next we bound the regret with the summations of two bonus terms, i.e. $\sum_{k=1}^K\sum_{h=1}^H \normbphiR{k}{h}$ and $\sum_{k=1}^K\sum_{h=1}^H \normbphiV{k}{h}$ with respect to reward functions and transition probabilities by Lemma~\ref{lem:mdp-subopt-gap}. Then we bound them separately by Lemma~\ref{lem:mdp-sumR} and Lemma~\ref{lem:mdp-sumV}.

\begin{lemma}\label{lem:mdp-subopt-gap}
    With probability at least $1-\delta$, on event $\cE_R\cap\cE_V$, it follows that
    \begin{align*}
    &\sum_{k=1}^K \sbr{V^k_1(s_{k,1}) - V^{\pi^k}_{1}(s_{k,1})}\\
    \le{}& 6\beta_R \sum_{k=1}^K\sum_{h=1}^H \normbphiR{k}{h} + 6\beta_V \sum_{k=1}^K\sum_{h=1}^H \normbphiV{k}{h} + 36H\cH\log\frac{4\ceil{\log_2 HK}}{\delta}.
    \end{align*}
    and
    \begin{align*}
    &\sum_{k=1}^K \sum_{h=1}^H \sbr{\PP_h(V_{h+1}^k - V_{h+1}^{\pi^k})}(s_{k,h}, a_{k,h}) \\
    \le{}& 8H\beta_R \sum_{k=1}^K \sum_{h=1}^H \normbphiR{k}{h} + 8H\beta_V \sum_{k=1}^K \sum_{h=1}^H \normbphiV{k}{h} + 29 H^2 \cH \log\frac{4\ceil{\log_2 HK}}{\delta}.
    \end{align*}
\end{lemma}
\begin{proof}
See Appendix~\ref{proof:mdp-subopt-gap} for a detailed proof.
\end{proof}

\begin{lemma}\label{lem:mdp-opt-pes-gap}
    With probability at least $1-\delta$, on event $\cE_R\cap\cE_V$, it follows that
    \begin{align*}
    &\sum_{k=1}^K \sum_{h=1}^H \sbr{\PP_h(V_{h+1}^k - \pesV_{h+1}^k)}(s_{k,h}, a_{k,h}) \\
    \le{}& 8H\beta_R \sum_{k=1}^K \sum_{h=1}^H \normbphiR{k}{h} + 16H\beta_V \sum_{k=1}^K \sum_{h=1}^H \normbphiV{k}{h} + 29 H^2 \cH \log\frac{4\ceil{\log_2 HK}}{\delta}.
    \end{align*}
\end{lemma}
\begin{proof}
See Appendix~\ref{proof:mdp-opt-pes-gap} for a detailed proof.
\end{proof}

\begin{lemma}\label{lem:mdp-sumR}
    Set $\lambda_R=\frac{d}{\max\cbr{B^2,W^2}}$. Then with probability at least $1-\delta$, on event $\cE_{R^\epsilon}\cap\cE_R$, we have
    \begin{align*}
    \sum_{k=1}^K\sum_{h=1}^H \normbphiR{k}{h} = \tilde{O}\Bigg(&\sqrt{dH\cU^*K} + \sqrt{dH^2K\numin^2} + \Bigg(\frac{\nu_{R^\epsilon} d^\frac{2+\epsilon}{2} H^\frac{1+\epsilon}{2} K^\frac{1-\epsilon'}{2(1+\epsilon')}}{\numin}\Bigg)^\frac{1}{\epsilon}\\
    &+ \rbr{d^\frac{4+\epsilon}{2} H^\frac{1+\epsilon}{2} \cH^\epsilon K^\frac{1-\epsilon}{2(1+\epsilon)}}^\frac{1}{\epsilon} + (\nu_R+\max\cbr{B,W})\sqrt{d}H\Bigg).
    \end{align*}
    Choosing $\numin = \tilde{O}\rbr{\nu_{R^\epsilon}^\frac{1}{1+\epsilon} d^\frac{1}{1+\epsilon} H^\frac{1-\epsilon}{2(1+\epsilon)} K^{-\frac{(1+\epsilon)(1+\epsilon')-2}{2(1+\epsilon)(1+\epsilon')}}}$ further gives
    \begin{align*}
    \sum_{k=1}^K\sum_{h=1}^H \normbphiR{k}{h} = \tilde{O}\bigg(&\sqrt{dH\cU^*K} + \nu_{R^\epsilon}^\frac{1}{1+\epsilon} d^\frac{3+\epsilon}{2(1+\epsilon)} H^\frac{3+\epsilon}{2(1+\epsilon)}K^\frac{1}{(1+\epsilon)(1+\epsilon')}\\
    &+ \rbr{d^\frac{4+\epsilon}{2} H^\frac{1+\epsilon}{2} \cH^\epsilon K^\frac{1-\epsilon}{2(1+\epsilon)}}^\frac{1}{\epsilon} + (\nu_R+\max\cbr{B,W})\sqrt{d}H\bigg).
    \end{align*}
\end{lemma}
\begin{proof}
See Appendix~\ref{proof:mdp-sumR} for a detailed proof.
\end{proof}

\begin{lemma}\label{lem:mdp-sumV}
    Set $\lambda_V=\frac{1}{\cH^2}$, let $\cA_0$ denote the event where Lemma~\ref{lem:mdp-subopt-gap} and Lemma~\ref{lem:mdp-opt-pes-gap} hold. Then with probability at least $1-2\delta$, on event $\cE_0\cap\cE_R\cap\cE_V\cap\cA_0$, we have
    \begin{align*}
    \sum_{k=1}^K\sum_{h=1}^H\normbphiV{k}{h} = \tilde{O}\Bigg(&\sqrt{dH\cV^*K} + \sqrt{dH^2K\sigmamin^2} + \frac{d^{5.5}H^{2.5}\cH^2}{\sigmamin}\\
    &+ \sqrt{d^4H^3\cH\beta_R\sum_{k=1}^K\sum_{h=1}^H\normbphiR{k}{h}} + d^{4.5}H^3\cH\bigg).
    \end{align*}
    Choosing $\sigmamin = \tilde{O}\rbr{\sqrt{d^5 H^{1.5} \cH^2}K^{-\frac{1}{4}}}$ further gives
    \begin{align*}
    \sum_{k=1}^K\sum_{h=1}^H\normbphiV{k}{h} = \tilde{O}\Bigg(&\sqrt{dH\cV^*K} + \sqrt{d^6 H^{3.5} \cH^2} K^\frac{1}{4} + \sqrt{d^4H^3\cH\beta_R\sum_{k=1}^K\sum_{h=1}^H\normbphiR{k}{h}}\\
    &+ d^{4.5}H^3\cH\Bigg).
    \end{align*}
\end{lemma}
\begin{proof}
See Appendix~\ref{proof:mdp-sumV} for a detailed proof.
\end{proof}

\ifarxiv
At the end of this section, we provide the proof of Theorem~\ref{thm:mdp-regret}.

\begin{proof}[Proof of Theorem~\ref{thm:mdp-regret}]
Based on high-probability events in Appendix~\ref{sec:mdp-events} and Lemmas in this section, we have
\begin{align*}
&\mathrm{Regret}(K)\\
={}& \sum_{k=1}^K \sbr{V^*_1(s_{k,1}) - V^{\pi^k}_{1}(s_{k,1})} \overset{(a)}{\le}{} \sum_{k=1}^K \sbr{V^k_1(s_{k,1}) - V^{\pi^k}_{1}(s_{k,1})}\\
\overset{(b)}{\le}{}& 6\beta_R \sum_{k=1}^K\sum_{h=1}^H \normbphiR{k}{h} + 6\beta_V \sum_{k=1}^K\sum_{h=1}^H \normbphiV{k}{h} + 36H\cH\log\frac{4\ceil{\log_2 HK}}{\delta}\\
\overset{(c)}{=}{}& \tilde{O}\bigg(d\sqrt{H\cU^*}K^\frac{1}{1+\epsilon} + d\sqrt{H\cV^*K} + \nu_{R^\epsilon}^\frac{1}{1+\epsilon} d^\frac{2+\epsilon}{1+\epsilon} H^\frac{3+\epsilon}{2(1+\epsilon)} K^\frac{2+(1-\epsilon)(1+\epsilon')}{2(1+\epsilon)(1+\epsilon')} + \sqrt{d^7 H^{3.5} \cH^2} K^\frac{1}{4}\\
\qquad{}& + \rbr{d^{2+\epsilon} H^\frac{1+\epsilon}{2} \cH^\epsilon K^\frac{1-\epsilon}{2}}^\frac{1}{\epsilon} + (\nu_R+\max\cbr{B,W})dHK^\frac{1-\epsilon}{2(1+\epsilon)} + d^5H^3\cH\bigg),
\end{align*}
where $(a)$ holds due to the optimism of Lemma~\ref{lem:mdp-opt}, (b) holds due to Lemma~\ref{lem:mdp-subopt-gap}. (c) holds due to Lemma~\ref{lem:mdp-sumR}, Lemma~\ref{lem:mdp-sumV} and
\[
\beta_V \sqrt{d^4H^3\cH\beta_R\sum_{k=1}^K\sum_{h=1}^H\normbphiR{k}{h}} \le \frac{1}{2}\beta_V^2 d^4H^3\cH + \frac{1}{2}\beta_R\sum_{k=1}^K\sum_{h=1}^H\normbphiR{k}{h}.
\]
\end{proof}
\else
At the end of this section, we state the formal version of Theorem~\ref{thm:mdp-regret} and provide its proof.

\begin{theorem}[Formal version of Theorem~\ref{thm:mdp-regret}]\label{thm:mdp-regret-formal}
    For the time-inhomogeneous linear MDPs with heavy-tailed rewards defined in Section~\ref{sec:pre-mdp}, we set parameters in Algorithm~\ref{algo:mdp} as follows: $\lambda_R = d/\max\cbr{B^2,W^2}$, $\lambda_V = 1/\cH^2$,
    $\numin$ in Lemma~\ref{lem:mdp-sumR}, $\sigmamin$ in Lemma~\ref{lem:mdp-sumV},
    $\beta_{R^\epsilon}$, $\beta_0$, $\beta_R$, $\beta_V$ in \eqref{eq:mdp-beta_Reps}, \eqref{eq:mdp-beta_0}, \eqref{eq:mdp-beta_R}, \eqref{eq:mdp-beta_V}, respectively.
    Then for any $\delta\in(0,1)$, with probability at least $1-16\delta$, the regret of $\algomdp$ is bounded by
    \begin{align*}
    \mathrm{Regret}(K) = \tilde{O}\bigg(&d\sqrt{H\cU^*}K^\frac{1}{1+\epsilon} + d\sqrt{H\cV^*K} + \nu_{R^\epsilon}^\frac{1}{1+\epsilon} d^\frac{2+\epsilon}{1+\epsilon} H^\frac{3+\epsilon}{2(1+\epsilon)} K^\frac{2+(1-\epsilon)(1+\epsilon')}{2(1+\epsilon)(1+\epsilon')}\\
    & + \sqrt{d^7 H^{3.5} \cH^2} K^\frac{1}{4} + \rbr{d^{2+\epsilon} H^\frac{1+\epsilon}{2} \cH^\epsilon K^\frac{1-\epsilon}{2}}^\frac{1}{\epsilon}\\
    & + (\nu_R+\max\cbr{B,W})dHK^\frac{1-\epsilon}{2(1+\epsilon)} + d^5H^3\cH\bigg),
    \end{align*}
    where
    \begin{align}
    \cU^* &= \min\cbr{\cU^*_0,\cU}\text{ with }\cU^*_0 = \frac{1}{K}\sum_{k=1}^K \EE_{(s_h,a_h)\sim d_h^{\pi^k}} \sbr{\sum_{h=1}^H \sbr{\bnu_{1+\epsilon} R_h}^\frac{2}{1+\epsilon}(s_h,a_h)},\label{eq:U}\\
    \cV^* &= \min\cbr{\cV^*_0,\cV}\text{ with }\cV^*_0 = \frac{1}{K}\sum_{k=1}^K \EE_{(s_h,a_h)\sim d_h^{\pi^k}} \sbr{\sum_{h=1}^H \sbr{\VV_h V^*_{h+1}}(s_h,a_h)},\label{eq:V}\\
    &\qquad\qquad\qquad\,\, d_h^{\pi^k}(s, a) = \PP^{\pi^k}((s_h,a_h)=(s,a)|s_0=s_{k,1})\label{eq:d}
    \end{align}
    and $\cH, \cU, \cV$ are defined in Assumption~\ref{ass:bounded}.
\end{theorem}
\begin{proof}[Proof of Theorem~\ref{thm:mdp-regret}]
Based on high-probability events in Appendix~\ref{sec:mdp-events} and Lemmas in this section, we have
\begin{align*}
&\mathrm{Regret}(K)\\
={}& \sum_{k=1}^K \sbr{V^*_1(s_{k,1}) - V^{\pi^k}_{1}(s_{k,1})} \overset{(a)}{\le}{} \sum_{k=1}^K \sbr{V^k_1(s_{k,1}) - V^{\pi^k}_{1}(s_{k,1})}\\
\overset{(b)}{\le}{}& 6\beta_R \sum_{k=1}^K\sum_{h=1}^H \normbphiR{k}{h} + 6\beta_V \sum_{k=1}^K\sum_{h=1}^H \normbphiV{k}{h} + 36H\cH\log\frac{4\ceil{\log_2 HK}}{\delta}\\
\overset{(c)}{=}{}& \tilde{O}\bigg(d\sqrt{H\cU^*}K^\frac{1}{1+\epsilon} + d\sqrt{H\cV^*K} + \nu_{R^\epsilon}^\frac{1}{1+\epsilon} d^\frac{2+\epsilon}{1+\epsilon} H^\frac{3+\epsilon}{2(1+\epsilon)} K^\frac{2+(1-\epsilon)(1+\epsilon')}{2(1+\epsilon)(1+\epsilon')} + \sqrt{d^7 H^{3.5} \cH^2} K^\frac{1}{4}\\
\qquad{}& + \rbr{d^{2+\epsilon} H^\frac{1+\epsilon}{2} \cH^\epsilon K^\frac{1-\epsilon}{2}}^\frac{1}{\epsilon} + (\nu_R+\max\cbr{B,W})dHK^\frac{1-\epsilon}{2(1+\epsilon)} + d^5H^3\cH\bigg),
\end{align*}
where $(a)$ holds due to the optimism of Lemma~\ref{lem:mdp-opt}, (b) holds due to Lemma~\ref{lem:mdp-subopt-gap}. (c) holds due to Lemma~\ref{lem:mdp-sumR}, Lemma~\ref{lem:mdp-sumV} and
\[
\beta_V \sqrt{d^4H^3\cH\beta_R\sum_{k=1}^K\sum_{h=1}^H\normbphiR{k}{h}} \le \frac{1}{2}\beta_V^2 d^4H^3\cH + \frac{1}{2}\beta_R\sum_{k=1}^K\sum_{h=1}^H\normbphiR{k}{h}.
\]
\end{proof}

\begin{remark}
If $\epsilon > \frac{1}{1+\epsilon'}$, then $K^\frac{2+(1-\epsilon)(1+\epsilon')}{2(1+\epsilon)(1+\epsilon')} < \sqrt{K}$. When the number of episodes $K$ is sufficiently large, the regret can be simplified to $\tilde{O}\rbr{d\sqrt{H\cU^*}K^\frac{1}{1+\epsilon} + d\sqrt{H\cV^*K}}$.
\end{remark}
\fi

\begin{remark}\label{rem:mdp-li2023variance-appendix}
When $\epsilon=1$, we return to the setting in \citet{li2023variance} and the regret reduces to $\tilde{O}(d\sqrt{H\cU^*K} + d\sqrt{H\cV^*K})$. While their variance-aware regret bound is $\tilde{O}(d\sqrt{H\cG^*K})$, where $\cG^*$ is a variance-dependent quantity defined in their work. We provide the definition of $\cG^*$ and its relationship with $\cU^*$, $\cV^*$ below.
\begin{align*}
\cG^* &= \min\cbr{\frac{1}{K}\sum_{k=1}^K \EE_{(s_h,a_h)\sim d_h^{\pi^k}} \sbr{\sum_{h=1}^H \sbr{\VV_h R_h + \VV_h V^*_{h+1}}(s_h,a_h)}, \cU + \cV}\\
& = \min\cbr{\cU^*_0 + \cV^*_0, \cU + \cV}\\
&\ge \min\cbr{\cU^*_0, \cU} + \min\cbr{\cV^*_0, \cV}\\
&= \cU^* + \cV^*,
\end{align*}
where $\cU^*$, $\cV^*$ and $d_h^{\pi^k}(s, a)$ are defined in Theorem~\ref{thm:mdp-regret}.
Thus we have $d\sqrt{H\cU^*K} + d\sqrt{H\cV^*K} \le 2 d\sqrt{H(\cU^* + \cV^*)K} \le 2 d\sqrt{H\cG^*K}$, which implies that we recover their result.
\end{remark}

\subsection{Proof of Corollary~\ref{cor:mdp-first-order}}\label{proof:mdp-first-order}

\begin{proof}[Proof of Corollary~\ref{cor:mdp-first-order}]
First, it holds that $\cV^* \le \cV \le \cH V^*_1$ according to \eqref{eq:mdp-first-order-relationship} in the proof of Lemma~\ref{lem:total-variance} in Appendix~\ref{proof:total-variance}.
Thus the first result follows.
Next, to make a fair comparison with the state-of-the-art result of first-order regret \citep{wagenmaker2022first}, we assume the reward functions are uniformly bounded, i.e., $R_h(s,a) \in [0,1]$ for all $(s,a)\in\cS\times\cA$ and $h\in[H]$.
Then $\epsilon=1$, $\cH = H$ and 
\begin{align*}
\cU^* \le \cU^*_0 &= \frac{1}{K}\sum_{k=1}^K \EE_{(s_h,a_h)\sim d_h^{\pi^k}} \sbr{\sum_{h=1}^H \sbr{\bnu_{1+\epsilon} R_h}^\frac{2}{1+\epsilon}(s_h,a_h)}\\
&\overset{(a)}{\le} \frac{1}{K}\sum_{k=1}^K \EE_{(s_h,a_h)\sim d_h^{\pi^k}} \sbr{\sum_{h=1}^H r_h(s_h,a_h)} \overset{(b)}{\le} V^*_1,
\end{align*}
where $(a)$ holds due to $R_h(s,a)$ is bounded in $[0,1]$ for all $(s,a)\in\cS\times\cA$ and $h\in[H]$, $(b)$ uses the optimality of $V^*_1$.
Finally, the proof is completed by the fact that $\cV^* \le H V^*_1$ and $\cU^* \le V^*_1$.
\end{proof}

\section{Proof of Lower Bound}
\label{proof:lower-bound-heavy}

\begin{proof}[Proof of Theorem~\ref{thm:lower-bound-heavy}]
The proof of Theorem~\ref{thm:lower-bound-heavy} follows from a combination of the lower bound constructions for heavy-tailed linear bandits in \citet{shao2018almost} and linear MDPs in \citet{zhou2021nearly}.
On one hand, we construct a linear MDP with deterministic transition probabilities by concatenating $H$ hard instances in \citet{shao2018almost} together. Summing the regret over the $H$ components yields $\Omega(d H K^{\frac{1}{1+\epsilon}})$.
On the other hand, \citet{zhou2021nearly} shows the regret is at least $\Omega(d \sqrt{H^3 K})$. Combining the results together gives the final lower bound.
\end{proof}

\section{Omitted Proofs in Appendix~\ref{proof:mdp-regret}}\label{proof:mdp}

\subsection{Proof of Lemma~\ref{lem:mdp-EReps}}\label{proof:mdp-EReps}

\begin{proof}[Proof of Lemma~\ref{lem:mdp-EReps}]
For each $h\in[H]$, we apply Theorem~\ref{thm:heavy} with $L = 1, B = W, T = k, \delta = \delta / H, \cF_t = \cG_{i,h}, y_t = |\varepsilon_{i,h}|^{1+\epsilon}, \btheta_t = \bpsi_{k,h}$ and $b = \nu_{R^\epsilon} / \numin$. We choose parameters $c_0, c_1,\ttau_0$ as in Theorem~\ref{thm:heavy}. Then with probability at least $1-3\delta$, $\nbr{\bpsi_{k,h} - \bpsi_h^*}_{\bH_{k,h}^{-1}} \le \beta_{R^\epsilon,k}$, that is event $\cE_{R^\epsilon}$ holds.
\end{proof}

\subsection{Proof of Lemma~\ref{lem:mdp-W}}\label{proof:mdp-W}

\begin{proof}[Proof of Lemma~\ref{lem:mdp-W}]
Notice that
\begin{align*}
&\abr{\sbr{\hat{\bnu}_{1+\epsilon} R_h - \bnu_{1+\epsilon} R_h}(s_{k,h},a_{k,h})}\\
={}& \abr{\dotp{\bphi_{k,h}}{\hbpsi_{k-1,h} - \bpsi^*_h}}\\
\le{}& \abr{\dotp{\bphi_{k,h}}{\hbpsi_{k-1,h} - \bpsi_{k-1,h}}} + \abr{\dotp{\bphi_{k,h}}{\bpsi_{k-1,h} - \bpsi^*_h}}\\
\le{}& \normbphiR{k}{h} \|\hbpsi_{k-1,h} - \bpsi_{k-1,h}\|_{\bH_{k-1,h}} + \normbphiR{k}{h} \nbr{\bpsi_{k-1,h} - \bpsi^*_h}_{\bH_{k-1,h}},
\end{align*}
where the last inequality holds due to Cauchy-Schwartz inequality. And
\begin{equation}\label{eq:proof-mdp-ER-help-1}
\nbr{\bpsi_{k-1,h} - \bpsi^*_h}_{\bH_{k-1,h}} \le \beta_{R^\epsilon,k-1}
\end{equation}
since $\cE_{R^\epsilon}$ holds.

We next give an upper bound of $\|\hbpsi_{k-1,h} - \bpsi_{k-1,h}\|_{\bH_{k-1,h}}$ by a novel perturbation analysis of adaptive Huber regression. For each $h\in[H]$, we apply Lemma~\ref{lem:perturbation} with $t = k-1, \btheta_t = \bpsi_{k-1,h}, \hbtheta_t = \hbpsi_{k-1,h}, \cF_s = \cG_{i,h}, y_s = |\varepsilon_{i,h}|^{1+\epsilon}, \hat{y}_s = |\hat{\varepsilon}_{i,h}|^{1+\epsilon}$. Notice
\begin{align*}
\abr{|\hat{\varepsilon}_{i,h}|^{1+\epsilon} - |\varepsilon_{i,h}|^{1+\epsilon}} &\le |\hat{\varepsilon}_{i,h} - \varepsilon_{i,h}|^{1+\epsilon}\\
&=\abr{\dotp{\bphi_{i,h}}{\btheta_{i,h} - \btheta^*_h}}^{1+\epsilon}\\
&\overset{(a)}{\le} \cH^\epsilon \nbr{\btheta_{i,h} - \btheta^*_h}_{\bH_{i,h}} \nbr{\bphi_{i,h}}_{\bH_{i,h}^{-1}}\\
&\overset{(b)}{\le} \cH^\epsilon \beta_{R,i-1} \nbr{\bphi_{i,h}}_{\bH_{i,h}^{-1}},
\end{align*}
where $(a)$ holds due to Cauchy-Schwartz inequality, and $\dotp{\bphi_{i,h}}{\btheta_{i,h}}$ together with $\dotp{\bphi_{i,h}}{\btheta^*_h}$ are bounded in $[0,\cH]$. And $(b)$ is due to $\cE_{R,k-1}$ holds.
Then we have $\hat{\beta}_s = \cH^\epsilon \beta_{R,i-1}$ in Lemma~\ref{lem:perturbation}, and it follows that
\begin{equation}\label{eq:proof-mdp-ER-help-2}
\|\hbpsi_{k-1,h} - \bpsi_{k-1,h}\|_{\bH_{k-1,h}} \le 6 \cH^\epsilon \beta_{R,k-1} \kappa.
\end{equation}
Combining \eqref{eq:proof-mdp-ER-help-1} with \eqref{eq:proof-mdp-ER-help-2}, we have
$$
\abr{\sbr{\hat{\bnu}_{1+\epsilon} R_h - \bnu_{1+\epsilon} R_h}(s_{k,h},a_{k,h})} \le (\beta_{R^\epsilon,k-1} + 6\cH^\epsilon \beta_{R,k-1} \kappa) \normbphiR{k}{h}.
$$
That is $\abr{\sbr{\hat{\bnu}_{1+\epsilon} R_h - \bnu_{1+\epsilon} R_h}(s_{k,h},a_{k,h})} \le W_{k,h}$.
\end{proof}

\subsection{Proof of Lemma~\ref{lem:mdp-ER}}\label{proof:mdp-ER}

\begin{proof}[Proof of Lemma~\ref{lem:mdp-ER}]
For each $h\in[H]$, we use Theorem~\ref{thm:heavy} with $L = 1, \delta = \delta / H, T = k, \cF_t = \cG_{i,h}, y_t = \tilde{R}_{i,h} := R_{i,h} \ind\cbr{\abr{\sbr{\hat{\bnu}_{1+\epsilon} R_h - \bnu_{1+\epsilon} R_h}(s_{i,h},a_{i,h})} \le W_{i,h}}$. Thus $b=1$ by definition. Denote $\tilde{\btheta}_{k,h}$ as the counterpart of $\btheta_{k,h}$ being the solution of adaptive Huber regression where $R_{i,h}$ is replaced by $\tilde{R}_{i,h}$ for all $i\le k$. Then by Theorem~\ref{thm:heavy}, with probability at least $1-3\delta$, we have
\begin{equation}\label{eq:proof-mdp-ER-btheta-error}
\nbr{\tilde{\btheta}_{k,h} - \btheta_h^*}_{\bH_{k,h}^{-1}} \le \beta_{R,k}
\end{equation}
for all $h\in[H]$ and $k \in [K]$.

Then, we continue the proof by induction on event $\cE_{R,k}$ over $k$. First, event $\cE_{R,0}$ holds trivially. Next, we suppose $\cE_{R,k-1}$ holds and will prove $\cE_{R,k}$ holds. Since $\cE_{R^\epsilon}\cap\cE_{R,k-1}$ holds, it follows that $\tilde{R}_{i,h} = R_{i,h}$ for all $h\in[H]$ and $i \le k$ by Lemma~\ref{lem:mdp-W} and the definition of $\tilde{R}_{i,h}$. And thus $\tilde{\btheta}_{k,h} = \btheta_{k,h}$. By \eqref{eq:proof-mdp-ER-btheta-error}, for all $h\in[H]$, we have
$$
\nbr{\btheta_{i,h} - \btheta_h^*}_{\bH_{i,h}^{-1}} \le \beta_{R,i},
$$
where
\begin{align*}
\beta_{R,i} &= 3 \sqrt{\lambda_R}B + 24 i^{\frac{1-\epsilon}{2(1+\epsilon)}} \sqrt{2\kappa} (\log 3K)^{\frac{1-\epsilon}{2(1+\epsilon)}} \rbr{\log \frac{2HK^2}{\delta}}^{\frac{\epsilon}{1+\epsilon}}\\
&= O\rbr{\sqrt{\lambda_R}B + \sqrt{d} i^{\frac{1-\epsilon}{2(1+\epsilon)}} \iota}.
\end{align*}
That is $\cE_{R,k}$ holds.

Finally, induction over all $k \in [K]$ completes the proof.
\end{proof}

\subsection{Proof of Lemma~\ref{lem:mdp-E0}}\label{proof:mdp-E0}

We will use the self-normalized bound with a covering argument in Lemma~\ref{lem:value-ci} for estimating next-state value function in the following proof frequently, which is the core technique used in \citet{jin2020provably,he2023nearly}.

\begin{proof}[Proof of Lemma~\ref{lem:mdp-E0}]
Define $\cV^{+}$ as the class of optimistic value functions mapping from $\cS$ to $\RR$ with the parametric form given in~\eqref{eq:function-pos} and $\cV^{-}$ the class of pessimistic value functions with the parametric form given in~\eqref{eq:function-pes}.
By Lemma~\ref{lem:rare-update} and Lemma~\ref{lem:covering},
\begin{equation}
\label{eq:covering}
\log \cN(\cV^\pm, \epsilon) \le \sbr{d \log \rbr{1+ \frac{4(B+L)}{\epsilon}} + d^2 \log \rbr{1 +\frac{ 8 d^{1/2} \beta_{R,K}^2}{\lambda_R \epsilon^2}} \rbr{1 +\frac{ 8 d^{1/2} \beta_V^2}{\lambda_V \epsilon^2}}} |\cK|,
\end{equation}
where $L = \cH \sqrt{\frac{dK}{\lambda_V}}$ and $|\cK| \le dH \log_2 \rbr{1 + \frac{K}{\lambda_R \nu_{\min}^2}} \rbr{1 + \frac{K}{\lambda_V \sigma_{\min}^2}}$.

\paragraph{The case where $f =  V_{h+1}^k$.}

One can find that $f \in \cV^+$.
To make use of Lemma~\ref{lem:value-ci}, we first specify the parameters defined therein.
We have $\|f\|_{\infty} \le  C_0 = \cH$ and $\epsilon_1 =  \frac{\lambda_V \cH \sqrt{d}}{K} \numin^2$. 
By~\eqref{eq:covering}, it follows that
\begin{align*}
& \log \cN(\cV^{+}, \epsilon_1) \\
\le{}& \sbr{ d \log \rbr{1+ \frac{4(B+L)K}{\lambda_V\cH \sqrt{d} \sigma_{\min}^2}} + d^2 \log \rbr{1 +\frac{8 \beta_{R,K}^2K^2}{\sqrt{d}\lambda_R\lambda_V^2 \cH^2\sigma_{\min}^4 }} \rbr{1 +\frac{8 \beta_V^2K^2}{\sqrt{d}\lambda_V^3 \cH^2\sigma_{\min}^4 }}} \\
&\qquad \cdot dH \log_2 \rbr{1 + \frac{K}{\lambda_R \nu_{\min}^2}} \rbr{ 1 + \frac{K}{\lambda_V\sigma_{\min}^2} }\\
\le{}& 6 d^3 H\iota_0^2.
\end{align*}
By the third condition of Lemma~\ref{lem:value-ci}, with probability at least $1-\frac{\delta}{H}$, $\nbr{\hat{\bw}_{k-1,h} - \bw_{h}\sbr{V^k_{h+1}}}_{\bSigma_{k,h}} \le \beta_0$ for all $k \in [K]$.
A union which finishes the proof.

\paragraph{The case where $f = \pesV_{h+1}^k$.}

The analysis on $\pesV_{h+1}^k$ is similar to (i).

\paragraph{The case where $f = \rbr{V_{h+1}^k}^2$.}

The analysis on $\rbr{V_{h+1}^k}^2$ is similar to (i) except for the following two changes. 
First, $C_0 = \cH^2$ and $\epsilon_1' =  \frac{\lambda_V \cH^2 \sqrt{d}}{K} \sigma_{\min}^2$.
Second, with $[\cV^+]^2 = \{ f^2: f \in \cV^+ \}$, we have $[V_{h+1}^k]^2 \in [\cV^+]^2$ and
\begin{align*}
    \log \cN([\cV^+]^2, \epsilon_1') 
    &\overset{(a)}{\le} \log \cN(\cV^+, \frac{\epsilon_1'}{2\cH}) 
    \le  \log \cN(\cV^+, \frac{\epsilon_1}{2}) \le  6 d^3 H\iota_0^2.
\end{align*}
Here $(a)$ uses the fact that the $\frac{\epsilon_1'}{2\cH}$-cover of $\cV^+$ is a $\epsilon_1$-cover of $[\cV^+]^2$.
\end{proof}

\subsection{Proof of Lemma~\ref{lem:mdp-opt}}\label{proof:mdp-opt}

\begin{proof}[Proof of Lemma~\ref{lem:mdp-opt}]
We prove the optimism inequality by induction. When $h = H+1$, results hold trivially since $V^k_{H+1} = V^*_{H+1} = 0$. We assume the statement is true for $h+1$, that is $V^k_{h+1}(\cdot) \ge V^*_{h+1}(\cdot)$. Next we prove the case of $h$.

For any $(s, a) \in \cS \times \cA$ and $k \in [K]$,
\begin{align*}
&\dotp{\bphi(\cdot,\cdot)}{\btheta_{k-1,h}+\hat{\bw}_{k-1,h}} + \beta_{R,k-1}\nbr{\bphi(\cdot,\cdot)}_{\bH_{k-1,h}^{-1}} + \beta_{V}\nbr{\bphi(\cdot,\cdot)}_{\bSigma_{k-1,h}^{-1}} - Q^*_h(\cdot,\cdot)\\
={}& \dotp{\bphi(\cdot,\cdot)}{\btheta_{k-1,h} - \btheta^*_h} + \beta_{R,k-1}\nbr{\bphi(\cdot,\cdot)}_{\bH_{k-1,h}^{-1}}\\
&\qquad + \dotp{\bphi(\cdot,\cdot)}{\hat{\bw}_{k-1,h} - \bw_h\sbr{V^*_{h+1}}} + \beta_{V}\nbr{\bphi(\cdot,\cdot)}_{\bSigma_{k-1,h}^{-1}}\\
\overset{(a)}{\ge}{}& \dotp{\bphi(\cdot,\cdot)}{\btheta_{k-1,h} - \btheta^*_h} + \beta_{R,k-1}\nbr{\bphi(\cdot,\cdot)}_{\bH_{k-1,h}^{-1}}\\
&\qquad + \dotp{\bphi(\cdot,\cdot)}{\hat{\bw}_{k-1,h} - \bw_h\sbr{V^k_{h+1}}} + \beta_{V}\nbr{\bphi(\cdot,\cdot)}_{\bSigma_{k-1,h}^{-1}}\\
\overset{(b)}{\ge}{}& \nbr{\bphi(\cdot,\cdot)}_{\bH_{k-1,h}^{-1}}\rbr{-\nbr{\btheta_{k-1,h} - \btheta^*_h}_{\bH_{k-1,h}} + \beta_{R,k-1}} \\
&\qquad + \nbr{\bphi(\cdot,\cdot)}_{\bSigma_{k-1,h}^{-1}}\rbr{-\nbr{\hat{\bw}_{k-1,h} - \bw_h\sbr{V^k_{h+1}}}_{\bSigma_{k-1,h}} + \beta_{V}}\\
\overset{(c)}{\ge}{}& 0,
\end{align*}
where $(a)$ is due to $\cE_{h}$ holds, $\cE_{h+1} \subset \cE_{h}$ and thus $\inner{\bphi(\cdot,\cdot)}{\bw_h\sbr{V^k_{h+1}} - \bw_h\sbr{V^*_{h+1}}} = [\PP_h (V^k_{h+1} - V^*_{h+1})](\cdot,\cdot) \ge 0$ by induction. $(b)$ is due to Cauchy-Schwartz inequality. And $(c)$ is due to $\cE_R\cap\cE_h$ holds.

We assume the sequence of updating episodes $1 = k_1 \le k_2 \le \dots \le k_{N_K} \le K$, such that the latest update episode before $k$ is $\kl = k_{N_k}$. Then for all $k\in[K]$, we have
\begin{align*}
& Q^k_{h+1}(\cdot,\cdot)\\
={}& \min_{i\le N_k} \cbr{\dotp{\bphi(\cdot,\cdot)}{\btheta_{k_i-1,h}+\hat{\bw}_{k_i-1,h}} + \beta_{R,k_i-1}\nbr{\bphi(\cdot,\cdot)}_{\bH_{k_i-1,h}^{-1}} + \beta_{V}\nbr{\bphi(\cdot,\cdot)}_{\bSigma_{k_i-1,h}^{-1}}, \cH}\\
\ge{}& Q^*_h(\cdot,\cdot),
\end{align*}
which implies the case of $h$ is true.

The proof for pessimism inequality is similar to the optimism.
\end{proof}

\subsection{Proof of Lemma~\ref{lem:mdp-E}}\label{proof:mdp-E}

\begin{proof}[Proof of Lemma~\ref{lem:mdp-E}]
First, $\abr{\sbr{\hat{\VV}_h V^k_{h+1} - \VV_h V^*_{h+1}}(s_{k,h},a_{k,h})} \le \cH^2$ since both $\sbr{\hat{\VV}_h V^k_{h+1}}(\cdot,\cdot)$ and $\sbr{\VV_h V^*_{h+1}}(\cdot,\cdot)$ are bounded in $[0,\cH^2]$. Then it follows that
\begin{align*}
&\abr{\sbr{\hat{\VV}_h V^k_{h+1} - \VV_h V^*_{h+1}}(s_{k,h},a_{k,h})}\\
\le{}& \abr{\sbr{\hat{\VV}_h V^k_{h+1} - \VV_h V^k_{h+1}}(s_{k,h},a_{k,h})} + \abr{\sbr{\VV_h V^k_{h+1} - \VV_h V^*_{h+1}}(s_{k,h},a_{k,h})}.
\end{align*}
We next bound the two terms in the RHS of the last inequality separately.

\paragraph{Bound the first term.}

It follows that
\begin{align*}
&\abr{\sbr{\hat{\VV}_h V^k_{h+1} - \VV_h V^k_{h+1}}(s_{k,h},a_{k,h})}\\
\le{}& \abr{\inner{\bphi_{k,h}}{\tbw_{k-1,h}}_{[0,\cH^2]} - \inner{\bphi_{k,h}}{\bw_h\sbr{\rbr{V^k_{h+1}}^2}}}\\
    &\qquad + \abr{\inner{\bphi_{k,h}}{\hbw_{k-1,h}}_{[0,\cH]}^2 - \inner{\bphi_{k,h}}{\bw_h\sbr{V^k_{h+1}}}^2}\\
\overset{(a)}{\le}{}& \abr{\inner{\bphi_{k,h}}{\tbw_{k-1,h} - \bw_h\sbr{\rbr{V^k_{h+1}}^2}}} + 2\cH\abr{\inner{\bphi_{k,h}}{\hbw_{k-1,h} - \bw_h\sbr{V^k_{h+1}}}}\\
\overset{(b)}{\le}{}& \normbphiV{k}{h} \nbr{\tbw_{k-1,h} - \bw_h\sbr{\rbr{V^k_{h+1}}^2}}_{\bSigma_{k-1,h}}\\
&\qquad + 2\cH \normbphiV{k}{h} \nbr{\hbw_{k-1,h} - \bw_h\sbr{V^k_{h+1}}}_{\bSigma_{k-1,h}}\\
\overset{(c)}{\le}{}& 3\cH\beta_0 \normbphiV{k}{h},
\end{align*}
where $(a)$ is due to both $\inner{\bphi_{k,h}}{\hbw_{k-1,h}}_{[0,\cH]}$ and $\inner{\bphi_{k,h}}{\bw_h\sbr{V^k_{h+1}}}$ are bounded in $[0,\cH]$. $(b)$ holds due to Cauchy-Schwartz inequality and $(c)$ is due to $\cE_0$ holds.

\paragraph{Bound the second term.}

We have
\begin{align*}
&\abr{\sbr{\VV_h V^k_{h+1} - \VV_h V^*_{h+1}}(s_{k,h},a_{k,h})}\\
\le{}& \abr{\sbr{\PP_h\rbr{V^k_{h+1}}^2-\rbr{V^*_{h+1}}^2}(s_{k,h},a_{k,h})} + \abr{\sbr{\PP_h V^k_{h+1}}^2(s_{k,h},a_{k,h}) - \sbr{\PP_h V^*_{h+1}}^2(s_{k,h},a_{k,h})}\\
={}& \abr{\sbr{\PP_h\rbr{V^k_{h+1}+V^*_{h+1}}\rbr{V^k_{h+1}-V^*_{h+1}}}(s_{k,h},a_{k,h})} \\
&\qquad + \abr{\sbr{\PP_h \rbr{V^k_{h+1}+V^*_{h+1}}}(s_{k,h},a_{k,h})\sbr{\PP_h \rbr{V^k_{h+1} - V^*_{h+1}}}(s_{k,h},a_{k,h})}\\
\overset{(a)}{\le}{}& 4\cH \sbr{\PP_h \rbr{V^k_{h+1} - V^*_{h+1}}}(s_{k,h},a_{k,h})\\
\overset{(b)}{\le}{}& 4\cH \sbr{\PP_h \rbr{V^k_{h+1} - \pesV^k_{h+1}}}(s_{k,h},a_{k,h})\\
\overset{(c)}{\le}{}& 4\cH \dotp{\bphi_{k,h}}{\hbw_{k-1,h}-\cbw_{k-1,h}} \\
&\qquad + 4\cH\rbr{\nbr{\hbw_{k-1,h} - \bw_{h}\sbr{V^k_{h+1}}}_{\bSigma_{k-1,h}} + \nbr{\cbw_{k-1,h} - \bw_{h}\sbr{V^k_{h+1}}}_{\bSigma_{k-1,h}}}\normbphiV{k}{h}\\
\overset{(d)}{\le}{}& 4\cH \dotp{\bphi_{k,h}}{\hbw_{k-1,h}-\cbw_{k-1,h}} + 8\cH\beta_0 \normbphiV{k}{h},
\end{align*}
where $(a)$ holds due to $V^k_{h+1}(\cdot),V^*_{h+1}(\cdot)$ are both bounded in $[0,\cH]$, $\cE_{R} \cap \cE_{h+1}$ holds and the optimism by Lemma~\ref{lem:mdp-opt}, $(b)$ holds due to the pessimism by Lemma~\ref{lem:mdp-opt}, $(c)$ holds due to Cauchy-Schwartz inequality and $(d)$ is due to $\cE_0$ holds.

Putting pieces together, we have
$$
\abr{\sbr{\hat{\VV}_h V^k_{h+1} - \VV_h V^*_{h+1}}(s_{k,h},a_{k,h})} \le 4\cH \dotp{\bphi_{k,h}}{\hbw_{k-1,h}-\cbw_{k-1,h}} + 11 \cH\beta_0 \normbphiV{k}{h}.
$$
\end{proof}

\subsection{Proof of Lemma~\ref{lem:mdp-D}}\label{proof:mdp-D}

\begin{proof}[Proof of Lemma~\ref{lem:mdp-D}]

First, $\sbr{\VV_h\rbr{V^k_{h+1}-V^*_{h+1}}}(s_{i,h},a_{i,h}) \le \cH^2$ since both $V^k_{h+1}(\cdot)$ and $V^*_{h+1}(\cdot)$ are bounded in $[0,\cH]$. Then, for any $i \le k$, we have
\begin{align*}
&\sbr{\VV_h\rbr{V^k_{h+1}-V^*_{h+1}}}(s_{i,h},a_{i,h})\\
\le{}& \sbr{\PP_h\rbr{V^k_{h+1}-V^*_{h+1}}^2}(s_{i,h},a_{i,h})\\
\overset{(a)}{\le}{}& 2\cH \sbr{\PP_h\rbr{V^k_{h+1}-V^*_{h+1}}}(s_{i,h},a_{i,h})\\
\overset{(b)}{\le}{}& 2\cH \sbr{\PP_h\rbr{V^k_{h+1}-\pesV^k_{h+1}}}(s_{i,h},a_{i,h})\\
\overset{(c)}{\le}{}& 2\cH \sbr{\PP_h\rbr{V^i_{h+1}-\pesV^i_{h+1}}}(s_{i,h},a_{i,h})\\
\overset{(d)}{\le}{}& 2\cH \dotp{\bphi_{i,h}}{\hbw_{i-1,h}-\cbw_{i-1,h}} + 4\cH\beta_0 \normbphiV{i}{h},
\end{align*}
where $(a)$ holds due to both $V^k_{h+1},V^*_{h+1}$ are bounded in $[0,\cH]$, $\cE_{R} \cap \cE_{h+1}$ holds and the optimism by Lemma~\ref{lem:mdp-opt}, $(b)$ holds due to the pessimism by Lemma~\ref{lem:mdp-opt}, $(c)$ holds due to $V^k_{h+1}$ is non-increasing and $\pesV^k_{h+1}$ is non-decreasing by definition and $(d)$ is due to Cauchy-Schwartz inequality and $\cE_0$ holds.

The proof of inequality for $\pesV^k_{h+1}$ is similar to the proof of $V^k_{h+1}$ above.
\end{proof}

\subsection{Proof of Lemma~\ref{lem:mdp-EV}}\label{proof:mdp-EV}

\begin{proof}[Proof of Lemma~\ref{lem:mdp-EV}]
We prove Lemma~\ref{lem:mdp-EV} by induction. First, when $h=H$, event $\cE_{H}$ holds since $V^k_{H+1}(\cdot) = V^*_{H+1}(\cdot) = \pesV^k_{H+1}(\cdot)$ = 0 for all $k\in[H]$. Then, we prove with probability at least $1-\frac{2\delta}{H}$, $\cE_{h}$ holds on event $\cE_{h+1}$. By induction over $h$, with probability at least $1-2\delta$, on event $\cE_0\cap\cE_R$, event $\cE_V = \cE_1$ holds.

Note $\hbw_{k-1,h} = \hbw_{k-1,h}\sbr{V^k_{h+1}} = \hbw_{k-1,h}\sbr{V^*_{h+1}} + \hbw_{k-1,h}\sbr{V^k_{h+1} - V^*_{h+1}}$ and $\bw_{h}\sbr{V^k_{h+1}} = \bw_{h}\sbr{V^*_{h+1}} + \bw_{h}\sbr{V^k_{h+1} - V^*_{h+1}}$. It follows that
\begin{align*}
&\nbr{\hbw_{k-1,h} - \bw_h\sbr{V^k_{h+1}}}_{\bSigma_{k-1,h}}\\
\le{}& \nbr{\hbw_{k-1,h}\sbr{V^*_{h+1}} - \bw_h\sbr{V^*_{h+1}}}_{\bSigma_{k-1,h}} + \nbr{\hbw_{k-1,h}\sbr{V^k_{h+1} - V^*_{h+1}} - \bw_h\sbr{V^k_{h+1} - V^*_{h+1}}}_{\bSigma_{k-1,h}}.
\end{align*}
Next, we bound the two terms in the RHS by Lemma~\ref{lem:value-ci} separately.

\paragraph{Bound the first term.}

The first condition of Lemma~\ref{lem:value-ci} is satisfied since $V^*_{h+1}$ is a deterministic function. We have $C_0=\cH$ and $\cA_{k,h} = \cbr{\sigma_{k,h}^2 \ge \sbr{\VV_h V^*_h}(s_{k,h},a_{k,h})}$ is $\cF_{k,h}$-measurable. Since $\cE_0\cap\cE_R\cap\cE_{h+1}$ holds, for all $k\in[K]$, $\abr{\sbr{\hat{\VV}_h V^k_{h+1} - \VV_h V^*_{h+1}}(s_{k,h},a_{k,h})} \le E_{k,h}$ by Lemma~\ref{lem:mdp-E}. Thus, $\sigma_{k,h}^2 \ge \sbr{\hat{\VV}_h V^k_{h+1}}(s_{k,h},a_{k,h}) + E_{k,h} \ge \sbr{\VV_h V^*_h}(s_{k,h},a_{k,h})$ for all $k\in[K]$, which implies $\bigcap_{k\in[K]}\cA_{k,h}$ holds and thus $C_\sigma = 1$. By Lemma~\ref{lem:value-ci}, with probability at least $1-\delta/H$, for all $k\in[K]$,
$$
\nbr{\hbw_{k-1,h}\sbr{V^*_{h+1}} - \bw_h\sbr{V^*_{h+1}}}_{\bSigma_{k-1,h}} \le \beta_1,
$$
where
\begin{align*}
\beta_1
&= \sqrt{d \lambda_V} \cH + 8\sqrt{d \log \rbr{ 1 + \frac{K}{\sigma_{\min}^2d\lambda_V} } \log \frac{4H K^2}{\delta} } + \frac{4}{ d^{2.5}H}\log\frac{4HK^2}{\delta}\\
&\le \sqrt{d\lambda_V}\cH + 8 \sqrt{d}\iota_1 + 4\iota_1.
\end{align*}

\paragraph{Bound the second term.}

The second condition of Lemma~\ref{lem:value-ci} is satisfied since $V^k_{h+1} - V^*_{h+1}$ are random functions. We have $C_0 = \cH$ and $\cA_{k,h} = \cbr{\sigma_{k,h}^2 \ge d^3 H \sbr{\VV_h \rbr{V^k_h - V^*_h}}(s_{k,h},a_{k,h})}$ is $\cF_{k,h}$-measurable. Since $\cE_0\cap\cE_R\cap\cE_{h+1}$ holds, for all $k\in[K]$, $\sbr{\VV_h \rbr{V^k_h - V^*_h}}(s_{k,h},a_{k,h}) \le D_{k,h}$ by Lemma~\ref{lem:mdp-D}. Thus, by definition of $\sigma_{k,h}$, we have $\bigcap_{k\in[K]}\cA_{k,h}$ holds and thus $C_\sigma = \frac{1}{\sqrt{d^3H}}$. And the log-covering number of the function class that covers $V^k_{h+1} - V^*_{h+1}$ can be bounded by Lemma~\ref{lem:covering} as $\log N_0 = |\cN(\cV^+, \epsilon_0)| \le 6d^3H\iota_1^2$ with $\cV^+$ defined in \eqref{eq:function-pos} and $\epsilon_0 = \min\cbr{ \frac{\sigma_{\min}}{\sqrt{d^3H}}, \frac{\lambda_V \cH \sqrt{d}}{K}\sigma_{\min}^2 }$. By Lemma~\ref{lem:value-ci}, with probability at least $1-\delta/H$, for all $k\in[K]$,
$$
\nbr{\hbw_{k-1,h}\sbr{V^k_{h+1} - V^*_{h+1}} - \bw_h\sbr{V^k_{h+1} - V^*_{h+1}}}_{\bSigma_{k-1,h}} \le \beta_2,
$$
where
\begin{align*}
\beta_2 &= 2\sqrt{d \lambda_V} \cH + \frac{32}{\sqrt{d^3H}}\sqrt{d\log \rbr{1 + \frac{K}{\sigma_{\min}^2d\lambda_V}} \log \frac{4N_0H K^2}{\delta} } + \frac{4}{d^{2.5}H}\log\frac{4N_0HK^2}{\delta}\\
&= O(\sqrt{d \lambda_V} \cH + \sqrt{d}\iota_1^2).
\end{align*}

The proof for the pessimism is similar to that of the optimism.

Finally, putting pieces together and we have $\beta_V = \beta_1 + \beta_2 = O(\sqrt{d \lambda_V} \cH + \sqrt{d}\iota_1^2)$.

\end{proof}

\subsection{Proof of Lemma~\ref{lem:mdp-subopt-gap}}\label{proof:mdp-subopt-gap}

\begin{proof}[Proof of Lemma~\ref{lem:mdp-subopt-gap}]
For any $k\in[K]$, recall $\kl$ is the latest update episode before episode $k$ satisfying $\kl \le k$ and $V^k_h(\cdot) = V^{\kl}_h(\cdot), \forall h\in[H]$.
By Lemma~\ref{lem:matrix-ratio}, due to $\bH_{k-1,h} \succeq \bH_{\kl-1,h}$ and $\det(\bH_{k-1,h}) \le 2 \det(\bH_{\kl-1,h})$ by the updating rule, it follows that for any $\bx \in \RR^d$,
\begin{equation}\label{eq:bonus-ratio}
\|\bx\|_{ \bH_{\kl-1,h}^{-1}} \le 2 \|\bx\|_{ \bH_{k-1,h}^{-1}}.
\end{equation}

\paragraph{Telescoping on $V_{h}^k(s_{k,h}) - V_h^{\pi^k}(s_{k,h})$.}

By definition, $Q_h^k(\cdot, \cdot) \le \dotp{\bphi(\cdot,\cdot)}{\btheta_{\kl-1,h} + \hbw_{\kl-1,h}} + \beta_{R,\kl-1} \nbr{\bphi_{k,h}}_{\bH_{\kl-1,h}^{-1}} + \beta_V \nbr{\bphi_{k,h}}_{\bSigma_{\kl-1,h}^{-1}}$ and $Q_h^{\pi^k}(\cdot,\cdot) = \inner{\bphi(\cdot,\cdot)}{\btheta^*_h + \bw_h\sbr{V^{\pi^k}_{h+1}}}$.
Since $a_{k,h} = \pi_h^k(s_{k,h}) = \argmax_{a \in \cA} Q_h^k(s_{k,h}, a)$, we have
\begin{align*}
&V_{h}^k(s_{k,h}) - V_h^{\pi^k}(s_{k,h})\\
={}& Q_h^k(s_{k,h}, a_{k,h}) - Q_h^{\pi^k}(s_{k,h}, a_{k,h})  \\
\le{}& \inner{\bphi_{k,h}}{\rbr{\btheta_{\kl-1,h} - \btheta_h^*} + \rbr{\hbw_{\kl-1,h} - \bw_h\sbr{V^{\pi^k}_{h+1}}}}\\
&\qquad + \beta_{R,\kl-1} \nbr{\bphi_{k,h}}_{\bH_{\kl-1,h}^{-1}} + \beta_V \nbr{\bphi_{k,h}}_{\bSigma_{\kl-1,h}^{-1}}\\
\overset{(a)}{\le}{}& \inner{\bphi_{k,h}}{\rbr{\btheta_{\kl-1,h} - \btheta_h^*} + \rbr{\hbw_{\kl-1,h} - \bw_h\sbr{V^{\kl}_{h+1}}} + \rbr{\bw_h\sbr{V^{k}_{h+1} - V^{\pi^k}_{h+1}}}} \\
&\qquad + 2\beta_{R,k-1} \normbphiR{k}{h} + 2\beta_V \normbphiV{k}{h}\\
\overset{(b)}{\le}{}& \sbr{\PP_h \rbr{V^{k}_{h+1} - V^{\pi^k}_{h+1}}}(s_{k,h},a_{k,h}) + 4\beta_{R,k-1} \normbphiR{k}{h} + 4\beta_V \normbphiV{k}{h}
\end{align*}
Here $(a)$ uses~\eqref{eq:bonus-ratio}, $(b)$ uses 
\begin{align*}
&\inner{\bphi_{k,h}}{\rbr{\btheta_{\kl-1,h} - \btheta_h^*} + \rbr{\hbw_{\kl-1,h} - \bw_h\sbr{V^{\kl}_{h+1}}}}\\
\le{}& \nbr{\bphi_{k,h}}_{\bH_{\kl-1,h}^{-1}} \nbr{\btheta_{\kl-1,h} - \btheta_h^*}_{\bH_{\kl-1,h}}\\
&\qquad + \nbr{\bphi_{k,h}}_{\bSigma_{\kl-1,h}^{-1}} \nbr{\hbw_{\kl-1,h} - \bw_h\sbr{V^{\kl}_{h+1}}}_{\bSigma_{\kl-1,h}}\\
\le{}& \beta_{R,\kl-1} \nbr{\bphi_{k,h}}_{\bH_{\kl-1,h}^{-1}} + \beta_V \nbr{\bphi_{k,h}}_{\bSigma_{\kl-1,h}^{-1}}\\
\le{}& 2\beta_{R,k-1} \normbphiR{k}{h} + 2\beta_V \normbphiV{k}{h}
\end{align*}
on $\cE_{R}\cap\cE_{V}$ and
$$
\inner{\bphi_{k,h}}{\bw_h\sbr{V^{k}_{h+1} - V^{\pi^k}_{h+1}}} = \sbr{\PP_h \rbr{V^{k}_{h+1} - V^{\pi^k}_{h+1}}}(s_{k,h},a_{k,h}).
$$
We define $X_{k,h} := \sbr{\PP_h \rbr{V^{k}_{h+1} - V^{\pi^k}_{h+1}}}(s_{k,h},a_{k,h}) - \sbr{V^{k}_{h+1}(s_{k,h+1}) - V^{\pi^k}_{h+1}(s_{k,h+1})}$. Then it follows that
\begin{align*}
& V_{h}^k(s_{k,h}) - V_h^{\pi^k}(s_{k,h})\\
={}& V^{k}_{h+1}(s_{k,h+1}) - V^{\pi^k}_{h+1}(s_{k,h+1}) + X_{k,h} + 4\beta_{R,k-1} \normbphiR{k}{h} + 4\beta_V \normbphiV{k}{h}.
\end{align*}
Summing the inequality above over $h'$ such that $h \le h' \le H$ and we have
\begin{equation}\label{eq:final-iter}
V_{h}^k(s_{k,h}) - V_h^{\pi^k}(s_{k,h}) \le \sum_{h'=h}^H \sbr{X_{k,h'} + 4\beta_{R,k-1} \normbphiR{k}{h'} + 4\beta_V \normbphiV{k}{h'}}.
\end{equation}
Therefore, setting $h=1$ and summing~\eqref{eq:final-iter} over $k \in [K]$, we have
\begin{equation}\label{eq:final0}
\sum_{k=1}^K \sbr{V_{1}^k(s_{k,h}) - V_1^{\pi^k}(s_{k,h})}
\le \sum_{k=1}^K \sum_{h=1}^H \sbr{X_{k,h} + 4\beta_{R,k-1} \normbphiR{k}{h} + 4\beta_V \normbphiV{k}{h}}.
\end{equation}

\paragraph{Bound $\sum_{k=1}^K \sum_{h=1}^H X_{k,h}$.}

Notice $X_{k,h}$ is $\cF_{k,h+1}$-measurable with $\EE[X_{k,h}|\cF_{k,h}] = 0$, $|X_{k,h}| \le 2\cH$ and
$$
\EE[X_{k,h}^2|\cF_{k,h}]  \le  \EE\sbr{\sbr{V_{h+1}^k(s_{k,h+1}) - V_{h+1}^{\pi^k}(s_{k,h+1})}^2 \bigg|\cF_{k,h}} \overset{(a)}{\le}  \cH \sbr{\PP_h(V_{h+1}^k - V_{h+1}^{\pi^k})}(s_{k,h}, a_{k,h}),
$$
where $(a)$ holds due to $\abr{V_{h+1}^k(\cdot) - V_{h+1}^{\pi^k}(\cdot)} \le \cH$ and optimism in Lemma~\ref{lem:mdp-opt}
By the variance-aware Freedman inequality in Lemma~\ref{lem:bern}, with probability at least $1-\frac{\delta}{2}$, it follows that
\begin{equation}\label{eq:final1}
\abr{\sum_{k=1}^K \sum_{h=1}^H X_{k,h}} \le 3 \sqrt{\iota} \sqrt{\cH \cdot \sum_{k=1}^K \sum_{h=1}^H \sbr{\PP_h(V_{h+1}^k - V_{h+1}^{\pi^k})}(s_{k,h}, a_{k,h})} + 10\cH\iota
\end{equation}
where $\iota = \log\frac{4\ceil{\log_2 HK}}{\delta}$.

\paragraph{Bound $\sum_{k=1}^K \sum_{h=1}^H \sbr{\PP_h(V_{h+1}^k - V_{h+1}^{\pi^k})}(s_{k,h}, a_{k,h})$.}

Notice
\begin{align*}
&\sum_{k=1}^K \sum_{h=1}^H \sbr{\PP_h(V_{h+1}^k - V_{h+1}^{\pi^k})}(s_{k,h}, a_{k,h})\\
={}& \sum_{k=1}^K \sum_{h=2}^H \sbr{V_{h}^k(s_{k,h}) - V_{h}^{\pi^k}(s_{k,h})} + \sum_{k=1}^K \sum_{h=1}^H  X_{k,h} \\
\overset{(a)}{\le}{}& \sum_{k=1}^K \sum_{h=2}^H  \sum_{h'=h}^H \sbr{X_{k,h'} + 4\beta_{R,k-1} \normbphiR{k}{h'} + 4\beta_V \normbphiV{k}{h'}} + \sum_{k=1}^K \sum_{h=1}^H  X_{k,h}\\
={}& \sum_{k=1}^K \sum_{h=2}^H (H-h+1)\sbr{X_{k,h} + 4\beta_{R,k-1} \normbphiR{k}{h} + 4\beta_V \normbphiV{k}{h}} + \sum_{k=1}^K \sum_{h=1}^H  X_{k,h}\\
\overset{(b)}{\le}{}& \sum_{k=1}^K \sum_{h=1}^H  X_{k,h}  b_{h} + 4H\beta_R \sum_{k=1}^K \sum_{h=2}^H \normbphiR{k}{h} + 4H\beta_V \sum_{k=1}^K \sum_{h=2}^H \normbphiV{k}{h}
\end{align*}
where $(a)$ uses~\eqref{eq:final-iter} and $(b)$ uses the notation
$$
b_h = \begin{cases}
1 & \text{if } h=1,\\
H-h+2 & \text{if } 2 \le h \le H.
\end{cases}
$$
Since $|b_h| \le H$, by the variance-aware Freedman inequality in Lemma~\ref{lem:bern}, with probability at least $1-\frac{\delta}{2}$, we have
\[
\left| \sum_{k=1}^K \sum_{h=1}^H  X_{k,h} b_h \right|
\le 3H \sqrt{\iota} \sqrt{\cH \cdot \sum_{k=1}^K \sum_{h=1}^H \sbr{\PP_h(V_{h+1}^k - V_{h+1}^{\pi^k})}(s_{k,h}, a_{k,h})} + 10H\cH \iota.
\]
Thus,
\begin{equation}\label{eq:final-expect-sum}
\begin{aligned}
&\sum_{k=1}^K \sum_{h=1}^H \sbr{\PP_h(V_{h+1}^k - V_{h+1}^{\pi^k})}(s_{k,h}, a_{k,h}) \\
\le{}& 3H \sqrt{\iota} \sqrt{\cH \cdot \sum_{k=1}^K \sum_{h=1}^H \sbr{\PP_h(V_{h+1}^k - V_{h+1}^{\pi^k})}(s_{k,h}, a_{k,h})} + 10H\cH \iota \\
&\qquad + 4H\beta_R \sum_{k=1}^K \sum_{h=2}^H \normbphiR{k}{h} + 4H\beta_V \sum_{k=1}^K \sum_{h=2}^H \normbphiV{k}{h}\\
\overset{(a)}{\le}{}& 8H\beta_R \sum_{k=1}^K \sum_{h=1}^H \normbphiR{k}{h} + 8H\beta_V \sum_{k=1}^K \sum_{h=1}^H \normbphiV{k}{h} + 29 H^2 \cH \iota,
\end{aligned}
\end{equation}
where $(a)$ holds since $x \le a\sqrt{x} + b$ implies $x \le a^2+2b$.

\paragraph{Putting pieces together.}

Finally, plug-in \eqref{eq:final0}, and we have
\begin{align*}
&\sum_{k =1}^K\sbr{V_{1}^k(s_{k,1}) - V_1^{\pi^k}(s_{k,1})}\\
\le{}& \sum_{k=1}^K \sum_{h=1}^H \sbr{X_{k,h} + 4\beta_{R,k-1} \normbphiR{k}{h} + 4\beta_V \normbphiV{k}{h}}\\
\overset{\eqref{eq:final1}}{\le}{}& 3 \sqrt{\iota} \sqrt{\cH \cdot \sum_{k=1}^K \sum_{h=1}^H \sbr{\PP_h(V_{h+1}^k - V_{h+1}^{\pi^k})}(s_{k,h}, a_{k,h})} + 10\cH\iota\\
    &\qquad + 4\beta_R\sum_{k=1}^K\sum_{h=1}^H\normbphiR{k}{h} + 4\beta_V\sum_{k=1}^K\sum_{h=1}^H\normbphiV{k}{h}\\
\overset{\eqref{eq:final-expect-sum}}{\le}{}& 3 \sqrt{\iota}  \sqrt{\cH \cdot \sbr{8H\beta_R \sum_{k=1}^K \sum_{h=1}^H \normbphiR{k}{h} + 8H\beta_V \sum_{k=1}^K \sum_{h=1}^H \normbphiV{k}{h} + 29 H^2 \cH \iota}}\\
    &\qquad + 4\beta_R\sum_{k=1}^K\sum_{h=1}^H\normbphiR{k}{h} + 4\beta_V\sum_{k=1}^K\sum_{h=1}^H\normbphiV{k}{h} + 10\cH\iota\\
\le{}& 6\beta_R\sum_{k=1}^K\sum_{h=1}^H\normbphiR{k}{h} + 6\beta_V\sum_{k=1}^K\sum_{h=1}^H\normbphiV{k}{h} + 36 H \cH \iota
\end{align*}
where the last inequality uses $\sqrt{a+b} \le \sqrt{a} + \sqrt{b}$ and $2\sqrt{ab} \le a+ b$ for non-negative numbers $a, b \ge 0$.
\end{proof}

\subsection{Proof of Lemma~\ref{lem:mdp-opt-pes-gap}}\label{proof:mdp-opt-pes-gap}

\begin{proof}[Proof of Lemma~\ref{lem:mdp-opt-pes-gap}]
For any $k\in[K]$, recall $\kl$ is the latest update episode before episode $k$ satisfying $\kl \le k$ and $V^k_h(\cdot) = V^{\kl}_h(\cdot), \pesV^k_h(\cdot) = \pesV^{\kl}_h(\cdot), \forall h\in[H]$.

By definition, $Q_h^k(\cdot, \cdot) \le \dotp{\bphi(\cdot,\cdot)}{\btheta_{\kl-1,h} + \hbw_{\kl-1,h}} + \beta_{R,\kl-1} \nbr{\bphi_{k,h}}_{\bH_{\kl-1,h}^{-1}} + \beta_V \nbr{\bphi_{k,h}}_{\bSigma_{\kl-1,h}^{-1}}$ and $\pesQ_h^k(\cdot, \cdot) \ge \dotp{\bphi(\cdot,\cdot)}{\btheta_{\kl-1,h} + \cbw_{\kl-1,h}} - \beta_{R,\kl-1} \nbr{\bphi_{k,h}}_{\bH_{\kl-1,h}^{-1}} - \beta_V \nbr{\bphi_{k,h}}_{\bSigma_{\kl-1,h}^{-1}}$.
Since $a_{k,h} = \pi_h^k(s_{k,h}) = \argmax_{a \in \cA} Q_h^k(s_{k,h}, a)$, we have
\begin{align*}
&V_{h}^k(s_{k,h}) - \pesV_h^{k}(s_{k,h})\\
\le{}& Q_h^k(s_{k,h}, a_{k,h}) - \pesQ_h^k(s_{k,h}, a_{k,h})  \\
\le{}& \inner{\bphi_{k,h}}{\hbw_{\kl-1,h} - \cbw_{\kl-1,h}} + 2\beta_{R,\kl-1} \nbr{\bphi_{k,h}}_{\bH_{\kl-1,h}^{-1}} + 2\beta_V \nbr{\bphi_{k,h}}_{\bSigma_{\kl-1,h}^{-1}}\\
\overset{(a)}{\le}{}& \inner{\bphi_{k,h}}{\rbr{\hbw_{\kl-1,h} - \bw_h\sbr{V^{\kl}_{h+1}}} - \rbr{\cbw_{\kl-1,h} - \bw_h\sbr{\pesV^{\kl}_{h+1}}}}\\
    &\qquad + \inner{\bphi_{k,h}}{\bw_h\sbr{V^{k}_{h+1} - \pesV^{k}_{h+1}}} + 4\beta_{R,k-1} \normbphiR{k}{h} + 4\beta_V \normbphiV{k}{h}\\
\overset{(b)}{\le}{}& \sbr{\PP_h \rbr{V^{k}_{h+1} - \pesV^{k}_{h+1}}}(s_{k,h},a_{k,h}) + 4\beta_{R,k-1} \normbphiR{k}{h} + 8\beta_V \normbphiV{k}{h}
\end{align*}
Here $(a)$ uses~\eqref{eq:bonus-ratio}, $(b)$ uses 
\begin{align*}
&\inner{\bphi_{k,h}}{\rbr{\hbw_{\kl-1,h} - \bw_h\sbr{V^{\kl}_{h+1}}} - \rbr{\cbw_{\kl-1,h} - \bw_h\sbr{\pesV^{\kl}_{h+1}}}}\\
\le{}& \nbr{\bphi_{k,h}}_{\bSigma_{\kl-1,h}^{-1}} \nbr{\hbw_{\kl-1,h} - \bw_h\sbr{V^{\kl}_{h+1}}}_{\bSigma_{\kl-1,h}}\\
    &\qquad + \nbr{\bphi_{k,h}}_{\bSigma_{\kl-1,h}^{-1}} \nbr{\cbw_{\kl-1,h} - \bw_h\sbr{\pesV^{\kl}_{h+1}}}_{\bSigma_{\kl-1,h}}\\
\le{}& 2\beta_V \nbr{\bphi_{k,h}}_{\bSigma_{\kl-1,h}^{-1}}\\
\le{}& 4\beta_V \normbphiV{k}{h}
\end{align*}
on $\cE_{R}\cap\cE_{V}$ and
$$
\inner{\bphi_{k,h}}{\bw_h\sbr{V^{k}_{h+1} - \pesV^{k}_{h+1}}} = \sbr{\PP_h \rbr{V^{k}_{h+1} - \pesV^{k}_{h+1}}}(s_{k,h},a_{k,h}).
$$

The rest of the proof is nearly the same as that in Lemma~\ref{lem:mdp-subopt-gap} except for the constants, and we omit it to avoid repetition.
\end{proof}

\subsection{Proof of Lemma~\ref{lem:mdp-sumR}}\label{proof:mdp-sumR}

\begin{proof}[Proof of Lemma~\ref{lem:mdp-sumR}]
Notice
$$
\sum_{k=1}^K\sum_{h=1}^H \normbphiR{k}{h} = \sum_{k=1}^K\sum_{h=1}^H \nu_{k,h} w_{k,h}
$$
where $w_{k,h} = \nbr{\bphi_{k,h}/\nu_{k,h}}_{\bH_{k-1,h}}$.
And we have $w_{k,h} \le c_0 \le 1$ by definition of $\nu_{k,h},c_0$ and thus for all $h\in[H]$,
\begin{equation}\label{eq:proof-mdp-sumR-w-sum}
\sum_{k=1}^K w_{k,h}^2 = \min\cbr{1,w_{k,h}^2} \le 2 \kappa
\end{equation}
by Lemma~\ref{lem:w-sum}, where $\kappa$ is defined in Algorithm~\ref{algo:mdp}.

Next we bound the sum of bonus $\sum_{k=1}^K\sum_{h=1}^H \nu_{k,h} w_{k,h}$ separately by the value of $\nu_{k,h}$. Recall $\nu_{k,h}$ is defined in \eqref{eq:mdp-nu}. We decompose $[K]\times[H]$ as the union of three disjoint sets $\cJ_1, \cJ_2, \cJ_3$ where
\begin{align*}
\cJ_1 &= \cbr{(k,h)\in[K]\times[H] | \nu_{k,h}\in\cbr{\hat{\nu}_{k,h},\numin}},\\
\cJ_2 &= \cbr{(k,h)\in[K]\times[H] \bigg| \nu_{k,h} = \frac{\normbphiR{k}{h}}{c_0}},\\
\cJ_3 &= \cbr{(k,h)\in[K]\times[H] \Bigg| \nu_{k,h} = \frac{\sqrt{\max\cbr{B,W}}}{c_1^{\frac{1}{4}} (2\kappa)^{\frac{1}{4}}} \normbphiR{k}{h}^\frac{1}{2}}.
\end{align*}

For the summation over $\cJ_1$, we have
\begin{equation}\label{eq:proof-mdp-sumR-J1}
\begin{aligned}
\sum_{(k,h)\in\cJ_1} \nu_{k,h} w_{k,h} &= \sum_{(k,h)\in\cJ_1} \max\cbr{\hat{\nu}_{k,h}, \numin} w_{k,h}\\
&\overset{(a)}{\le} \sqrt{\sum_{k=1}^K\sum_{h=1}^H (\hat{\nu}_{k,h}^2+\numin^2)} \sqrt{\sum_{k=1}^K\sum_{h=1}^H w_{k,h}^2}\\
&\overset{(b)}{\le} \sqrt{2H\kappa} \sqrt{\sum_{k=1}^K\sum_{h=1}^H \hat{\nu}_{k,h}^2 + HK\numin^2},
\end{aligned}
\end{equation}
where $(a)$ holds due to Cauchy-Schwartz inequality and $(b)$ holds due to \eqref{eq:proof-mdp-sumR-w-sum}.

We provide a upper bound for $\sum_{k=1}^K\sum_{h=1}^H \hat{\nu}_{k,h}^2$ in Lemma~\ref{lem:mdp-hatnu-sum}.
\begin{lemma}\label{lem:mdp-hatnu-sum}
With probability at least $1-\delta$, on event $\cE_{R^\epsilon}\cap\cE_R$, we have
\begin{align*}
\sum_{k=1}^K\sum_{h=1}^H \hat{\nu}_{k,h}^2 = O\Bigg(&\cU^*K + \rbr{\beta_{R^\epsilon} + \cH^\epsilon \beta_{R} \kappa}^\frac{2}{1+\epsilon} \rbr{\sum_{k=1}^K\sum_{h=1}^H \normbphiR{k}{h}}^\frac{2}{1+\epsilon}\\
    & + H\nu_R^2\log\frac{2\ceil{\log_2K}}{\delta}\Bigg).
\end{align*}
\end{lemma}
\begin{proof}
    See Appendix~\ref{proof:mdp-hatnu-sum} for a detailed proof.
\end{proof}

Next, for the summation over $\cJ_2$, we have $\nbr{\frac{\bphi_{k,h}}{\nu_{k,h}}}_{\bH_{k-1,h}^{-1}} = c_0$. Then
\begin{equation}\label{eq:proof-mdp-sumR-J2}
\sum_{(k,h)\in\cJ_2} \nu_{k,h} w_{k,h} = \frac{1}{c_0^2} \sum_{(k,h)\in\cJ_2} \nbr{\bphi_{k,h}}_{\bH_{k-1,h}^{-1}} w_{k,h}^2
\overset{(a)}{\le} \frac{1}{c_0^2\sqrt{\lambda_R}} \sum_{k=1}^K\sum_{h=1}^H w_{k,h}^2
\overset{(b)}{\le} \frac{2H\kappa}{c_0^2\sqrt{\lambda_R}},
\end{equation}
where $(a)$ holds due to $\bH_{k-1,h} \succeq \lambda_R \bI$, thus $\nbr{\bphi_{k,h}}_{\bH_{k-1,h}^{-1}} \le \frac{\nbr{\bphi_{k,h}}}{\sqrt{\lambda_R}} \le \frac{1}{\sqrt{\lambda_R}}$. And $(b)$ holds due to \eqref{eq:proof-mdp-sumR-w-sum}.

Then, for the summation over $\cJ_3$, we have $\nu_{k,h} = \frac{\max\cbr{B,W}}{\sqrt{2 c_1 \kappa}}w_{k,h}$. Therefore
\begin{equation}\label{eq:proof-mdp-sumR-J3}
\sum_{(k,h)\in\cJ_3} \nu_{k,h} w_{k,h} = \frac{\max\cbr{B,W}}{\sqrt{2c_1\kappa}} \sum_{(k,h)\in\cJ_3} w_{k,h}^2 \le \frac{\max\cbr{B,W} H \sqrt{2\kappa}}{\sqrt{c_1}},
\end{equation}
where the inequality holds due to \eqref{eq:proof-mdp-sumR-w-sum}.

Finally, combining \eqref{eq:proof-mdp-sumR-J1}, \eqref{eq:proof-mdp-sumR-J2}, \eqref{eq:proof-mdp-sumR-J3} and Lemma~\ref{lem:mdp-hatnu-sum}, with probability at least $1-\delta$, we have
\begin{align*}
\sum_{k=1}^K\sum_{h=1}^H \normbphiR{k}{h} &= O\rbr{\sqrt{H\kappa}\rbr{\beta_{R^\epsilon} + \cH^\epsilon \beta_{R} \kappa}^\frac{1}{1+\epsilon} \rbr{\sum_{k=1}^K\sum_{h=1}^H \normbphiR{k}{h}}^\frac{1}{1+\epsilon} + C} 
\end{align*}
where
$$
C = \sqrt{H\kappa} \sqrt{\cU^*K + HK\numin^2 + H\nu_R^2\log\frac{2\ceil{\log_2K}}{\delta}} + \frac{H\kappa}{c_0^2\sqrt{\lambda_R}} + \frac{\max\cbr{B,W} H \sqrt{\kappa}}{\sqrt{c_1}}.
$$
Using the inequality that $x\le ax^\frac{1}{1+\epsilon} + b$ implies $x \le a^\frac{1+\epsilon}{\epsilon} + \frac{1+\epsilon}{\epsilon}b$, we have
$$
\sum_{k=1}^K\sum_{h=1}^H \normbphiR{k}{h} = O\rbr{(H\kappa)^\frac{1+\epsilon}{2\epsilon}\rbr{\beta_{R^\epsilon} + \cH^\epsilon \beta_{R} \kappa}^\frac{1}{\epsilon} + C}.
$$
Next, we simplify the expression above by hiding logarithmic terms. Notice $\kappa = \tilde{O}(d)$. Setting $\lambda_R=\frac{d}{\max\cbr{B^2,W^2}}$, we have $\beta_{R^\epsilon} = \tilde{O}\rbr{\sqrt{d}\frac{\nu_{R^\epsilon}}{\numin}K^\frac{1-\epsilon'}{2(1+\epsilon')}}$ and $\beta_R = \tilde{O}\rbr{\sqrt{d}K^\frac{1-\epsilon}{2(1+\epsilon)}}$. Therefore,
\begin{align*}
\sum_{k=1}^K\sum_{h=1}^H \normbphiR{k}{h} &= O(C) + \tilde{O}\rbr{\rbr{d^\frac{1+\epsilon}{2} H^\frac{1+\epsilon}{2} \beta_{R^\epsilon}}^\frac{1}{\epsilon}
+ \rbr{d^\frac{3+\epsilon}{2} H^\frac{1+\epsilon}{2} \cH^\epsilon \beta_R}^\frac{1}{\epsilon}}\\
&= O(C) + \tilde{O}\Bigg(\Bigg(\frac{\nu_{R^\epsilon} d^\frac{2+\epsilon}{2} H^\frac{1+\epsilon}{2} K^\frac{1-\epsilon'}{2(1+\epsilon')}}{\numin}\Bigg)^\frac{1}{\epsilon}
+ \rbr{d^\frac{4+\epsilon}{2} H^\frac{1+\epsilon}{2} \cH^\epsilon K^\frac{1-\epsilon}{2(1+\epsilon)}}^\frac{1}{\epsilon}\Bigg).
\end{align*}
We then simplify $C$ as
$$
C = \tilde{O}\rbr{\sqrt{dH\cU^*K} + \sqrt{dH^2K\numin^2} + \rbr{\nu_R+\max\cbr{B,W}}\sqrt{d}H}.
$$
Finally, we have
\begin{align*}
\sum_{k=1}^K\sum_{h=1}^H \normbphiR{k}{h} = \tilde{O}\Bigg(&\sqrt{dH\cU^*K} + \sqrt{dH^2K\numin^2} + \Bigg(\frac{\nu_{R^\epsilon} d^\frac{2+\epsilon}{2} H^\frac{1+\epsilon}{2} K^\frac{1-\epsilon'}{2(1+\epsilon')}}{\numin}\Bigg)^\frac{1}{\epsilon}\\
&+ \rbr{d^\frac{4+\epsilon}{2} H^\frac{1+\epsilon}{2} \cH^\epsilon K^\frac{1-\epsilon}{2(1+\epsilon)}}^\frac{1}{\epsilon} + \rbr{\nu_R+\max\cbr{B,W}}\sqrt{d}H\Bigg).
\end{align*}

\end{proof}

\subsection{Proof of Lemma~\ref{lem:mdp-sumV}}\label{proof:mdp-sumV}

\begin{proof}[Proof of Lemma~\ref{lem:mdp-sumV}]
Denote $b_{k,h} = \nbr{\bphi_{k,h}/\sigma_{k,h}}_{\bSigma_{k-1,h}}$, then
$$
\sum_{k=1}^K\sum_{h=1}^H \normbphiV{k}{h} = \sum_{k=1}^K\sum_{h=1}^H \sigma_{k,h} b_{k,h}.
$$
And we have $b_{k,h} \le 1$ by definition of $\sigma_{k,h}$ and thus for all $h\in[H]$,
\begin{equation}\label{eq:proof-mdp-sumV-w-sum}
\sum_{k=1}^K b_{k,h}^2 = \min\cbr{1,b_{k,h}^2} \le 2 \kappa_V
\end{equation}
by Lemma~\ref{lem:w-sum}, where $\kappa_V := d\log\rbr{1+\frac{K}{d\lambda_V\sigmamin^2}}$.

Next we bound the sum of bonus $\sum_{k=1}^K\sum_{h=1}^H \sigma_{k,h} b_{k,h}$ separately by the value of $\sigma_{k,h}$. Recall $\sigma_{k,h}$ is defined in \eqref{eq:mdp-sigma}. We decompose $[K]\times[H]$ as the union of three disjoint sets $\cJ_1, \cJ_2, \cJ_3$ where
\begin{align*}
\cJ_1 &= \cbr{(k,h)\in[K]\times[H] \bigg| \sigma_{k,h}\in\cbr{\hat{\sigma}_{k,h},\sqrt{d^3HD_{k,h}},\sigmamin}},\\
\cJ_2 &= \cbr{(k,h)\in[K]\times[H] \big| \sigma_{k,h} = \normbphiV{k}{h}},\\
\cJ_3 &= \cbr{(k,h)\in[K]\times[H] \bigg| \sigma_{k,h} = \sqrt{d^\frac{5}{2}H\cH} \normbphiV{k}{h}^\frac{1}{2}}.
\end{align*}

For the summation over $\cJ_1$, we have
\begin{equation}\label{eq:proof-mdp-sumV-J1}
\begin{aligned}
\sum_{(k,h)\in\cJ_1} \sigma_{k,h} b_{k,h} &= \sum_{(k,h)\in\cJ_1} \max\cbr{\hat{\sigma}_{k,h},\sqrt{d^3HD_{k,h}},\sigmamin} \cdot b_{k,h}\\
&\overset{(a)}{\le} \sqrt{\sum_{k=1}^K\sum_{h=1}^H (\hat{\sigma}_{k,h}^2+d^3HD_{k,h}+\sigmamin^2)} \sqrt{\sum_{k=1}^K\sum_{h=1}^H b_{k,h}^2}\\
&\overset{(b)}{\le} \sqrt{2H\kappa_V} \sqrt{\sum_{k=1}^K\sum_{h=1}^H \hat{\sigma}_{k,h}^2 + d^3H \sum_{k=1}^K\sum_{h=1}^H D_{k,h} + HK\sigmamin^2},
\end{aligned}
\end{equation}
where $(a)$ holds due to Cauchy-Schwartz inequality and $(b)$ holds due to \eqref{eq:proof-mdp-sumV-w-sum}.

We provide upper bounds for $\sum_{k=1}^K\sum_{h=1}^H \hat{\sigma}_{k,h}^2$ and $\sum_{k=1}^K\sum_{h=1}^H D_{k,h}$ in Lemma~\ref{lem:mdp-hatsigma-sum} and Lemma~\ref{lem:mdp-D-sum} respectively.

\begin{lemma}\label{lem:mdp-hatsigma-sum}
With probability at least $1-2\delta$, on event $\cE_0\cap\cE_R\cap\cE_V\cap\cA_0$, we have
\begin{align*}
\sum_{k=1}^K\sum_{h=1}^H \hat{\sigma}_{k,h}^2 = O\bigg(&\cV^*K + (\beta_0+H\beta_V) \cH \sum_{k=1}^K \sum_{h=1}^H \normbphiV{k}{h} + H\cH\beta_R\sum_{k=1}^K\sum_{h=1}^H\normbphiR{k}{h}\\
&\qquad + H^2 \cH^2 \log\frac{4\ceil{\log_2 HK}}{\delta}\bigg).
\end{align*}
\end{lemma}
\begin{proof}
    See Appendix~\ref{proof:mdp-hatsigma-sum} for a detailed proof.
\end{proof}

\begin{lemma}\label{lem:mdp-D-sum}
On event $\cE_0\cap\cA_0$, we have
\begin{align*}
\sum_{k=1}^K\sum_{h=1}^H D_{k,h} = O\bigg(&(\beta_0+H\beta_V) \cH \sum_{k=1}^K \sum_{h=1}^H \normbphiV{k}{h} + H\cH\beta_R\sum_{k=1}^K\sum_{h=1}^H\normbphiR{k}{h}\\
& + H^2 \cH^2 \log\frac{4\ceil{\log_2 HK}}{\delta}\bigg).
\end{align*}
\end{lemma}
\begin{proof}
    See Appendix~\ref{proof:mdp-D-sum} for a detailed proof.
\end{proof}

Next, for the summation over $\cJ_2$, we have $\nbr{\frac{\bphi_{k,h}}{\sigma_{k,h}}}_{\bSigma_{k-1,h}^{-1}} = 1$. Then
\begin{equation}\label{eq:proof-mdp-sumV-J2}
\sum_{(k,h)\in\cJ_2} \sigma_{k,h} b_{k,h} = \sum_{(k,h)\in\cJ_2} \nbr{\bphi_{k,h}}_{\bSigma_{k-1,h}^{-1}} b_{k,h}^2
\overset{(a)}{\le} \frac{1}{\sqrt{\lambda_V}} \sum_{k=1}^K\sum_{h=1}^H b_{k,h}^2
\overset{(b)}{\le} \frac{2H\kappa_V}{\sqrt{\lambda_V}},
\end{equation}
where $(a)$ holds due to $\bSigma_{k-1,h} \succeq \lambda_V \bI$, thus $\nbr{\bphi_{k,h}}_{\bSigma_{k-1,h}^{-1}} \le \frac{\nbr{\bphi_{k,h}}}{\sqrt{\lambda_V}} \le \frac{1}{\sqrt{\lambda_V}}$. And $(b)$ holds due to \eqref{eq:proof-mdp-sumV-w-sum}.

Then, for the summation over $\cJ_3$, we have $\sigma_{k,h} = d^{2.5}H\cH b_{k,h}$. Therefore
\begin{equation}\label{eq:proof-mdp-sumV-J3}
\sum_{(k,h)\in\cJ_3} \sigma_{k,h} b_{k,h} = d^{2.5}H\cH \sum_{(k,h)\in\cJ_3} b_{k,h}^2 \le 2d^{2.5}H^2\cH\kappa_V,
\end{equation}
where the inequality holds due to \eqref{eq:proof-mdp-sumV-w-sum}.
\end{proof}

Finally, combining \eqref{eq:proof-mdp-sumV-J1}, \eqref{eq:proof-mdp-sumV-J2}, \eqref{eq:proof-mdp-sumV-J3} and Lemma~\ref{lem:mdp-hatsigma-sum}, with probability at least $1-2\delta$, we have
\[
\sum_{k=1}^K\sum_{h=1}^H\normbphiV{k}{h} = O\rbr{\sqrt{H\kappa_V}\sqrt{(\beta_0+H\beta_V)d^3H\cH \sum_{k=1}^K\sum_{h=1}^H\normbphiV{k}{h}}+C}
\]
where
\begin{align*}
C = &\sqrt{H\kappa_V}\sqrt{\cV^*K+HK\sigmamin^2+d^3H^2\cH\beta_R\sum_{k=1}^K\sum_{h=1}^H\normbphiR{k}{h}+d^3H^3\cH^2\log\frac{4\ceil{\log_2 HK}}{\delta}}\\
&+ \frac{H\kappa_V}{\sqrt{\lambda_V}} + d^{2.5}H^2\cH\kappa_V.
\end{align*}
Using the inequality that $x \le a\sqrt{x}+b$ implies $x \le a^2+2b$, we have
\[
\sum_{k=1}^K\sum_{h=1}^H\normbphiV{k}{h} = O\rbr{(\beta_0+H\beta_V)d^3H^2\cH\kappa_V + C}.
\]
Next, we simplify the expression above by hiding logarithmic terms. Notice $\kappa_V=\tilde{O}(d)$. Setting $\lambda_V = \frac{1}{\cH^2}$, we have $\beta_0 = \tilde{O}\rbr{\frac{\sqrt{d^3H\cH^2}}{\sigmamin}}$ and $\beta_V=\tilde{O}(\sqrt{d})$. Therefore,
\begin{align*}
\sum_{k=1}^K\sum_{h=1}^H\normbphiV{k}{h} &= O(C) + \tilde{O}\rbr{d^4H^2\cH\beta_0 + d^4H^3\cH\beta_V}\\
&= O(C) + \tilde{O}\rbr{\frac{d^{5.5}H^{2.5}\cH^2}{\sigmamin} + d^{4.5}H^3\cH}.
\end{align*}
We then simplify $C$ as
\[
C = \tilde{O}\rbr{\sqrt{dH\cV^*K} + \sqrt{dH^2K\sigmamin^2} + \sqrt{d^4H^3\cH\beta_R\sum_{k=1}^K\sum_{h=1}^H\normbphiR{k}{h}} + d^{2.5}H^2\cH}.
\]
Finally, we have
\begin{align*}
\sum_{k=1}^K\sum_{h=1}^H\normbphiV{k}{h} = \tilde{O}\Bigg(&\sqrt{dH\cV^*K} + \sqrt{dH^2K\sigmamin^2} + \frac{d^{5.5}H^{2.5}\cH^2}{\sigmamin}\\
&+ \sqrt{d^4H^3\cH\beta_R\sum_{k=1}^K\sum_{h=1}^H\normbphiR{k}{h}} + d^{4.5}H^3\cH\bigg).
\end{align*}

\subsection{Proof of Lemma~\ref{lem:mdp-hatnu-sum}}
\label{proof:mdp-hatnu-sum}

\begin{proof}[Proof of Lemma~\ref{lem:mdp-hatnu-sum}]
On event $\cE_{R^\epsilon}\cap\cE_R$, by Lemma~\ref{lem:mdp-W}, for all $(k,h)\in[K]\times[H]$, we have $\abr{\sbr{\hat{\bnu}_{1+\epsilon} R_h - \bnu_{1+\epsilon} R_h}(s_{k,h},a_{k,h})} \le W_{k,h}$. Thus
\begin{align*}
\sum_{k=1}^K\sum_{h=1}^H \hat{\nu}_{k,h}^2 &= \sum_{k=1}^K\sum_{h=1}^H \rbr{\sbr{\hat{\bnu}_{1+\epsilon}R_h}(s_{k,h},a_{k,h}) + W_{k,h}}^\frac{2}{1+\epsilon}\\
&\le \sum_{k=1}^K\sum_{h=1}^H \rbr{\sbr{\bnu_{1+\epsilon}R_h}(s_{k,h},a_{k,h}) + 2W_{k,h}}^\frac{2}{1+\epsilon}\\
&\overset{(a)}{\le} 2 \sum_{k=1}^K\sum_{h=1}^H \sbr{\bnu_{1+\epsilon}R_h}^\frac{2}{1+\epsilon}(s_{k,h},a_{k,h}) + 2\sum_{k=1}^K\sum_{h=1}^H W_{k,h}^\frac{2}{1+\epsilon},
\end{align*}
where $(a)$ holds due to Jensen's inequality: for any non-negative number $a,b$, we have $(a+b)^{1+\xi} \le 2^\xi a + 2^\xi b \le 2a + 2b$ for $\xi \in [0,1]$.
We next bound the two terms in the RHS separately.

\paragraph{For the first term.}

On one hand, we have $\sum_{k=1}^K\sum_{h=1}^H \sbr{\bnu_{1+\epsilon}R_h}^\frac{2}{1+\epsilon}(s_{k,h},a_{k,h}) \le \cU K$ by Assumption~\ref{ass:bounded}.

On the other hand, we denote $X_k = \sum_{h=1}^H \sbr{\bnu_{1+\epsilon} R_h}^\frac{2}{1+\epsilon}(s_{k,h},a_{k,h})$. Use the notation $\cF_k:=\cF_{H,k}$ as the $\sigma$-field generated by all the random variables up to $k$-th episode for short. Then $X_k \ge 0$ is $\cF_k$-measurable with $|X_k| \le H\nu_R^2$ and $\Var[X_k|\cF_{k-1}] \le \EE[X_k^2|\cF_{k-1}] \le H\nu_R^2 \EE[X_k|\cF_{k-1}]$. With probability at least $1-\delta$, the variance-aware Freedman inequality in Lemma~\ref{lem:bern} gives
\begin{align*}
\sum_{k=1}^K X_k  &\le \sum_{k=1}^K \EE[X_k|\cF_{k-1}] + 3 \sqrt{H\nu_R^2  \sum_{k=1}^K \EE [X_k|\cF_{k-1}] \log \frac{2\ceil{\log_2K}}{\delta}} + 5H\nu_R^2 \log\frac{2\ceil{\log_2K}}{\delta} \\
&\le 3\sum_{k=1}^K \EE[X_k|\cF_{k-1}] + 7H\nu_R^2 \log\frac{2\ceil{\log_2K}}{\delta}.
\end{align*}
We relate $\EE[X_k|\cF_{k-1}]$ to policy $\pi^k$ as
\begin{align*}
\EE[X_k|\cF_{k-1}] &= \EE\sbr{\sum_{h=1}^H \sbr{\bnu_{1+\epsilon} R_h}^\frac{2}{1+\epsilon}(s_{k,h},a_{k,h}) | \cF_{k-1}} = \EE_{(s_h,a_h)\sim d_h^{\pi^k}}\sbr{\sum_{h=1}^H \sbr{\bnu_{1+\epsilon} R_h}^\frac{2}{1+\epsilon}(s_{h},a_{h})},
\end{align*}
where $d_h^{\pi^k}(s, a)$ is defined in \eqref{eq:d}.
Then we have
\begin{align*}
&\sum_{k=1}^K\sum_{h=1}^H \sbr{\bnu_{1+\epsilon}R_h}^\frac{2}{1+\epsilon}(s_{k,h},a_{k,h})\\
\le{}& 3 \sum_{k=1}^K\EE_{(s_h,a_h)\sim d_h^{\pi^k}}\sbr{\sum_{h=1}^H \sbr{\bnu_{1+\epsilon} R_h}^\frac{2}{1+\epsilon}(s_{h},a_{h})} + 7H\nu_R^2 \log\frac{2\ceil{\log_2K}}{\delta}\\
={}& 3 \cU^*_0K + 7H\nu_R^2 \log\frac{2\ceil{\log_2K}}{\delta},
\end{align*}
where $\cU^*_0$ is defined in \eqref{eq:U}.

\paragraph{For the second term.}

Recall $W_{k,h}$ is defined in \eqref{eq:mdp-W}. We have
\begin{align*}
\sum_{k=1}^K\sum_{h=1}^H W_{k,h}^\frac{2}{1+\epsilon} &= \sum_{k=1}^K\sum_{h=1}^H \sbr{(\beta_{R^\epsilon,k-1} + 6 \cH^\epsilon \beta_{R,k-1} \kappa) \normbphiR{k}{h}}^\frac{2}{1+\epsilon}\\
&\le (\beta_{R^\epsilon} + 6 \cH^\epsilon \beta_{R} \kappa)^\frac{2}{1+\epsilon} \rbr{\sum_{k=1}^K\sum_{h=1}^H \normbphiR{k}{h}}^\frac{2}{1+\epsilon},
\end{align*}
where the inequality holds due to Jensen's inequality.

Finally, putting pieces together and using $\cU^* = \min\cbr{\cU^*_0,\cU}$ complete the proof.
\end{proof}

\subsection{Proof of Lemma~\ref{lem:mdp-hatsigma-sum}}
\label{proof:mdp-hatsigma-sum}

\begin{proof}[Proof of Lemma~\ref{lem:mdp-hatsigma-sum}]
On event $\cE_0\cap\cE_R\cap\cE_V$, by Lemma~\ref{lem:mdp-E}, for all $(k,h)\in[K]\times[H]$, we have \\$\abr{\sbr{\hat{\VV}_h V^k_{h+1} - \VV_h V^*_{h+1}}(s_{k,h},a_{k,h})} \le E_{k,h}$. Thus
\begin{align*}
\sum_{k=1}^K\sum_{h=1}^H \hat{\sigma}_{k,h}^2 &= \sum_{k=1}^K\sum_{h=1}^H \rbr{\sbr{\hat{\VV}_h V^k_{h+1}}(s_{k,h},a_{k,h}) + E_{k,h}}\\
&\le \sum_{k=1}^K\sum_{h=1}^H \sbr{\VV_h V^*_{h+1}}(s_{k,h},a_{k,h}) + 2\sum_{k=1}^K\sum_{h=1}^H E_{k,h}.
\end{align*}
We next bound the two terms in the RHS separately.

\paragraph{For the first term.}

On one hand, we denote $X_k =  \sum_{h=1}^H [\VV_h V_{h+1}^{*}](s_{k,h}, a_{k,h})$.
Use the notation $\cF_k:=\cF_{k,h}$ as the $\sigma$-field generated by all the random variables up to $k$-th episodes for short.
Then $X_k \ge 0$ is $\cF_k$-measurable with $|X_k| \le H \cH^2$ and $\Var[X_k|\cF_{k-1}] \le \EE[X_k^2|\cF_{k-1}] \le H \cH^2 \EE [X_k|\cF_{k-1}]$.
With probability at least $1-\delta$, the variance-aware Freedman inequality in Lemma~\ref{lem:bern} gives
\begin{align*}
\sum_{k=1}^K X_k  
&\le \sum_{k=1}^K \EE[X_k|\cF_{k-1}] + 3 \sqrt{H \cH^2 \sum_{k=1}^K \EE [X_k|\cF_{k-1}] \log \frac{2\ceil{\log_2K}}{\delta}} + 5H \cH^2 \log\frac{2\ceil{\log_2K}}{\delta} \\
&\le 3\sum_{k=1}^K \EE[X_k|\cF_{k-1}] + 7H \cH^2 \log\frac{2\ceil{\log_2K}}{\delta}.
\end{align*}
We relate $\EE[X_k|\cF_{k-1}]$ to $\pi^k$ as
$$
\EE[X_k|\cF_{k-1}] = \EE\sbr{\sum_{h=1}^H \sbr{\VV_h V_{h+1}^{*}}(s_{k,h},a_{k,h}) | \cF_{k-1}} = \EE_{(s_h,a_h)\sim d_h^{\pi^k}}\sbr{\sum_{h=1}^H \sbr{\VV_h V_{h+1}^{*}}(s_{h},a_{h})},
$$
where $d_h^{\pi^k}(s, a)$ is defined in \eqref{eq:d}.
Then we have
\begin{align*}
&\sum_{k=1}^K\sum_{h=1}^H \sbr{\VV_h V^*_{h+1}}(s_{k,h},a_{k,h})\\
\le{}& 3 \sum_{k=1}^K\EE_{(s_h,a_h)\sim d_h^{\pi^k}}\sbr{\sum_{h=1}^H \sbr{\VV_h V^*_{h+1}}(s_{h},a_{h})} + 7H\cH^2 \log\frac{2\ceil{\log_2K}}{\delta}\\
={}& 3 \cV^*_0K + 7H\cH^2 \log\frac{2\ceil{\log_2K}}{\delta},
\end{align*}
where $\cV^*_0$ is defined in \eqref{eq:V}.

We restate the total variance lemma (Lemma C.5 in \citet{jin2018q}) here for completeness.

\begin{lemma}[Total variance lemma]
\label{lem:total-variance}
    With probability at least $1-\delta$, we have
    \[
    \sum_{k=1}^K \sum_{h=1}^H [\VV_h V_{h+1}^{\pi^k}](s_{k,h}, a_{k,h})  \le  2 \cV K + 2H \cH^2 \log\frac{1}{\delta} .
    \]
\end{lemma}
\begin{proof}
    See Appendix~\ref{proof:total-variance} for a detailed proof.
\end{proof}

On the other hand, we have
\begin{align*}
&\sum_{k=1}^K\sum_{h=1}^H \sbr{\VV_h V^*_{h+1}}(s_{k,h},a_{k,h})\\
=& \sum_{k=1}^K\sum_{h=1}^H \sbr{\VV_h V^{\pi^k}_{h+1}}(s_{k,h},a_{k,h}) + \sum_{k=1}^K\sum_{h=1}^H \sbr{\VV_h V^*_{h+1} - \VV_h V^{\pi^k}_{h+1}}(s_{k,h},a_{k,h})\\
\overset{(a)}{\le}{}& 2\cV K + 2H\cH^2\log\frac{1}{\delta} + 2\cH \sum_{k=1}^K\sum_{h=1}^H\sbr{\PP_h\rbr{V^k_{h+1} - \pesV^k_{h+1}}}\\
\overset{(b)}{\le}{}& 2\cV K + 2H\cH^2\log\frac{1}{\delta} + 16H\cH\beta_R\sum_{k=1}^K\sum_{h=1}^H\normbphiR{k}{h} + 16H\cH\beta_V\sum_{k=1}^K\sum_{h=1}^H\normbphiV{k}{h}\\
    &\qquad + 58H^2\cH^2\log\frac{4\ceil{\log_2 HK}}{\delta}\\
\le{}& 2\cV K + 16H\cH\beta_R\sum_{k=1}^K\sum_{h=1}^H\normbphiR{k}{h} + 16H\cH\beta_V\sum_{k=1}^K\sum_{h=1}^H\normbphiV{k}{h}\\
    &\qquad + 60H^2\cH^2\log\frac{4\ceil{\log_2 HK}}{\delta},
\end{align*}
where $(a)$ holds with probability at least $1-\delta$ due to Lemma~\ref{lem:total-variance}
and \eqref{eq:var-diff}. $(b)$ is due to $\cA_0$ holds.
\begin{equation}\label{eq:var-diff}
\begin{aligned}
&\sbr{\VV_h V_{h+1}^* - \VV_h V_{h+1}^{\pi^k}}(s_{k,h}, a_{k,h})\\
={}& [\PP_h V_{h+1}^*]^2(s_{k,h}, a_{k,h}) - [\PP_h V_{h+1}^{\pi^k}]^2(s_{k,h}, a_{k,h})\\
    &\qquad - \rbr{[\PP_h V_{h+1}^*]^2(s_{k,h}, a_{k,h}) - [ \PP_h V_{h+1}^{\pi^k}]^2(s_{k,h}, a_{k,h}) }\\
\overset{(a)}{\le}{}& [\PP_h V_{h+1}^*]^2(s_{k,h}, a_{k,h}) - \PP_h [ V_{h+1}^{\pi^k}]^2(s_{k,h}, a_{k,h})\\
\overset{(b)}{\le}{}& 2\cH \sbr{\PP_h (V_{h+1}^* - V_{h+1}^{\pi^k})}(s_{k,h}, a_{k,h})\\
\overset{(c)}{\le}{}& 2\cH \sbr{\PP_h (V_{h+1}^k - V_{h+1}^{\pi^k})}(s_{k,h}, a_{k,h}),
\end{aligned}
\end{equation}
where $(a)$ holds due to $V_{h+1}^*(\cdot) \ge V_{h+1}^{\pi^k}(\cdot)$. $(b)$ holds due to $V_{h+1}^*(\cdot)$ and $V_{h+1}^{\pi^k}$ are bounded in $[0,\cH]$. And $(c)$ holds due to Lemma~\ref{lem:mdp-opt}.

\paragraph{For the second term.}

Recall $E_{k,h}$ is defined in \eqref{eq:mdp-E}. We have
\begin{align*}
\sum_{k=1}^K\sum_{h=1}^H E_{k,h} 
\le{}& \sum_{k=1}^K\sum_{h=1}^H\sbr{4\cH \dotp{\bphi_{k,h}}{\hbw_{k-1,h} - \cbw_{k-1,h}} + 11\cH\beta_0\normbphiV{k}{h}}\\
\overset{(a)}{\le}{}& \sum_{k=1}^K\sum_{h=1}^H \sbr{4\cH \sbr{\PP_h \rbr{V^k_{h+1} - \pesV^k_{h+1}}}(s_{k,h},a_{k,h}) + 22\cH\beta_0\normbphiV{k}{h}}\\
\overset{(b)}{\le}{}& (22\beta_0+48H\beta_V) \cH \sum_{k=1}^K \sum_{h=1}^H \normbphiV{k}{h} + 32H\cH\beta_R\sum_{k=1}^K\sum_{h=1}^H\normbphiR{k}{h}\\
    &\qquad + 116 H^2 \cH^2 \log\frac{4\ceil{\log_2 HK}}{\delta}\\
={}& O\Bigg((\beta_0+H\beta_V) \cH \sum_{k=1}^K \sum_{h=1}^H \normbphiV{k}{h} + H\cH\beta_R\sum_{k=1}^K\sum_{h=1}^H\normbphiR{k}{h}\\
    &\qquad + H^2 \cH^2 \log\frac{4\ceil{\log_2 HK}}{\delta}\Bigg),
\end{align*}
where $(a)$ is due to $\cE_0$ holds and $(b)$ is due to $\cA_0$ holds.

Finally, putting pieces together and using $\cV^* = \min\cbr{\cV^*_0,\cV}$ completes the proof.
\end{proof}

\subsection{Proof of Lemma~\ref{lem:mdp-D-sum}}
\label{proof:mdp-D-sum}

\begin{proof}[Proof of Lemma~\ref{lem:mdp-D-sum}]
Recall $D_{k,h}$ is defined in \eqref{eq:mdp-D}. We have
\begin{align*}
\sum_{k=1}^K\sum_{h=1}^H D_{k,h}
\le{}& \sum_{k=1}^K\sum_{h=1}^H \sbr{2\cH \dotp{\bphi_{k,h}}{\hbw_{k-1,h} - \cbw_{k-1,h}} + 4\cH\beta_0\normbphiV{k}{h}}\\
\overset{(a)}{\le}{}& \sum_{k=1}^K\sum_{h=1}^H \sbr{2\cH \sbr{\PP_h \rbr{V^k_{h+1} - \pesV^k_{h+1}}}(s_{k,h},a_{k,h}) + 8\cH\beta_0\normbphiV{k}{h}}\\
\overset{(b)}{\le}{}& (8\beta_0+24H\beta_V) \cH \sum_{k=1}^K \sum_{h=1}^H \normbphiV{k}{h} + 16H\cH\beta_R\sum_{k=1}^K\sum_{h=1}^H\normbphiR{k}{h}\\
    &\qquad + 58 H^2 \cH^2 \log\frac{4\ceil{\log_2 HK}}{\delta}\\
={}& O\Bigg((\beta_0+H\beta_V) \cH \sum_{k=1}^K \sum_{h=1}^H \normbphiV{k}{h} + H\cH\beta_R\sum_{k=1}^K\sum_{h=1}^H\normbphiR{k}{h}\\
    &\qquad + H^2 \cH^2 \log\frac{4\ceil{\log_2 HK}}{\delta}\Bigg),
\end{align*}
where $(a)$ is due to $\cE_0$ holds and $(b)$ is due to $\cA_0$ holds.
\end{proof}

\subsection{Proof of Lemma~\ref{lem:total-variance}}
\label{proof:total-variance}

\begin{proof}[Proof of Lemma~\ref{lem:total-variance}]
	The	proof leverages the technique in Lemma~C.5 of \citet{jin2018q}.
	We define the filtration $\cF_k$ as the $\sigma$-field generated by all the random variables over the first $k$ episodes.
	And let $X_k =  \sum_{h=1}^H [\VV_h V_{h+1}^{\pi^k}](s_{k,h}, a_{k,h})$.
	It follows that $\pi^k$ is $\cF_{k-1}$-measurable, $X_k$ is $\cF_k$-measurable, and $|X_k| \le H \cH^2$.
	Let $\EE_k(\cdot) := \EE[\cdot|\cF_k]$ for short.
    We have
	\begin{align*}
    \cV &\overset{(a)}{\ge} \Var\rbr{\sum_{h=1}^H r_h(s_{k,h}, a_{k,h})}\\
    &= \EE_{k-1}\sbr{\sum_{h=1}^H r_h(s_{k,h}, a_{k,h}) - V_1^{\pi^k}(s_{k,1})}^2\\
    & \overset{(b)}{=}\EE_{k-1}\sbr{\sum_{h=1}^H \rbr{r_h(s_{k,h}, a_{k,h})  + V_{h+1}^{\pi^k}(s_{k,h+1}) - V_h^{\pi^k}(s_{k,h})}}^2\\
    & \overset{(c)}{=} \sum_{h=1}^H \EE_{k-1}\sbr{r_h(s_{k,h}, a_{k,h})  + V_{h+1}^{\pi^k}(s_{k,h+1}) - V_h^{\pi^k}(s_{k,h})}^2\\
    &= \sum_{h=1}^H \EE_{k-1}\sbr{V_{h+1}^{\pi^k}(s_{k,h+1}) - [\PP_h V_{h+1}^{\pi^k}](s_{k,h},a_{k,h})}^2\\
    &=  \EE_{k-1}\sum_{h=1}^H  [\VV_h V_{h+1}^{\pi^k}](s_{k,h}, a_{k,h})\\
    &= \EE[X_k|\cF_{k-1}]
	\end{align*}
	where $(a)$ holds due to the definition of $\cV$, $(b)$ holds due to $V_{H+1}^{\pi^k}(\cdot) = 0$ and $(c)$ holds due to Markov property of MDPs.
    
    We also note that
    \[
    \Var\rbr{\sum_{h=1}^H r_h(s_{k,h}, a_{k,h})} \le \cH \EE\sbr{\sum_{h=1}^H r_h(s_{k,h}, a_{k,h})} \le \cH V^*_1,
    \]
    where the second inequality holds due to the optimality of $V^*_1$.
    Thus, we actually have
    \begin{equation}\label{eq:mdp-first-order-relationship}
    \cV \le \cH V^*_1,
    \end{equation}
    which shows the relationship between our result with the first-order regret.
	
    Then $\Var[X_k|\cF_{k-1}] \le H \cH^2 \EE [X_k|\cF_{k-1}]$, and we have
     \[
     \sum_{k=1}^K \Var[X_k|\cF_{k-1}] \le H \cH^2 \sum_{k=1}^K \EE[X_k|\cF_{k-1}] \le H \cH^2 \cV K.
     \]
	By the Freedman inequality in Lemma~\ref{lem:freedman}, with probability at least $1-\delta$, we have
	\begin{align*}
		\sum_{k=1}^K \sum_{h=1}^H [\VV_h V_{h+1}^{\pi^k}](s_{k,h}, a_{k,h}) =& \sum_{k=1}^K X_k   \le \sum_{k=1}^K \EE[X_k|\cF_{k-1}] 
		+ \sqrt{2 H \cH^2 \cV K \log \frac{1}{\delta}}
		+ \frac{2}{3} H \cH^2 \log\frac{1}{\delta}\\
		\le& \cV K + 2\sqrt{H \cH^2 \cV K \log\frac{1}{\delta}} + \frac{2}{3} H \cH^2 \log\frac{1}{\delta}\\
		\le& 2 \cV K + 2H \cH^2 \log\frac{1}{\delta}.
	\end{align*}
\end{proof}

\section{Auxiliary Lemmas}
\label{proof:auxiliary}

\begin{lemma}[Freedman inequality \citep{freedman1975tail}]
	\label{lem:freedman}
	Let $\{ X_t \}_{t \in [T]}$ be a stochastic process that adapts to the filtration $\cF_t$ so that $X_t$ is $\cF_t$-measurable, $\EE[X_t|\cF_{t-1}] = 0$, $|X_t| \le M$ and $\sum_{t=1}^T \EE[X_t^2|\cF_{t-1}] \le V$ where $M>0$ and $V > 0$ are positive constants.
	Then with probability at least $1-\delta$, we have
	\[
	\sum_{t=1}^T X_t \le \sqrt{2V\ln\frac{1}{\delta}} + \frac{2M}{3}\ln\frac{1}{\delta}.
	\]
\end{lemma}

\begin{lemma}[Variance-aware Freedman inequality, Theorem~5 in \citet{li2023q}]
	\label{lem:bern}
	Let $\{ X_t \}_{t \in [T]}$ be a stochastic process that adapts to the filtration $\cF_t$ so that $X_t$ is $\cF_t$-measurable, $\EE[X_t|\cF_{t-1}] = 0$, $|X_t| \le M$ and $\sum_{t=1}^T \EE[X_t^2|\cF_{t-1}] \le V^2$ where $M>0$ and $V > 0$ are positive constants.
	Then with probability at least $1-\delta$, we have
	\[
	\abr{\sum_{t=1}^T X_t}  \le 3 \sqrt{\sum_{t=1}^T \EE[X_t^2|\cF_{t-1}]  \cdot \log\frac{2K}{\delta}} + 5M\log\frac{2K}{\delta},
	\]
	where $K = 1+\ceil{2\log_2 \frac{V}{M} } $.
\end{lemma}
\begin{proof}[Proof of Lemma~\ref{lem:bern}]
By Theorem~5 in \citet{li2023q}, we have for any positive integer $K \ge 1$, 
\[
\PP\left(
\abr{\sum_{t=1}^T X_t}  \le \sqrt{8\max\cbr{\sum_{t=1}^T \EE[X_t^2|\cF_{t-1}]  ,  \frac{V^2}{2^K}}\cdot  \ln\frac{2K}{\delta}}+ \frac{4M}{3}\ln\frac{2K}{\delta}
\right) \ge 1-\delta.
\]
By setting $K = 1+\ceil{2\log_2 \frac{V}{M} } $, we have $\frac{V^2}{2^K} \le M^2$.
Using $\max\{a, b\} \le a + b$, $\sqrt{a+b} \le \sqrt{a} + \sqrt{b}$ for any $a, b \ge0$ and $\ln \frac{2K}{\delta} \ge 1$, we complete the proof.
\end{proof}

The following two lemmas are the counterpart lemmas of Theorem~\ref{thm:heavy} under light-tail assumption.
\begin{lemma}[Bernstein inequality for self-normalized martingales, Lemma~4.1 in \citet{he2023nearly}]
	\label{lem:self-bern}
	Let $\{ \cG_t\}_{t \ge 0}$ be a filtration and $ \{\bx_t, \eta_t \}_{t \ge 0}$ be a stochastic process so that $\bx_t \in \RR^d$ is $\cG_t$-measurable and $\eta_t \in \RR$ is $\cG_{t+1}$-measurable.
	If $\|\bx_t\| \le L$ and $\{\eta_t\}_{t \ge 1}$ satisfies that $\EE[\eta_t|\cG_t] = 0$, $	\EE[\eta_t^2|\cG_t] \le \sigma^2$ and $|\eta_t \min \cbr{1, \| \bx_t\|_{\bZ_{t-1}^{-1}}}| \le M$ for all $ t \ge 1$.
	Then, for any $\delta \in (0, 1)$, with probability at least $1-\delta$, we have for all $t \ge 1$,
	\[
	\nbr{\sum_{i=1}^t \bx_i \eta_i}_{\bZ_t^{-1}} \le 
	8 \sigma \sqrt{d \log\rbr{1+\frac{tL^2}{d\lambda}}\log \frac{4t^2}{\delta}} + 4 M \log \frac{4t^2}{\delta},
	\]
	where $\bZ_t = \lambda \bI + \sum_{i=1}^t \bx_i \bx_i^\top$ for $t\ge 1$ and $\bZ_0 = \lambda \bI$.
\end{lemma}

\begin{lemma}[Hoeffding inequality for self-normalized martingales, Theorem 1 in \citet{abbasi2011improved}]
	\label{lem:self-hoff}
	Let $\{ \cG_t\}_{t \ge 0}$ be a filtration and $ \{\bx_t, \eta_t \}_{t \ge 0}$ be a stochastic process so that $\bx_t \in \RR^d$ is $\cG_t$-measurable and $\eta_t \in \RR$ is $\cG_{t+1}$-measurable..
	If $\|\bx_t\| \le L$ and $\{\eta_t\}_{t \ge 1}$ satisfies that $\EE[\eta_t|\cG_t] = 0$ and $|\eta_t| \le M$ for all $ t \ge 1$.
	Then, for any $\delta \in (0, 1)$, with probability at least $1-\delta$, we have for all $t \ge 1$,
	\[
	\nbr{\sum_{i=1}^t \bx_i \eta_i}_{\bZ_t^{-1}} \le 
	M \sqrt{d \log\rbr{1+\frac{tL^2}{d\lambda}} + \log \frac{1}{\delta}},
	\]
	where $\bZ_t = \lambda \bI + \sum_{i=1}^t \bx_i \bx_i^\top$ for $t\ge 1$ and $\bZ_0 = \lambda \bI$.
\end{lemma}

The next lemma is a general result of the concentration bounds for estimated value functions, which is the core technique used in \citet{jin2020provably,zhou2022computationally,he2023nearly}. We include the proof here for completeness.

\begin{lemma}
\label{lem:value-ci}
Fix any $h \in [H]$.
Consider a specific value function $f(\cdot)$ which satisfies 
\begin{enumerate}[leftmargin=*,label=(\arabic*)]
    \item $\sup_{s \in \cS}|f(s)| \le C_0$;
    \item $f \in \cV$ where $\cV$ is a class of functions with $\cN(\cV, \epsilon)$ the $\epsilon$-covering number of $\cV$ with respective to the distance $\mathrm{dist}(f, f'):= \sup_{s \in \cS} |f(s)-f'(s)|$.
\end{enumerate}
We assume there exists a deterministic $C_\sigma > 0$ and $\cA_{k,h}$ (which is $\cF_{k,h}$-measurable) such that $\cA_{k,h}  \subseteq \cbr{ \sigma_{k,h}^2 \ge [\VV_h f](s_{k,h}, a_{k,h})/ C_\sigma^2 }$ for all $k \in [K]$.
Let $\bw_{h}[f],\hbw_{k,h}[f]$ be defined in \eqref{eq:bwf}, \eqref{eq:bwf} and $\sigma_{k,h}, \bSigma_{k,h}$ be defined in our algorithm.
Under any of the following conditions, with probability at least $1-\delta/H$, it follows for all $k \in [K]$,
\begin{equation}\label{eq:ci0}
    \nbr{\hbw_{k,h}[f] - \bw_{h}[f]}_{\bSigma_{k,h}} \le \beta.
\end{equation}
\begin{enumerate}[leftmargin=*,label=(\roman*)]
    \item If $f(\cdot)$ is a deterministic function and $\bigcap_{k \in [K]}\cA_{k,h}$ is true,~\eqref{eq:ci0} holds with
    \[
    \beta = \sqrt{d \lambda_V} C_0 + 8C_\sigma \sqrt{d \log \left( 1 + \frac{K}{\sigma_{\min}^2d\lambda_V} \right) \log \frac{4H K^2}{\delta} } + \frac{4C_0}{ d^{2.5}H\cH}\log\frac{4HK^2}{\delta}.
    \]
    \item If $f(\cdot)$ is a random function and $\bigcap_{k \in [K]}\cA_{k,h}$ is true,~\eqref{eq:ci0} holds with
    \[
    \beta = 2\sqrt{d \lambda_V} C_0 + 32C_\sigma \sqrt{d \log \left( 1 + \frac{K}{\sigma_{\min}^2d\lambda_V} \right) \log \frac{4H K^2N_0}{\delta} } + \frac{4C_0}{ d^{2.5}H \cH} \log\frac{4HK^2N_0}{\delta}
    \]
    where $N_0 = |\cN(\cV, \epsilon_0)|$ and $\epsilon_0 = \min \left\{C_\sigma \sigma_{\min}, \frac{\lambda_V C_0 \sqrt{d}}{K}\sigma_{\min}^2\right\} $.
    \item  If $f(\cdot)$ is a random function,~\eqref{eq:ci0} holds with
    \begin{gather*}
        \beta = 2\sqrt{d \lambda_V} C_0 + \frac{C_0}{\sigma_{\min}} \sqrt{d\log \left( 1 + \frac{K}{\sigma_{\min}^2d\lambda_V} \right)  + \log \frac{N_1}{\delta} }.
    \end{gather*}
    where $N_1 = |\cN(\cV, \epsilon_1)|$ and $\epsilon_1 = \frac{\lambda_V C_0 \sqrt{d}}{K}\sigma_{\min}^2$.
\end{enumerate}
\end{lemma}

\begin{proof}
By definition of $\bw_{h}[f]$,
\begin{align*}
\hbw_{k,h}[f] &= \bSigma_{k,h}^{-1}\sum_{i=1}^k \sigma_{i,h}^{-2} \bphi_{i,h} f(s_{i,h+1})\\
&= \bSigma_{k,h}^{-1}\sum_{i=1}^k \sigma_{i,h}^{-2} \bphi_{i,h} \sbr{\bphi_{i,h}^\top \bw_{h}[f] + f(s_{i,h+1}) - [\PP_h f](s_{i,h},a_{i,h})}\\
&= \bw_{h}[f] - \lambda_V \bSigma_{k,h}^{-1} \bw_{h}[f] + \bSigma_{k,h}^{-1}\sum_{i=1}^k \sigma_{i,h}^{-2} \bphi_{i,h} \sbr{f(s_{i,h+1}) - [\PP_h f](s_{i,h},a_{i,h})}.
\end{align*}
Then it follows that
\begin{align*}
&\nbr{\hbw_{k,h}[f] - \bw_{h}[f]}_{\bSigma_{k,h}}\\
\le{}& \lambda_V \nbr{\bSigma_{k,h}^{-1} \bw_{h}[f]}_{\bSigma_{k,h}} + \nbr{\bSigma_{k,h}^{-1}\sum_{i=1}^k \sigma_{i,h}^{-2} \bphi_{i,h} \sbr{f(s_{i,h+1}) - [\PP_h f](s_{i,h},a_{i,h})}}_{\bSigma_{k,h}}\\
={}& \lambda_V \nbr{\bw_{h}[f]}_{\bSigma_{k,h}^{-1}} + \nbr{\sum_{i=1}^k \sigma_{i,h}^{-2} \bphi_{i,h} \sbr{f(s_{i,h+1}) - [\PP_h f](s_{i,h},a_{i,h})}}_{\bSigma_{k,h}^{-1}}\\
\le{}& \sqrt{d\lambda_V} C_0 + \nbr{\sum_{i=1}^k \sigma_{i,h}^{-2} \bphi_{i,h} \sbr{f(s_{i,h+1}) - [\PP_h f](s_{i,h},a_{i,h})}}_{\bSigma_{k,h}^{-1}},
\end{align*}
where the last inequality holds due to $\nbr{\bw_{h}[f]} \le \sqrt{d}C_0$.
\begin{enumerate}[leftmargin=*,label=(\roman*)]
    \item{Assume $f(\cdot)$ is a deterministic function.
        
    We set $\cG_i = \cF_{i,h}$, $\bx_i = \sigma_{i,h}^{-1} \bphi_{i,h}$, $\eta_i = \sigma_{i,h}^{-1} \sbr{f(s_{i,h+1}) - [\PP_h f](s_{i,h},a_{i,h})}$ and $\bZ_k = \lambda_V\bI + \sum_{i=1}^k \sigma_{i,h}^{-2} \bphi_{i,h}\bphi_{i,h}^\top= \bH_{k,h}$.
    Since $f(\cdot)$ is deterministic, $\cA_{k,h} \in \cF_{k,h}$.
    Clearly $\bx_i \in \cG_i, \EE[\eta_i|\cG_i] = 0$ and $\EE[\eta_i^2|\cG_i] \le C_\sigma^2$.
    We also have $\|\bx_i\| \le \sigma_{\min}^{-1}, |\eta_i| \le  C_0 \sigma_{i,h}^{-1}$ and $ \|\bx_i\|_{\bZ_{i-1}} = w_{i,h}$.
    As a result, $|\eta_i| \min \left\{1, \|\bx_i\|_{\bZ_{i-1}} \right\} \le C_0 \frac{w_{i,h}}{\sigma_{i,h}} \le \frac{C_0}{d^{2.5}H\cH}$ where the last inequality uses $\sigma_{i,h}^2 \ge \cH d^{2.5} H \|\bphi_{i,h}\|_{\bH_{k-1,h}^{-1}}$.
    By Lemma~\ref{lem:self-bern}, it follows that with probability $1-\frac{\delta}{H}$, for all $k \in [K]$,
    \begin{align*}
    &\nbr{\sum_{i=1}^k \sigma_{i,h}^{-2} \bphi_{i,h} \sbr{f(s_{i,h+1}) - [\PP_h f](s_{i,h},a_{i,h})}}_{\bSigma_{k,h}^{-1}}
    = \left\| \sum_{i=1}^k \bx_i \eta_i \right\|_{\bZ_k^{-1}} \\
    &\le 8C_\sigma \sqrt{d \log \left( 1 + \frac{K}{\sigma_{\min}^2d\lambda_V} \right) \log \frac{4H K^2}{\delta} } + \frac{4C_0}{d^{2.5}H\cH}\log\frac{4HK^2}{\delta}.
    \end{align*}
    
    Finally, on the event $\bigcap_{k \in [K]}\cA_{k,h}$, we will have all the indicator functions equal to one.
    }
    \item{If $f(\cdot)$ is a random function, covering arguments are used to handle the possible correlation between $f(\cdot)$ and history data.
    
    Denote the $\epsilon_0$-net of $\cV$ by $\cN(\cV, \epsilon_0)$ where $\epsilon_0 = \min \left\{C_\sigma \sigma_{\min}, \frac{\lambda_V C_0 \sqrt{d}}{K}\sigma_{\min}^2\right\} $.
    Hence, for any $f \in \cV$, there exists $\bar{f} \in\cN(\cV, \epsilon_0)$ such that $\|\bar{f} - f\|_{\infty} = \sup_{s \in \cS}|f(s)-\bar{f}(s)| \le \epsilon_0$. 
    Then,
    \begin{align*}
    &\nbr{\sum_{i=1}^k \sigma_{i,h}^{-2} \bphi_{i,h} \sbr{f(s_{i,h+1}) - [\PP_h f](s_{i,h},a_{i,h})}}_{\bSigma_{k,h}^{-1}}\\
    \le{}& \underbrace{\nbr{\sum_{i=1}^k \sigma_{i,h}^{-2} \bphi_{i,h} \sbr{\bar{f}(s_{i,h+1}) - [\PP_h \bar{f}](s_{i,h},a_{i,h})}}_{\bSigma_{k,h}^{-1}}}_{\text{(i)}}\\
    &\qquad + \underbrace{\nbr{\sum_{i=1}^k \sigma_{i,h}^{-2} \bphi_{i,h} \sbr{f(s_{i,h+1}) - \bar{f}(s_{i,h+1}) - [\PP_h (f-\bar{f})](s_{i,h},a_{i,h})}}_{\bSigma_{k,h}^{-1}}}_{\text{(ii)}}.
    \end{align*}
    For term (ii), due to $\|\bphi_{i,h}\| \le 1$ and $|f(s_{i,h+1}) - \bar{f}(s_{i,h+1}) - [\PP_h (f-\bar{f})](s_{i,h},a_{i,h})| \le \epsilon_0$, we have
    \begin{align*}
    &\nbr{\sum_{i=1}^k \sigma_{i,h}^{-2} \bphi_{i,h} \sbr{f(s_{i,h+1}) - \bar{f}(s_{i,h+1}) - [\PP_h (f-\bar{f})](s_{i,h},a_{i,h})}}_{\bSigma_{k,h}^{-1}}\\
    \le{}& \frac{K\epsilon_0}{\sigma_{\min}^2\sqrt{\lambda_V}} \le \sqrt{d \lambda_V} C_0.
    \end{align*}
    We next bound term (i) as follows.
    On the event $\cA_{k,h}$, by definition of $\epsilon_0$,
    \begin{align*}
    [\VV_h \bar{f}](s_{k,h}, a_{k,h}) &\le 2  [\VV_h f](s_{k,h}, a_{k,h}) + 2 [\VV_h (\bar{f}-f)](s_{k,h}, a_{k,h})\\
    &\le 2 C_\sigma^2 \sigma_{k,h}^2 + 2\epsilon_0^2 \le 4 C_\sigma^2 \sigma_{k,h}^2.
    \end{align*}
    For any fixed $f' \in \cV$, we set $\cG_i = \cF_{i,h}$, $\bx_i = \sigma_{i,h}^{-1} \bphi_{i,h}$, $\eta_i = \sigma_{i,h}^{-1} \sbr{\bar{f}(s_{i,h+1}) - [\PP_h \bar{f}](s_{i,h},a_{i,h})} \ind_{\cA_{k,h}}$ and $\bZ_k = \lambda_V\bI + \sum_{i=1}^k \sigma_{i,h}^{-2} \bphi_{i,h}\bphi_{i,h}^\top= \bH_{k,h}$.
    Moreover, due to the choice of $\sigma_{i,h}$, it follows that $|\eta_i| \min \left\{1, \|\bx_i\|_{\bZ_{i-1}} \right\} \le C_0 \frac{w_{i,h}}{\sigma_{i,h}} \le \frac{C_0}{d^{2.5}H\cH}$.
    By Lemma~\ref{lem:self-bern}, it follows that with probability $1-\frac{\delta}{H}$, for all $k \in [K]$,
    \begin{align*}
    &\sup_{f' \in \cN(\cV, \epsilon_0)}\nbr{\sum_{i=1}^k \sigma_{i,h}^{-2} \bphi_{i,h} \sbr{\bar{f}(s_{i,h+1}) - [\PP_h \bar{f}](s_{i,h},a_{i,h})}\ind_{\cA_{k,h}}}_{\bSigma_{k,h}^{-1}}\\
    \le{}& 32C_\sigma \sqrt{d \log \left( 1 + \frac{K}{\sigma_{\min}^2d\lambda_V} \right) \log \frac{4H K^2N_0}{\delta}} + \frac{4C_0}{ d^{2.5}H \cH}\log\frac{4HK^2N_0}{\delta}.
    \end{align*}
    where $N_0 = |\cN(\cV, \epsilon_0)|$. 
    
    As a result, on the event $\bigcap_{k \in [K]}\cA_{k,h}$, all the indicator functions equal to one, which completes the proof.
    }
    \item The proof is almost similar to the second item except that we use Lemma~\ref{lem:self-hoff} to analyze the term $(I)$.
    Noticing we also have $|\eta_i| = |\sigma_{i,h}^{-1} \sbr{\bar{f}(s_{i,h+1}) - [\PP_h \bar{f}](s_{i,h},a_{i,h})}| \le \frac{C_0}{\sigma_{\min}}$.
    By Lemma~\ref{lem:self-hoff}, it follows that with probability $1-\frac{\delta}{H}$, for all $k \in [K]$,
    \begin{align*}
    &\sup_{f' \in \cN(\cV, \epsilon_0)}\nbr{\sum_{i=1}^k \sigma_{i,h}^{-2} \bphi_{i,h} \sbr{\bar{f}(s_{i,h+1}) - [\PP_h \bar{f}](s_{i,h},a_{i,h})}}_{\bSigma_{k,h}^{-1}}\\
    \le{}& \frac{C_0}{\sigma_{\min}} \sqrt{d\log \left( 1 + \frac{K}{\sigma_{\min}^2d\lambda_V} \right)  + \log\frac{N_1}{\delta} },
    \end{align*}
    where $N_1 := |\cN(\cV, \varepsilon_1)|$.
\end{enumerate}
\end{proof}

\begin{lemma}[Lemma 11 in \citet{abbasi2011improved}]
\label{lem:w-sum}
Let $\{\bx_t\}_{t\in[T]} \subset \RR^d$ and assume $\|\bx_t\| \le L$ for all $t\in[T]$. Set $\bSigma_t = \sum_{s=1}^t \bx_t \bx_t^\top + \lambda \bI$. Then it follows that
\[
    \sum_{t=1}^T \min \cbr{1, \|\bx_t\|^2_{\bSigma_{t-1}^{-1}}} \le 2d \log\rbr{1 + \frac{T L^2}{d\lambda}}.
\]
\end{lemma}

\begin{lemma}[Lemma 12 in \citet{abbasi2011improved}]
	\label{lem:matrix-ratio}
	Suppose $\bA, \bB \in \RR^{d \times d}$ are two positive definite matrices satisfying that $\bA \succeq \bB$, then for any $\bx \in \RR^d$,
	\[
	\|\bx\|_{\bB^{-1}} \le \|\bx\|_{\bA^{-1}} \sqrt{\frac{\det(\bA)}{\det(\bB)}}.
	\]
\end{lemma}

The following lemmas are concerning function class and covering number, which are also used in \citet{he2023nearly}.
Let $\cK = \cbr{k_1, k_2, \cdots}$ denote the set of episodes where the algorithm updates the value function in Algorithm~\ref{algo:mdp}. For a fixed number of episodes $K$, we have $|\cK| \le K$ trivially.
Lemma~\ref{lem:rare-update} shows $|\cK|$ is only logarithmically related to $K$ due to the mechanism of rare-switching value function updates.
\begin{lemma}
	\label{lem:rare-update}
	\[
	|\cK| \le  dH \log_2 \rbr{1 + \frac{K}{\lambda_R \numin^2}} \rbr{1 + \frac{K}{\lambda_V \sigma_{\min}^2}}.
	\]
\end{lemma}
\begin{proof}
According to the updating policy, for each episode $k_i$, there exists a stage $h' \in [H]$ such that $\det(\bH_{k_i-1,h'}) \ge 2 \det( \bH_{k_{i-1}-1,h'})$ or $\det(\bSigma_{k_i-1,h'}) \ge 2 \det( \bSigma_{k_{i-1}-1,h'})$.
Since we always have $\bH_{k_i-1,h} \succeq \bH_{k_{i-1}-1,h}$ and $\bSigma_{k_i-1,h} \succeq \bSigma_{k_{i-1}-1,h}$ for all $h \in [H]$, it then follows that
\[
\prod_{h \in [H]} \det(\bH_{k_i-1,h}) \prod_{h \in [H]} \det(\bSigma_{k_i-1,h}) \ge 2 \prod_{h \in [H]} \det(\bH_{k_{i-1}-1,h}) \prod_{h \in [H]} \det(\bSigma_{k_{i-1}-1,h}).
\]
By induction, it follows that
\begin{align*}
\prod_{h \in [H]} \det(\bH_{h,k_{|\cK|}-1}) \prod_{h \in [H]} \det(\bSigma_{h,k_{|\cK|}-1})
\ge{}& 2^{|\cK|} \prod_{h \in [H]} \det(\bH_{k_{1}-1,h}) \prod_{h \in [H]} \det(\bSigma_{k_{1}-1,h})\\
\ge{}& 2^{|\cK|} \prod_{h \in [H]} \det(\lambda_R \bI) \prod_{h \in [H]} \det(\lambda_V \bI)\\
={}& 2^{|\cK|} (\lambda_R\lambda_V)^{d H}.
\end{align*}
On the other hand, due to $\bH_{h,k_{|\cK|}-1} \preceq \bH_{h, K}$, the determinant $\det(\bH_{h,k_{|\cK|}-1})$ is upper bounded by
\[
\prod_{h \in [H]} \det(\bH_{h,k_{|\cK|}-1}) 
\le \prod_{h \in [H]} \det(\bH_{h,K})  \le \left( \lambda_R + \frac{K}{\numin^2} \right)^{dH}.
\]
A similar result holds for $\prod_{h \in [H]} \det(\bSigma_{h,k_{|\cK|}-1})$.
Finally, we have
\[
|\cK| \le  dH \log_2 \rbr{1 + \frac{K}{\lambda_R \numin^2}} \rbr{1 + \frac{K}{\lambda_V \sigma_{\min}^2}}.
\]
The proof is completed.
\end{proof}

The optimistic value function $V_h^k(\cdot) = \min_{k_i \le k} \max_a Q_h^{k_i}(\cdot, a)$ belongs to the function class $\cV^{+}$
\begin{equation}
\label{eq:function-pos}
\begin{aligned}
\cV^{+} = \bigg\{f | 	f (\cdot) &= \max_{a \in \cA}\min_{i \le |\cK|} \min\cbr{\dotp{\bphi(\cdot, a)}{\btheta_i + \bw_i} + \beta_{R} \|\bphi(\cdot, a) \|_{\bH_i^{-1}} + \beta_{V} \|\bphi(\cdot, a) \|_{\bSigma_i^{-1}}, \cH},\\
&\|\btheta_i\| \le B, \|\bw_i \| \le L, \bH_i \succeq \lambda_R \bI, \bSigma_i \succeq \lambda_V \bI\bigg\}.
\end{aligned}
\end{equation}
while the pessimistic value function $\pesV_h^k(\cdot) = \max_{k_i \le k} \max_a \pesQ_h^{k_i}(\cdot, a)$ belong to the function class $\cV^{-}$,
\begin{equation}
\label{eq:function-pes}
\begin{aligned}
\cV^{-} = \bigg\{f | 	f (\cdot) &= \max_{a \in \cA}\max_{i \le |\cK|} \max\cbr{\dotp{\bphi(\cdot, a)}{\btheta_i + \bw_i} - \beta_{R} \|\bphi(\cdot, a) \|_{\bH_i^{-1}} - \beta_{V} \|\bphi(\cdot, a) \|_{\bSigma_i^{-1}}, 0},\\
&\|\btheta_i\| \le B, \|\bw_i \| \le L, \bH_i \succeq \lambda_R \bI, \bSigma_i \succeq \lambda_V \bI\bigg\}.
\end{aligned}
\end{equation}
Here $B$ and $L = \cH \sqrt{\frac{dK}{\lambda_V}}$ are uniform upper bounds of $\btheta_i$ and $\bw_i$ respectively (See Lemma E.2 of \citet{he2023nearly} for its proof).

Lemma~\ref{lem:covering} gives the covering number of function class $\cV^{+}, \cV^{-}$ and the squared version.

\begin{lemma}[Covering number of value functions]
	\label{lem:covering}
	Let $\cV^{\pm}$ denote the class of optimistic or pessimistic value functions with definition in~\eqref{eq:function-pos} and~\eqref{eq:function-pes} respectively. And denote the squared version of $\cV$ as $\sbr{\cV}^2 := \cbr{f^2|f\in\cV}$.
	Let $\cN(\cV, \epsilon)$ be the $\epsilon$-covering number of $\cV$ with respective to the distance $\mathrm{dist}(f, f'):= \sup_{s \in \cS} |f(s)-f'(s)|$.
	Then,
    \[
    \log \cN(\cV^\pm, \epsilon)
    \le \sbr{d \log \rbr{1+ \frac{4(B+L)}{\epsilon}} + d^2 \log \rbr{1 +\frac{8 d^{1/2} \beta_{R,K}^2}{\lambda_R \epsilon^2}} \rbr{1+\frac{ 8 d^{1/2} \beta_V^2}{\lambda_V \epsilon^2}}} |\cK|,
    \]
    {\small
    \begin{align*}
    \log \cN([\cV^{+}]^2, \epsilon)
    \le \sbr{d \log \rbr{1+ \frac{8H(B+L)}{\epsilon}} + d^2 \log \rbr{1 +\frac{32d^{1/2}H^2 \beta_{R,K}^2}{\lambda_R \epsilon^2}} \rbr{1+\frac{ 32d^{1/2}H^2 \beta_V^2}{\lambda_V \epsilon^2}}} |\cK|,
    \end{align*}
    }
    where $|\cK|$ is the number of episodes where the algorithm updates the value function in Algorithm~\ref{algo:mdp}.
\end{lemma}
\begin{proof}
The proof is nearly identical to Lemma E.6, E.7, E.8 in \citet{he2023nearly} except for the following differences.
First, the weight in dot product is $\btheta_i + \bw_i$ instead of $\bw_i$, so $L$ is replaced by $B+L$ in the results. Second, since we maintain two different matrices $\bH_i$ and $\bSigma_i$, the term with respect to the extra matrix is added accordingly.
\end{proof}


\end{document}